\documentclass[twoside,11pt]{article}

 \usepackage[abbrvbib, preprint]{jmlr2e}
\usepackage{algorithm}
\usepackage{algorithmic}
\usepackage{amsmath} 
\usepackage{subfig}
\usepackage[nameinlink,capitalize,noabbrev]{cleveref}
\usepackage{bm}
\usepackage{amssymb}
\usepackage{mathtools}
\usepackage{enumitem}
\usepackage{xspace}
\usepackage{xcolor}
\usepackage{booktabs}
\usepackage{multirow}
\usepackage{hyperref}
\hypersetup{
	colorlinks,
    linkcolor=black,
	urlcolor=pink,
	citecolor=blue}

\usepackage{circuitikz} 
\usetikzlibrary{shapes.geometric, positioning, calc, arrows, automata}

\DeclareMathOperator*{\argmin}{argmin}
\DeclareMathOperator*{\argmax}{argmax}

\newcommand{\regret}{\mathrm{Regret}}
\newcommand{\grad}{\nabla}

\newtheorem{assumption}{Assumption}

\usepackage{lastpage}
\jmlrheading{23}{2022}{1-\pageref{LastPage}}{1/21; Revised 5/22}{9/22}{21-0000}{Author One and Author Two}

\firstpageno{1}

\begin{document}

\title{Infinite-Horizon Reinforcement Learning with Multinomial Logistic Function Approximation}

\author{\name Jaehyun Park  \email jhpark@kaist.ac.kr
       \AND
       \name Junyeop Kwon  \email junyeopk@kaist.ac.kr
       \AND
       \name Dabeen Lee$^{\dagger}$ \email dabeenl@kaist.ac.kr
       \AND
       \addr 
        Department of Industrial and Systems Engineering, KAIST, Daejeon 34141, South Korea\\
        $^\dagger$ corresponding author
       }

\editor{My editor}

\maketitle

\begin{abstract}
We study model-based reinforcement learning with non-linear function approximation where the transition function of the underlying Markov decision process (MDP) is given by a multinomial logistic (MNL) model. We develop a provably efficient discounted value iteration-based algorithm that works for both infinite-horizon average-reward and discounted-reward settings. For average-reward communicating MDPs, the algorithm guarantees a regret upper bound of $\tilde{\mathcal{O}}(dD\sqrt{T})$ where $d$ is the dimension of feature mapping, $D$ is the diameter of the underlying MDP, and $T$ is the horizon. For discounted-reward MDPs, our algorithm achieves $\tilde{\mathcal{O}}(d(1-\gamma)^{-2}\sqrt{T})$ regret where $\gamma$ is the discount factor. Then we complement these upper bounds by providing several regret lower bounds. We prove a lower bound of $\Omega(d\sqrt{DT})$ for learning communicating MDPs of diameter $D$ and a lower bound of $\Omega(d(1-\gamma)^{3/2}\sqrt{T})$ for learning discounted-reward MDPs with discount factor $\gamma$. Lastly, we show a regret lower bound of $\Omega(dH^{3/2}\sqrt{K})$ for learning $H$-horizon episodic MDPs with MNL function approximation where $K$ is the number of episodes, which improves upon the best-known lower bound for the finite-horizon setting. 
\end{abstract}

\section{Introduction}

Function approximation schemes have been successful in modern reinforcement learning (RL) under the presence of large state and action spaces. Applications and domains where function approximation approaches have been deployed include video games~\citep{atari-mnih}, Go~\citep{go-silver}, robotics~\citep{robotics}, and autonomous driving~\citep{driving}. Such empirical success has motivated a plethora of theoretical studies that establish provable guarantees for RL with function approximation. The first line of theoretical work considers linear function approximation, such as linear Markov Decision Processes (MDPs)~\citep{yang-wang-2019,jin-linear-2020} and linear mixture MDPs~\citep{pmlr-v120-jia20a,pmlr-v119-ayoub20a,pmlr-v139-zhou21a} where the reward and transition functions are linear in some feature mappings. While (nearly) minimax optimal algorithms have been developed for linear MDPs~\citep{He2023,Agarwal23,Hu2022} and for linear mixture MDPs~\citep{zhou-mixture-finite-optimal}, the linearity assumption is restrictive and rarely holds in practice. In particular, when a linear model is misspecified, these algorithms may suffer from linear regret~\citep{jin-linear-2020}.

RL with general function approximation has recently emerged as an alternative to the linear function approximation framework. The term general here means that it makes minimal structural assumptions about the family of functions taken for approximation. Some concepts that lead to conditions ensuring sample-efficient learning are the Bellman rank~\citep{jiang17}, the eluder dimension~\citep{Wang-Eluder}, the Bellman eluder dimension~\citep{jin2021bellman}, the bilinear class~\citep{bilinear}, the decision-estimation coefficient~\citep{foster2023statistical}, and the generalized eluder coefficient~\citep{zhong2023gec}. Recently, \cite{he2024sampleefficient} considered infinite-horizon average-reward MDPs with general function approximation. However, algorithms for these frameworks require an oracle to query from some abstract function class. In practice, the oracle would correspond to solving an abstract non-convex optimization/regression problem. Furthermore, no regret lower bound has been identified for a general function approximation framework.

More concrete non-linear function approximation models have been proposed recently. \cite{yang-NTK,pmlr-v119-xu20c,pmlr-v120-yang20a} considered representing the $Q$ function by an overparametrized neural network based on the neural tangent kernel. \cite{wang2021optimism} studied generalized linear models for approximating the $Q$ function. \cite{liu-neural,zhang2023on} focused on the case where the $Q$ function is smooth and lies in the Besov space or the Barron space, and they used a two-layer neural network to approximate the $Q$ function. \cite{HwangOh2023} proposed a framework to represent the transition function by a multinomial logistic (MNL) model. Indeed, the multinomial logistic model can naturally represent state transition probabilities, providing a practical alternative to linear function approximation. The model is widely used for modeling multiple outcomes, such as multiclass classification~\citep{mlbook-bishop}, news recommendations~\citep{li-recommendation1,li-recommendation2}, and assortment optimization~\citep{assortment}. 

For RL with MNL approximation, \cite{HwangOh2023} developed an efficient model-based algorithm that achieves an $\tilde{\mathcal{O}}(\kappa^{-1}d H^{2}\sqrt{K})$ regret  where $d$ is the dimension of the transition core, $H$ is the horizon, $K$ is the number of episodes, and $\kappa\in(0,1)$ is a problem-dependent quantity. Recently, \cite{cho2024randomized} and~\cite{li2024provably} developed algorithms that both achieve a regret bound of $\tilde{\mathcal{O}}(dH^{2}\sqrt{K}+\kappa^{-1}d^2H^2)$, avoiding a dependence on $\kappa$ in the leading term. Their algorithms are based on recently proposed online Newton-based parameter estimation schemes for logistic bandits due to~\cite{zhang-sugiyama23,lee2024nearlyminimaxoptimalregret}. 
Moreover, \cite{li2024provably} presented the first lower bound for this setting, given by $\Omega(dH\sqrt{\kappa^* K})$ where $\kappa^*\in(0,1)$ is another problem-dependent constant similar to $\kappa$. 

\paragraph{Our Contributions}
This paper contributes to the RL with MNL approximation literature with the following new theoretical results. Our results are summarized in \Cref{results}.
\begin{table*}[h!]
\caption{Summary of Our Results on Regret Upper and Lower Bounds for RL with MNL Approximation}\label{results}
\begin{center}
\begin{tabular}{c|c|c}
\toprule
{\bf \small Setting} & {\bf\small Regret Upper Bound}   & {\bf \small Regret Lower Bound}\\
\midrule
\multirow{2}{*}{\small Finite-Horizon} & {\small $\tilde{\mathcal{O}}\left(dH^2\sqrt{K}+\kappa^{-1}d^2H^2\right)$}   & {\small $\Omega\left(dH^{3/2}\sqrt{K}\right)$ }\\
& {\small \citep{cho2024randomized,li2024provably}}&(\Cref{thm:lb-finite})\\
\midrule
\multirow{2}{*}{\small Average-Reward} & {\small $\tilde{\mathcal{O}}\left(dD\sqrt{T} + \kappa^{-1}d^2D\right)$} &{\small $\Omega\left(d \sqrt{DT}\right)$}\\
&(\Cref{thm:average-ub})&(\Cref{thm:lb-infinite}) \\
\midrule
\multirow{2}{*}{\small Discounted-Reward} &{\small $\tilde{\mathcal{O}}\left(d(1-\gamma)^2\sqrt{T}+\kappa^{-1}d^2(1-\gamma)^2\right)$ }&{\small $\Omega\left(d (1-\gamma)^{3/2}\sqrt{T}\right)$} \\
&(\Cref{thm:discounted-ub})&(\Cref{thm:lb-infinite-discounted})\\
\bottomrule
\end{tabular}
\end{center}
\end{table*}
\begin{itemize}
    \item We prove that there is a family of $H$-horizon episodic MDPs with MNL transitions for which any algorithm incurs a regret of $\Omega(dH^{3/2}\sqrt{K})$. This improves upon the lower bound due to~\cite{li2024provably} by a factor of $O(\sqrt{H/\kappa^{*}})$.

    \item We develop \texttt{UCMNLK}, a discounted extended value iteration-based algorithm that works for infinite-horizon average-reward and discounted-reward MDPs with MNL function approximation. For learning average-reward MDPs with diameter at most $D$, \texttt{UCMNLK} guarantees $\tilde{\mathcal{O}}(dD\sqrt{T}+\kappa^{-1}d^2D)$ regret. For learning discounted-reward MDPs with discount factor $\gamma$, \texttt{UCMNLK} provides a regret upper bound of $\tilde{\mathcal{O}}((d\sqrt{T}+\kappa^{-1}d^2)/(1-\gamma)^{2})$.
   
    \item We prove a lower bound of $\Omega(d\sqrt{DT})$ for learning infinite-horizon average-reward communicating MDPs with MNL transitions and diameter $D$. 
     
    \item We show a lower bound of $\Omega(d(1-\gamma)^{3/2}\sqrt{T})$ for learning infinite-horizon discounted-reward MDPs with discount factor $\gamma$.
\end{itemize}
While \texttt{UCMNLK} is inspired by \texttt{UCLK} of \cite{pmlr-v139-zhou21a} developed for discounted-reward linear mixture MDPs, it has several novel components and thus works for the average-reward setting as well. First, we develop an efficient extended value iteration scheme for MNL function approximation. As the multinomial logistic probability function is non-convex in the transition parameter vector $\theta$, optimization over $\theta$ is not tractable. Instead, we construct and optimize over confidence polytopes for the true transition probability, thereby achieving computational efficiency. Second, for the average-reward setting, we approximate a given average-reward MDP by a discounted-reward MDP with an appropriate discount factor. We show that the discounted value function returned by extended value iteration has a bounded span. This leads to an analysis based on a novel regret decomposition.

We derive the lower bounds by approximating a multinomial logistic function to a linear function, based on the mean value theorem. This approximation technique allows us to bridge MDPs with a multinomial logistic transition model and linear mixture MDPs. Then we deduce our results from the known regret lower bounds for linear mixture MDPs by~\cite{zhou-mixture-finite-optimal,pmlr-v139-zhou21a,yuewu2022}.

\section{Preliminaries} \label{sec: Preliminaries}

\paragraph{Notations} For a vector $x\in\mathbb{R}^d$ and a positive semidefinite matrix $A\in\mathbb{R}^{d\times d}$, $\|x\|_2$ denotes the $\ell_2$-norm of $x$, and we denote by $\|x\|_A = \sqrt{x^{\top}Ax}$ the weighted $\ell_2$-norm of $x$. Given a matrix $A$, $\|A\|_2$ denotes its spectral norm. For a symmetric matrix $A$, let $\lambda_{\min}(A)$ and $\lambda_{\max}(A)$ denote its minimum and maximum eigenvalues, respectively. Let $\mathbf{1}\{\mathcal{E}\}$ be the indicator function of event $\mathcal{E}$. A random variable $Y \in \mathbb{R}$ is $R$-sub-Gaussian if $\mathbb{E} [Y] = 0$ and $\mathbb{E} [\exp(sY)] \le \exp ({R^2 s^2 /2})$ for any $s \in \mathbb{R}$. Let $\Delta(\mathcal{X})$ denote the family of probability measures on $\mathcal{X}$. For any positive integers $m,n$ with $m<n$, $[n]$ and $[m:n]$ denote $\{1,\ldots,n\}$ and $\{m,\ldots,n\}$, respectively.
    
\subsection{Infinite-Horizon Average-Reward MDP}\label{sec:infinite-average}

We consider an infinite-horizon MDP specified by $M = (\mathcal{S}, \mathcal{A}, p, r)$, 
where $\mathcal{S}$ is the state space, $\mathcal{A}$ is the action space, $p(s'\mid s,a)$ denotes the unknown transition probability of transitioning to state $s^{\prime}$ from state $s$ after taking action $a$, and $r : \mathcal{S} \times \mathcal{A} \rightarrow [0, 1]$ is the known reward function. Throughout this paper, we assume that both $\mathcal{S}$ and $\mathcal{A}$ are finite. A stationary policy $\pi : \mathcal{S} \rightarrow \Delta(\mathcal{A})$ is given by $\pi(a\mid s)$ specifying the probability of taking action $a$ at state $s$. When $\pi$ is deterministic, i.e., for each $s\in \mathcal{S}$ there exists $a\in\mathcal{A}$ with $\pi(a\mid s)=1$, we write that $a=\pi(s)$ with abuse of notation. Starting from an initial state $s_1=s$, for each time step $t$, an algorithm $\mathfrak{A}$ selects action $a_t$ based on state $s_t$, and then $s_{t+1}$ is drawn according to the transition function $p(\cdot\mid s_t, a_t)$. Then the cumulative reward of $\mathfrak{A}$ incurred over $T$ times steps is given by $R(\mathfrak{A}, s,T)=\sum_{t=1}^T r(s_t,a_t)$,  and the average reward of $\mathfrak{A}$ is defined as $J(\mathfrak{A},s)= \liminf_{T\to\infty}\mathbb{E}\left[R(\mathfrak{A},s,T)\mid s_1=s\right]/T$. It is known that the average reward can be maximized by a deterministic stationary policy~\cite[See][]{puterman2014markov}. Given a stationary policy $\pi$ starting from state $s$, the average reward is given by $J^{\pi}(s) = \liminf_{T\to\infty} \mathbb{E}\left[\sum_{t=1}^T r(s_t,a_t)\mid s_1=s\right]/T$.

In this paper, following \cite{Auer_Jaksch2010}, we focus on the class of communicating MDPs that have a finite diameter. Here, the diameter is defined as follows. Given an MDP $M$ and a policy $\pi$, let $T(s' \mid M, \pi, s)$ denote the number of steps after which state $s'$ is reached from state $s$ for the first time. Then the diameter of $M$ is defined as $D(M) = \max_{s \neq s' \in \mathcal{S}} \min_{\pi: \mathcal{S} \rightarrow \mathcal{A}} \mathbb{E} \left[T(s' \mid M, \pi, s)\right]$. %
For a communicating MDP $M$, it is known that the optimal average reward does not depend on the initial state $s$~\cite[See][]{puterman2014markov}, and therefore, there exists $J^*$ such that $J^* = J^*(s) :=\max_{\pi}J^\pi(s).$
Based on this, we consider $\regret(T) = T\cdot J^* -\sum_{t=1}^T r(s_t,a_t)$ as our notion of regret to assess the performance of any algorithm $\mathfrak{A}$ for infinite-horizon average-reward MDPs.

\subsection{Discounted-Reward MDP}\label{sec:infinite-discounted}

Given an infinite-horizon MDP $M=(\mathcal{S}, \mathcal{A}, p,r)$, consider a non-stationary  policy $\pi$ given by $\{\pi_t\}_{t=1}^\infty$ where $\pi_t:\{\mathcal{S}\times\mathcal{A}\}^{t-1}\times \mathcal{S}\to \Delta(A)$  samples an action from $\mathcal{A}$ based on history $(s_1,a_1,\ldots, s_{t-1},a_{t-1},s_t)$. Given a discount factor $\gamma\in[0,1)$, we consider the value function and the action-value function defined as $V_t^{\pi}(s)= \mathbb{E}\left[\sum_{i=0}^\infty \gamma^i r(s_{t+i},s_{t+i})\mid s_t=s\right]$ and $Q_t^{\pi}(s,a)= \mathbb{E}\left[\sum_{i=0}^\infty \gamma^i r(s_{t+i},s_{t+i})\mid s_t=s,a_t=a\right]$
for $(s,a)\in\mathcal{S}\times \mathcal{A}$. Note that $V_t^{\pi}(s)$ and $Q_t^{\pi}(s,a)$ capture the infinite-horizon discounted reward under policy $\pi$ from time step $t$. The functions, however, are well-defined only when the probability of the event that $s_t=s$ and $a_t=a$ is positive. Nevertheless, \citep[Appendix A]{zhou-mixture-finite-optimal} provides a slightly more technical definition that avoids the issue and is consistent with the above definition. Furthermore, we define the optimal value function $V^*$ and the optimal action-value function $Q^*$ as $V^*(s)= \max_{\pi} V_1^{\pi}(s)$ and  $Q^*(s,a)= \max_{\pi} Q_1^{\pi}(s,a)$ for $(s,a)\in\mathcal{S}\times\mathcal{A}$. It is known that there exists a deterministic stationary policy $\pi^*$ such that $V_1^{\pi^*}(s)=V^*(s)$ and $Q_1^{\pi^*}(s,a)=Q^*(s,a)$ for $(s,a)\in\mathcal{S}\times\mathcal{A}$~\citep[See][]{puterman2014markov,agarwal21}. Moreover, $V^*$ and $Q^*$ satisfy the following Bellman optimality equation. 
\begin{equation}\label{bellman-discounted}
\begin{aligned}
Q^*(s,a) &= r(s,a) + \gamma\sum_{s'\in\mathcal{S}}p(s'\mid s,a)V^*(s')\\
V^*(s)&=\max_{a\in\mathcal{A}}Q^*(s,a).
\end{aligned}
\end{equation}
For discounted-reward MDPs, following~\citep{liu2021regretboundsdiscountedmdps,pmlr-v139-zhou21a}, we consider $\regret(\pi,T) =\sum_{t=1}^T V^*(s_t) - \sum_{t=1}^TV_t^{\pi}(s_t)$ as our notion of regret of a non-stationary policy $\pi$ for discounted-reward MDPs.

\subsection{Multinomial Logistic Model}

Despite being finite, the state space $\mathcal{S}$ and the action space $\mathcal{A}$ can be intractably large, in which case tabular model-based reinforcement learning algorithms suffer from a large regret. To remedy this, linear and linear mixture MDPs take some structural assumptions on the underlying MDP which lead to efficient learning. However, imposing linearity structures is indeed restrictive and limits the scope of practical applications. Inspired by this issue, we consider the recent framework of MNL function approximation proposed by \cite{HwangOh2023}, assuming that the transition function is given by a feature-based multinomial logistic model as follows. For each $(s, a, s') \in \mathcal{S} \times \mathcal{A} \times \mathcal{S}$, its associated feature vector ${\varphi}(s,a,s') \in \mathbb{R}^d$ is known, and the transition probability is given by
\begin{equation}\label{transition-mnl}
p(s' \mid s, a) := p(s'\mid s,a,\theta^*)
\end{equation}
where 
$$p(s'\mid s,a,\theta):=\frac{\exp\left({\varphi}(s,a,s')^{\top} {\theta} \right)}{\sum_{s'' \in \mathcal{S}_{s,a}} \exp \left({\varphi}(s_t,a_t,s'')^\top \theta \right)}.$$
Here, ${\theta}^* \in \mathbb{R}^d$ is an unknown parameter, which we refer to as the transition core, and $\mathcal{S}_{s,a} := \left\{ s' \in \mathcal{S} : \mathbb{P} (s' \mid s,a) > 0 \right\}$ is the set of reachable states from $s$ in one step after taking action $a$. Let $\mathcal{U} := \max_{(s,a) \in \mathcal{S} \times \mathcal{A}} | \mathcal{S}_{s,a}|$. The general intuition is that the ambient dimension $d$ of the feature vectors and the parameter vector is small compared to the size of $\mathcal{S}$ and that of $\mathcal{A}$. Moreover, it is often the case that $\mathcal{S}_{s,a}$ is small in comparison with $\mathcal{S}$. 

Throughout this paper, we assume the following.
\begin{assumption}\label{ass:L bound}
There exist some $L_{{\varphi}}, L_{{\theta}}$ such that  $\left\| {\varphi} (s,a,s') \right\|_2 \le L_{{\varphi}}$ for all $(s,a,s') \in \mathcal{S} \times \mathcal{A}\times \mathcal{S}$ and $\left\| {\theta}^* \right\|_2 \le L_{{\theta}}$.
\end{assumption}
\noindent
\Cref{ass:L bound} is standard in contextual bandits and RL with function approximation. Let $\Theta=\{\theta\in\mathbb{R}^d: \|\theta\|_2\leq L_\theta\}$.
\begin{assumption} \label{ass:kappa bound}
There exists $\kappa\in(0,1)$ such that we have $\inf_{{\theta}\in\Theta} p_{t,s'}( {\theta} ) p_{t,s''}({\theta}) \ge \kappa$ for all $t\in[T]$ and  $s', s'' \in \mathcal{S}_{s_t,a_t}$.
\end{assumption}
\noindent
\Cref{ass:kappa bound} is also common in the generalized linear contextual bandit literature~\citep{NIPS2010_c2626d85,pmlr-v70-li17c,Oh2019,pmlr-v115-kveton20a,russac2020algorithms} and is taken for RL with MNL function approximation~\citep{HwangOh2023,li2024provably,cho2024randomized}. It guarantees that the associated Fisher information matrix of the log-likelihood function in our setting is non-singular. 
\begin{assumption} \label{ass:recenter}
For every $(s,a)\in \mathcal{S}\times \mathcal{A}$, there exists $s'\in \mathcal{S}_{s,a}$ such that $\varphi(s,a,s')=0$.
\end{assumption}
\noindent
In fact, we may impose~\Cref{ass:recenter} without loss of generality, by the following procedure. For a given pair $(s,a)\in \mathcal{S}\times \mathcal{A}$, we take an arbitrary $s'\in \mathcal{S}_{s,a}$ and replace $\varphi(s,a,s'')$ by $\varphi(s,a,s'')-\varphi(s,a,s')$ for all $s''\in\mathcal{S}_{s,a}$. Note that $\|\varphi(s,a,s'')-\varphi(s,a,s')\|_2\leq 2 L_{\varphi}$ and the probability term $p(s'\mid s,a,\theta)$  remains the same. Therefore, up to doubling the parameter $L_{\varphi}$, Assumptions~\ref{ass:L bound} and \ref{ass:kappa bound} remain valid even after the procedure to enforce \Cref{ass:recenter}.

\section{Algorithm and Regret Bounds}\label{sec:algorithm}

In this section, we present our algorithm, upper-confidence multinomial logistic kernel reinforcement learning (\texttt{UCMNLK} described by \Cref{alg:UCMNLK}). \texttt{UCMNLK} runs extended value iteration on a discounted-reward MDP. For the average-reward MDP, we approximate it by a discounted-reward MDP. \texttt{UCMNLK} is inspired by \texttt{UCLK} by~\cite{pmlr-v139-zhou21a} for infinite-horizon discounted-reward linear mixture MDPs. In contrast to \texttt{UCLK}, however, \texttt{UCMNLK} optimizes over the transition probability $p$, not the parameter vector $\theta$. This is because for our MNL function approximation framework, optimization over $\theta$ is a non-convex problem while optimizing over $p$ is a linear program.

We construct certain confidence polytopes for the true transition probability function based on the recent online Newton method-based technique for estimating the transition probability vector due to~\citep{zhang-sugiyama23,lee2024nearlyminimaxoptimalregret,cho2024randomized,li2024provably}, explained in \Cref{sec:confidence}.

\subsection{Confidence Polytope for the True Transition Function}\label{sec:confidence}

For simplicity, we use shorthand notation $\mathcal{S}_t$ for $\mathcal{S}_{s_t,a_t}$ for $t\in[T]$. We define the transition response variable $y_{t,s'}:=\mathbf{1}\left\{s_{t+1} = s'\right\}$ for $t\in[T]$ and $s'\in \mathcal{S}_{t}$. Here, $y_{t,s'}$ basically corresponds to a sample from the multinomial distribution over $\mathcal{S}_{t}$ with probability $p(s'|s_t,a_t)$. Next, we introduce notation $p_{t,s'}(\theta)$ to denote
$p_{t,s'}(\theta)=p(s'\mid s_t,a_t,\theta).$
Then we have $p(s'\mid s,a,\theta^*)=p(s'\mid s,a)$ and $p_{t,s'}(\theta^*)= p(s'\mid s_t,a_t)$. For each time step $t\in[T]$, we consider a per-time loss function, its gradient, and its Hessian given by
\begin{equation}\label{mle}
\begin{aligned}
\ell_t(\theta) &= - \sum_{s' \in \mathcal{S}_{t}} y_{t,s'} \log p_{t,s'}( {\theta}),\\
\grad_{{\theta}} \left( \ell_{t}({\theta}) \right) 
    &= - \sum_{s' \in \mathcal{S}_{t}} \left( y_{t,s'}  - p_{t,s'}( {\theta}) \right) {\varphi}_{t,s'},\\
\grad^2_\theta(\ell_t(\theta))&=\sum_{s' \in \mathcal{S}_t} p_{t,s'} ({\theta}) {\varphi}_{t, s'} {\varphi}_{t, s'}^{\top}- \sum_{s' \in \mathcal{S}_t} \sum\limits_{s'' \in \mathcal{S}_t} p_{t,s'} ({\theta}) p_{t,s''}({\theta}) {\varphi}_{t, s'} {\varphi}^{\top}_{t, s''},
\end{aligned}
\end{equation}
respectively. We show in \Cref{sec:basic} that the Hessian $\grad_\theta^2(\ell_t(\theta))$ is positive semidefinite for any $\theta\in \Theta$ under \Cref{ass:kappa bound}. 
Motivated by recent progress on online learning frameworks for multinomial logistic bandit~\citep{zhang-sugiyama23}, multinomial logit contextual bandit~\citep{lee2024nearlyminimaxoptimalregret}, and RL with MNL approximation~\citep{li2024provably,cho2024randomized}, we apply the following online algorithm to estimate the true transition core $\theta^*$.  We start with $\widehat \theta_1=0$. At time step $t\in[T]$, given $\widehat \theta_1,\ldots,\widehat \theta_t$, we prepare
\begin{equation}\label{sigma-matrix}
\begin{aligned}
    \Sigma_t = \lambda I_d+\sum_{i=1}^{t-1} \grad^2_\theta(\ell_i(\widehat\theta_{i+1})),\ \widehat \Sigma_t = \Sigma_t + \eta \grad^2_\theta(\ell_t(\widehat \theta_t))
    \end{aligned}
\end{equation}
where $\eta$ is a step size and $I_d$ is the $d\times d$ identity matrix. Note that $\widehat \Sigma_t$ is positive definite. Then we set $\widehat \theta_{t+1}$ to
\begin{equation}\label{update-core}
\begin{aligned}
\argmin_{\theta\in\Theta}\left\{\grad_\theta(\ell_t(\widehat\theta_t))^\top (\theta - \widehat \theta_t) + \frac{1}{2\eta}\|\theta - \widehat\theta_t\|^2_{\widehat\Sigma_t}\right\}.
\end{aligned}
\end{equation}
As~\eqref{update-core} can be viewed as an online mirror descent step with the associated Bregman divergence given by $\|\theta- \vartheta\|_{\widehat \Sigma_t}^2/2$, we may compute $\widehat \theta_{t+1}$ as follows.
$$\widehat \theta_{t+1} = \argmin_{\theta\in\Theta} \left\|\theta - \right(\widehat \theta_t - \eta \widehat\Sigma_t^{-1}\grad_\theta(\ell_t(\widehat \theta_t))\left)\right\|_{\widehat\Sigma_t}$$
The following lemma provides confidence ellipsoids for estimating the transition core $\theta^*$.
\begin{lemma}\label{lem:confidence interval}
Suppose that Assumptions~\ref{ass:L bound}--\ref{ass:recenter} hold. Let $\delta \in (0,1)$, $\eta = (1/2)\log\mathcal{U}+(L_\theta L_\varphi +1)$, and $\lambda \geq 84\sqrt{2}(L_\theta L_\varphi^3 + dL_\varphi^2)\eta$. With probability at least $1-\delta$, $\theta^*$ is contained in
  \begin{equation} \label{eq:confidence-set}\mathcal{C}_{t}:=\left\{\theta\in \Theta:\|\widehat \theta_{t} - \theta^*\|_{\Sigma_{t}}\leq \beta_t\right\}
    \end{equation}
where $\beta_{t} = f(L_\theta,L_\varphi)\sqrt{d}(\log(\mathcal{U}t/\delta))^2$ for every $t\in [T]$ and $f$ is a polynomial in $L_\theta,L_\varphi$.
    \end{lemma}
Based on Lemma~\ref{lem:confidence interval}, we may construct confidence sets for the true transition function. Let $p^*\in\mathbb{R}^{S\times A\times S}$ denote a vector representation of the true transition function. That is, the coordinate $p^*_{s,a,s'}$ of $p^*$ corresponding to $(s,a,s')\in \mathcal{S}\times \mathcal{A}\times \mathcal{S}$ equals $p(s'\mid s,a, \theta^*)$.
\begin{lemma}\label{lem:confidence-polytope}
Suppose that $\theta^*\in \mathcal{C}_t$ where $\mathcal{C}_t$ is defined as in~\eqref{eq:confidence-set}. Let $p^*\in\mathbb{R}^{S\times A\times S}$ be the vector representation of the true transition function. Then for $t\in[T]$,
\begin{equation}\label{eq:confidence-polytope}
p^*\in \mathcal{P}_t:=\left\{p\in [0,1]^{S\times A\times S}: 
    p\text{ satisfies }\eqref{constraint1},\eqref{constraint2}\right\}
\end{equation}
where 
\begin{align}
&\sum_{s'\in\mathcal{S}_{s,a}} p_{s,a,s'}=1,\label{constraint1}\\ 
&\sum_{s'\in\mathcal{S}_{s,a}}\left|p_{s,a,s'}- p(s'\mid s,a,\widehat \theta_t)\right|\leq B_{s,a}^{1,t}+B_{s,a}^{2,t}\label{constraint2}
\end{align}
 with $B_{s,a}^{1,t}=\beta_t\sum_{s'\in \mathcal{S}_{s,a}}p(s'\mid s,a,\widehat \theta_t)\|\varphi(s,a,s')- \sum_{s''\in \mathcal{S}_{s,a}}p(s'\mid s,a,\widehat\theta_t)\varphi(s,a,s'')\|_{\Sigma_t^{-1}}$ and $B_{s,a}^{2,t}=3\beta_t^2 \max_{s'\in\mathcal{S}_{s,a}}\|\varphi(s,a,s')\|_{\Sigma_t^{-1}}^2$ for all $(s,a)\in\mathcal{S}\times\mathcal{A}$.
    \end{lemma}
Note that $\mathcal{P}_t$ defined in Lemma~\ref{lem:confidence-polytope} is a polytope. Hence, one can efficiently optimize a linear function over $\mathcal{P}_t$.

\subsection{Algorithm Description of \texttt{UCMLK}}\label{sec:UCMNLK}

\Cref{alg:UCMNLK} describes \texttt{UCMNLK}. 
As \texttt{UCRL2-VTR}~\citep{yuewu2022} and \texttt{UCLK}~\citep{pmlr-v139-zhou21a}, \texttt{UCMNLK} proceeds with multiple episodes. Each episode consists of the planning phase and the execution phase. 

\begin{algorithm}[tb] 
\renewcommand\thealgorithm{1}
\caption{Upper-Confidence Multinomial Logistic Kernel Reinforcement Learning (\texttt{UCMNLK})}
\label{alg:UCMNLK}
\begin{algorithmic}
\STATE\textbf{Input:} 
 feature map $\varphi:\mathcal{S} \times \mathcal{A} \times \mathcal{S} \rightarrow \mathbb{R}^d$, 
    confidence level $\delta \in (0,1)$, discount factor $\gamma\in[0,1)$, number of rounds $N$, and parameters $\lambda,L_\varphi, L_\theta, \kappa,\mathcal{U}$
\STATE\textbf{Initialize:} 
  $t = 1$, $\widehat{{\theta}}_1 = 0$,
    $\Sigma_1 = \lambda {I_d}$, %
    and observe the initial state $s_1 \in \mathcal{S}$ 
\FOR {episodes $k=1,2,\ldots,$}
\STATE Set $t_k = t$
\STATE Set $Q_k$ as the output of $\texttt{DEVI}(\gamma,\mathcal{P}_{t_k},N)$ where $\mathcal{P}_{t_k}$ is given as in~\eqref{eq:confidence-polytope} 
\STATE Take a deterministic policy $\pi_k$ by taking $\pi_{k}(s)\in \argmax_{a\in\mathcal{A}} Q_k(s,a)$ for $s\in\mathcal{S}$
\WHILE {$\det(\Sigma_t) \le 2 \det(\Sigma_{t_k})$}
\STATE Take action $a_t = \pi_k(s_t)$ \STATE Observe $s_{t+1}$ sampled from $p(\cdot\mid s_t,a_t)$
\STATE Compute $\widehat \theta_{t+1}$ as in~\eqref{update-core}
\STATE Set ${\Sigma}_{t+1} = {\Sigma}_t + \grad^2_\theta(\ell_t(\widehat \theta_{t+1}))$ as in~\eqref{sigma-matrix}
\STATE Update $t\leftarrow t+1$
\ENDWHILE
\ENDFOR
\end{algorithmic}
\end{algorithm}
\begin{algorithm}[!ht]
   \caption*{Discounted Extended Value Iteration (\texttt{DEVI}($\gamma,\mathcal{P},N$))}
\begin{algorithmic}
\STATE \textbf{Inputs:} 
    discount factor $\gamma$, number of rounds $N$, confidence polytope $\mathcal{P}$
\STATE \textbf{Initialize:}  $Q^{(0)}(s,a) =(1-\gamma)^{-1}$ for $(s,a) \in \mathcal{S}\times \mathcal{A}$ 
\FOR {rounds $n=1,2,\ldots,N$}
\STATE Set $V^{(n-1)}(s)=\max_{a\in\mathcal{A}} Q^{(n-1)}(s,a)$ for $s\in\mathcal{S}$ 
\STATE For $(s,a)\in \mathcal{S}\times\mathcal{A}$, set 
\begin{align*}
&Q^{(n)}(s,a)\\
&=r(s,a) + \gamma \max_{p\in\mathcal{P}}\left\{\sum\nolimits_{s'\in\mathcal{S}_{s,a}}p_{s,a,s'}V^{(n-1)}(s')\right\}
\end{align*}
\ENDFOR
\STATE \textbf{Return} $Q^{(N)}(s,a)$ for $(s,a)\in \mathcal{S}\times\mathcal{A}$
\end{algorithmic}
\end{algorithm}

In the planning phase, \texttt{UCMNLK} computes an optimistic policy by running discounted extended value iteration (\texttt{DEVI}) as follows. For the $k$th episode, we denote by $t_k$ the first time step of episode $k$. Before episode $k$ begins, we construct $\widehat \theta_{t_k}$ and $\mathcal{C}_{t_k}$ for estimating $\theta^*$ based on~\eqref{update-core} and~\eqref{eq:confidence-set}. Then we prepare the confidence polytope $\mathcal{P}_{t_k}$ according to~\eqref{eq:confidence-polytope}, over which we run \texttt{DEVI}. Lastly, based on the action-value function $Q_k$ returned by \texttt{DEVI}, we deduce a greedy policy $\pi_k$. 

In the execution phase, \texttt{UCMNLK} applies the policy $\pi_k$ and receives a trajectory with corresponding rewards. Note that \texttt{UCMNLK} switches to the next episode when the determinant of the matrix $\Sigma_t$ doubles compared to the beginning of episode $k$. 

Note that in one round of extended value iteration, we optimize over the probability distributions $p$ in the confidence polytope $\mathcal{P}_{t_k}$. In our case, optimizing over $\theta$ requires maximizing the sum of multinomial logistic functions, which is a non-convex optimization problem. Instead, we maximize over probability distribution $p$, which boils down to solving a linear program.

The following results state our regret upper bounds of \texttt{UCMLK} for the average-reward and the discounted-reward settings.

\begin{theorem}[{\bf Average-Reward}] \label{thm:average-ub}
    Let $M$ be an average-reward MDP governed by the model~\eqref{transition-mnl} with diameter at most $D$. 
    Let $\delta\in(0,1)$, $\eta = (1/2)\log\mathcal{U}+(L_\theta L_\varphi +1)$, $\lambda \geq 84\sqrt{2}(L_\theta L_\varphi^3 + dL_\varphi^2)\eta$, $\gamma = 1-\sqrt{d/DT}$, and $N\geq \sqrt{{DT}/{d}}\log({\sqrt{T}}/{dD})$. Then \texttt{UCMNLK} guarantees that with probability at least $1-2\delta$,
    \begin{equation*}
   \regret(T) =\tilde{\mathcal{O}}\left(f(L_\theta,L_\varphi) \left(dD \sqrt{T}+\kappa^{-1}d^2D\right)\right)
    \end{equation*}
     where $f$ is a polynomial in $(L_\theta,L_\varphi)$ and $\tilde{\mathcal{O}}(\cdot)$ hides logarithmic factors of $T$, $\mathcal{U}$, and $1/\delta$.
    \end{theorem}

\begin{theorem}[{\bf Discounted-Reward}] \label{thm:discounted-ub}
    Let $M$ be a discounted-reward MDP governed by~\eqref{transition-mnl}.
   Let $\delta\in(0,1)$, $\eta = (1/2)\log\mathcal{U}+(L_\theta L_\varphi +1)$, $\lambda \geq 84\sqrt{2}(L_\theta L_\varphi^3 + dL_\varphi^2)\eta$,  and $N\geq \log({\sqrt{T}}/{d})/(1-\gamma)$. 
   Then \texttt{UCMNLK} guarantees that with probability at least $1-2\delta$,
    \begin{align*}
   \regret(\pi, T) 
&=\tilde{\mathcal{O}}\left(f(L_\theta,L_\varphi) \left(d (1-\gamma)^{-2}\sqrt{T}+\kappa^{-1}d^2(1-\gamma)^{-2}\right)\right)
    \end{align*}
      where $f$ is a polynomial in $(L_\theta,L_\varphi)$ and  $\tilde{\mathcal{O}}(\cdot)$ hides logarithmic factors of $T$, $\mathcal{U}$, and $1/\delta$.
    \end{theorem}

\subsection{Regret Analysis of \texttt{UCMNLK}}\label{sec:UCMNLK-analysis}

We denote by $V_k$ the value function for the $k$th episode of \Cref{alg:UCMNLK} given by $V_k(s) = \max_{a\in\mathcal{A}} Q_k(s,a)$ for $s\in\mathcal{S}$. We prove the following lemma establishing convergence of \texttt{DEVI}. \begin{lemma}\label{convergence-devi}
Suppose that $\theta^*\in \mathcal{C}_{t}$ for $t\in[T]$ where $\mathcal{C}_{t}$ is defined as in~\eqref{eq:confidence-set}. Then for each episode $k$ and $t_k\leq t<t_{k+1}-1$, it holds that
\begin{align*}
Q_k(s_t,a_t)
&\leq
r(s_t,a_t) + \gamma \max_{p\in\mathcal{P}_{t_k}}\left\{\sum\nolimits_{s'\in\mathcal{S}_t} p_{s_t,a_t,s'} V_k(s')\right\}+\gamma^N.
\end{align*}
\end{lemma}
Let $K_T$ denote the total number of distinct episodes over the horizon of $T$ time steps. For simplicity, we assume that the last time step of the last episode and that time step $T+1$ is the beginning of the $(K_T+1)$th episode, i.e., $t_{K_T+1} = T+1$. Then it follows from \Cref{convergence-devi} that the regret function for the average-reward case satisfies the following.
\begin{align*}
\begin{aligned}
\regret(T)&= T\cdot J^*-\sum_{t=1}^T r(s_t,a_t) \\
&\leq T\gamma^N+\underbrace{\sum_{k=1}^{K_T}\sum_{t=t_k}^{t_{k+1}-1}\left(J^* -(1-\gamma)V_k(s_{t+1})\right)}_{(a)}+  \underbrace{\sum_{k=1}^{K_T}\sum_{t=t_k}^{t_{k+1}-1}\left(V_k(s_{t+1})-Q_k(s_t,a_t)\right)}_{(b)}\\
&\quad + \underbrace{\gamma\sum_{k=1}^{K_T}\sum_{t=t_k}^{t_{k+1}-1}\left(\sum_{s'\in\mathcal{S}_t}p^*_{s_t,a_t,s'}V_k(s') - V_k(s_{t+1})\right)}_{(c)}\\
&\quad +\underbrace{\gamma\sum_{k=1}^{K_T}\sum_{t=t_k}^{t_{k+1}-1}\max_{p\in\mathcal{P}_{t_k}}\left\{\sum_{s'\in\mathcal{S}_t}\left(p_{s_t,a_t,s'}-p^*_{s_t,a_t,s'}\right)V_k(s')\right\}}_{(d)}.
\end{aligned}
\end{align*}

For regret term $(a)$, recall that $V^*$ and $Q^*$ are the optimal value function and the optimal action-value function for the discounted-reward setting with discount factor $\gamma$. The following lemma proves that $V_k$ and $Q_k$ are optimistic estimators of $V^*$ and $Q^*$, respectively.
\begin{lemma}\label{average-optimism}
Suppose that $\theta^*\in \mathcal{C}_{t}$ for $t\in[T]$ where $\mathcal{C}_{t}$ is defined as in~\eqref{eq:confidence-set}. Then for each episode $k$, $1/(1-\gamma)\geq V_k(s)\geq V^*(s)$ and $1/(1-\gamma)\geq Q_k(s,a)\geq Q^*(s,a)$.
\end{lemma}
\Cref{average-optimism} implies that $J^*-(1-\gamma)V_k(s_{t+1})\leq J^* - (1-\gamma)V^*(s_{t+1})$. Then we apply \citep[Lemma 2,][]{pmlr-v119-wei20c} to argue that $J^* - (1-\gamma)V^*(s_{t+1})\leq (1-\gamma)D$. For regret term $(b)$, note that $V_k(s_{t+1})=Q_k(s_{t+1},a_{t+1})$ for $t\in[t_k:t_{k+1}-2]$, which leads to a telescoping structure. For regret term $(c)$, we first observe that $\sum_{s'\in\mathcal{S}_t}p^*_{s_t,a_t,s'}V_k(s') - V_k(s_{t+1})$ equals $\sum_{s'\in\mathcal{S}_t}p^*_{s_t,a_t,s'}W_k(s') - W_k(s_{t+1})$ where $W_k= V_k - \min_{s'\in\mathcal{S}} V_k(s')$. Then the following lemma implies that $W_k(s)\in[0,D]$ for any $s\in\mathcal{S}$.
\begin{lemma}\label{bounded-span}
Suppose that $\theta^*\in \mathcal{C}_{t}$ for $t\in[T]$ where $\mathcal{C}_{t}$ is defined as in~\eqref{eq:confidence-set}. If the underlying MDP has diameter at most $D$, $\max_{s\in\mathcal{S}}V_k(s)- \min_{s\in\mathcal{S}}V_k(s)\leq D$ for each episode $k$.
\end{lemma}
Based on this, regret term $(c)$ is the sum of a martingale difference sequence where each element has an absolute value at most $D$. Regret term $(d)$ is the cumulative estimation error. Based on \Cref{lem:confidence-polytope,bounded-span}, we can show that $(d)$ is bounded above by $D\sum_{k=1}^{K_T}\sum_{t=t_k}^{t_{k_1}-1}(B_{s_t,a_t}^{1,t_k}+B_{s_t,a_t}^{2,t_k})=\tilde{\mathcal{O}}(dD\sqrt{T})$
which corresponds to the leading term in the regret upper bound of \Cref{thm:average-ub}.

The discounted-reward case is similar to the average-reward case, and its analysis follows the analysis of \texttt{UCLK} due to~\cite{pmlr-v139-zhou21a}.

\begin{figure*}[h!]
\begin{center}
\begin{tikzpicture}[
    scale = 0.9, every node/.style={scale=0.9},
    roundnode/.style={circle, draw=red!60, fill=red!5, thick, minimum size = 10mm},
    squarednode/.style={rectangle, draw=red!60, fill=red!5, thick},
    ]

    \node[roundnode] (x_1) at (-2,1) {$x_1$};
    \node[roundnode] (x_2) at (1.5,1) {$x_{2}$};
    \node [state,draw opacity = 0, fill opacity = 1] (x_4) at (5,1) {$\displaystyle\cdots$};
    \node[roundnode] (x_h) at (8.5,1) {\small$x_{H}$};
    \node[roundnode] (x_h+1) at (12,1) {\small$x_{H+1}$};

    \node[squarednode,rounded corners, minimum width = 130mm, minimum height = 8mm, ] (x_h+2) at (3.5,-3) {$x_{H+2}$};
    \node[squarednode, opacity = 0, minimum size = 8mm] (op1) at (-2,-3) {};
    \node[squarednode, opacity = 0, minimum size = 8mm] (op2) at (1.5,-3) {};
    \node[squarednode, opacity = 0, minimum size = 8mm] (op3) at (5,-3) {};
    \node[squarednode, opacity = 0, minimum size = 8mm] (op4) at (8.5,-3) {};
    \node[squarednode, opacity = 0, minimum size = 10mm] (op5) at (9.5,-3) {};

    \path [-stealth]
        (x_1) edge [bend right] node {} (x_2)
        (x_2) edge [bend right] node {} (x_4)
        (x_4) edge [bend right] node {} (x_h)
        (x_h) edge [bend right] node {} (x_h+1)

        (x_1) edge [bend left] node[below left] {} (op1)
        (x_2) edge [bend left] node {} (op2)
        (x_4) edge [bend left] node {} (op3)
        (x_h) edge [bend left] node {} (op4)
 
        (x_h+1) edge [loop below] node[] {1} ()
        (op5) edge [loop right] node[] {1} ();
    \path [-stealth, dashed]
        (x_1) edge [bend left] node[above] {\small $\frac{(H-1)\exp(-a^\top\theta_1)}{1+(H-1)\exp(-a^\top\theta_1)}$} (x_2)
        (x_2) edge [bend left] node[above] {} (x_4)
        (x_4) edge [bend left] node {} (x_h)
        (x_h) edge [bend left] node[above] {\small $\frac{(H-1)\exp(-a^\top\theta_H)}{1+(H-1)\exp(-a^\top\theta_1)}$} (x_h+1)

        (x_1) edge [bend right] node[below right] {\small $\frac{1}{1+(H-1)\exp(-a^\top\theta_1)}$}  (op1)
        (x_2) edge [bend right] node[below right] {\small $\frac{1}{1+(H-1)\exp(-a^\top\theta_2)}$} (op2)
        (x_4) edge [bend right] node {} (op3)
        (x_h) edge [bend right] node[below right] {\small $\frac{1}{1+(H-1)\exp(-a^\top\theta_H)}$} (op4);

    \draw (-0.25, 1.1) node {$\displaystyle\vdots$};
    \draw (3.25, 1.1) node {$\displaystyle\vdots$};
    \draw (6.75, 1.1) node {$\displaystyle\vdots$};
    \draw (10.25, 1.1) node {$\displaystyle\vdots$};
    
    \draw (-2, -1) node {$\displaystyle\cdots$};
    \draw (1.5, -1) node {$\displaystyle\cdots$};
    \draw (5, -1) node {$\displaystyle\cdots$};
    \draw (8.5, -1) node {$\displaystyle\cdots$};

\end{tikzpicture}%
\caption{Illustration of the Hard Finite-Horizon MDP Instance}\label{fig:finite}
\end{center}
\end{figure*}

\section{Regret Lower Bounds}

In this section, we provide regret lower bounds for learning MDPs with MNL function approximation. \Cref{sec:main-lb-finite} provides a lower bound for learning $H$-horizon episodic MDPs with distinct transition cores over the horizon. \Cref{sec:main-lb-infinite} presents lower bounds for learning infinite-horizon average-reward MDPs with diameter at most $D$ and discounted-reward MDPs with discount factor $\gamma$. 

\subsection{Lower Bound for Learning Finite-Horizon Episodic MDPs}\label{sec:main-lb-finite}

To provide a regret lower bound on learning finite-horizon MDPs with MNL approximation, we consider an instance inspired by~\cite{zhou-mixture-finite-optimal} illustrated as in \cref{fig:finite}.
There are $H+2$ states $x_1,\ldots,x_{H+2}$ where $x_{H+1}$ and $x_{H+2}$ are absorbing states. We have action space $\mathcal{A}=\{-1,1\}^{d-1}$. For any action $a\in\mathcal{A}^{d-1}$, the reward function is given by $f(x_i,a)=1$ if $i=H+2$ and $f(x_i,a)=0$ if $i\neq H+2$. The transition core $\bar \theta_h$ for each step $h\in[H]$ is given by
$$\bar \theta_h = \left(\frac{\theta_h}{\alpha}, \frac{1}{\beta}\right)\quad\text{where}\quad \theta_h\in\left\{-\bar \Delta,\bar\Delta\right\}^{d-1},\ \bar \Delta= \frac{1}{d-1}\log\left(\frac{(1-\delta)(\delta+(d-1)\Delta)}{\delta(1-\delta-(d-1)\Delta)}\right),$$
with $\delta=1/H$, $\Delta=1/(4\sqrt{2HK})$, $\alpha=\sqrt{\bar\Delta/(1+(d-1)\bar\Delta)}$, and $\beta = \sqrt{1/(1+(d-1)\bar\Delta)}$. Moreover, the feature vector is given by
$\varphi(x_h, a, x_{H+2})=(0,0)$ and $\varphi(x_h,a, x_{h+1})=(-\alpha a,\beta \log(H-1))$ for $h\in[H]$. Here, we denote this MDP by $M_{\theta}$ to indicate that it is parameterized by $\theta=\{\theta_h\}_{h=1}^H$.
\begin{theorem}\label{thm:lb-finite}
Suppose that $d\geq 2$, $H\geq 3$, $K\geq \{(d-1)^2H/2,H^3(d-1)^2/32\}$. Then for any algorithm $\mathfrak{A}$, there exists an MDP  $M_\theta$ described as in~\Cref{fig:finite} such that $L_\theta\leq 3/2$ and $L_{\varphi}\leq 1+\log(H-1)$, 
$$\mathbb{E}\left[\regret(M_\theta, \mathfrak{A},K)\right]\geq \frac{(d-1)H^{3/2}\sqrt{K}}{480\sqrt{2}}$$
where the expectation is taken over the randomness generated by $M_\theta$ and $\mathfrak{A}$.
\end{theorem}
Recall that the lower bound provided by~\cite{li2024provably} is $\Omega(dH\sqrt{K\kappa^*})$ where $\kappa^*$ is a constant satisfying $p_t(s',\theta^*)p_t(x'',\theta^*)\geq \kappa^*$ for all $t\in[T]$ and $s',s''\in \mathcal{S}_{s_t,a_t}$. Hence, our lower bound from \Cref{thm:lb-finite} improves the previous lower bound by a factor of $O(\sqrt{H/\kappa^*})$. 

Notice that the instance $M_\theta$ has $L_{\varphi}\leq 1+ \log(H-1)$. Nonetheless, the regret upper bounds by~\cite{HwangOh2023,cho2024randomized,li2024provably} grow polynomially in $L_{\varphi}$, so the upper bounds remain the same up to logarithmic factors in $\log H$.

Let us briefly explain how the lower bound is derived. We consider a multinomial logistic function given by $f:\mathbb{R}\to \mathbb{R}$ as
\begin{equation}\label{eq:lb-mnl}
f(x) = \frac{1}{1+ (\delta^{-1}-1)\exp(-x)}.
\end{equation}
Then it follows that
$$p(x_i\mid x_h,a,\bar\theta_h)=\begin{cases}
f(a^\top \theta_h),&\text{if $i=H+2$}\\
1-f(a^\top \theta_h),&\text{if $i=h+1$}
\end{cases}$$
with $-(d-1)\bar\Delta\leq a^\top\theta_h\leq (d-1)\bar\Delta$ for any $a\in\mathcal{A}$.
One of the main steps to derive the lower bound is to construct an upper bound on the gap between $p(x_i\mid x_h,a,\bar\theta)$ and $p(x_i\mid x_h,a,\bar\theta')$ for $\bar\theta\neq\bar\theta'$. We use the mean value theorem to argue that the gap is bounded above by $c^\top (\theta-\theta')$ for some $c\in\mathbb{R}^{d-1}$. To be more precise, we can show that for any $x,y\in[-(d-1)\bar\Delta,(d-1)\bar\Delta]$ with $x\geq y$, we have
$$0\leq f(x)-f(y)\leq (\delta +(d-1)\Delta) (x-y).$$
This bridges the multinomial logistic function to a linear function. Then we may reduce our analysis to the linear case, and therefore, we may follow some arguments of \cite{zhou-mixture-finite-optimal}.

\subsection{Lower Bounds for Learning Infinite-Horizon MDPs}\label{sec:main-lb-infinite}

In this section, we prove regret lower bounds for learning communicating MDPs of diameter at most $D$ and discounted-reward MDPs with discount factor $\gamma$. Our construction of the following hard-to-learn MDP is motivated by the instance proposed by~\cite{yuewu2022} for the linear mixture MDP case. There are two states $x_0$ and $x_1$ as in \Cref{fig:infinite}.
\begin{figure*}[!ht]
\begin{center}
\begin{tikzpicture}[
    roundnode/.style={circle, draw=red!60, fill=red!5, thick, minimum size = 10mm},
    squarednode/.style={rectangle, draw=red!60, fill=red!5, thick},
    ]
    \node[roundnode] (x_0) at (0,1) {\small$x_{0}$};
    \node[roundnode] (x_1) at (4,1) {\small$x_{1}$};

    \node[roundnode] (x'_0) at (6,1) {\small$x_{0}$};
    \node[roundnode] (x'_1) at (10,1) {\small$x_{1}$};
    \path [-stealth]
        (x_0) edge [] node[above] {\small $\frac{1}{1+(\delta^{-1}-1)\exp(-a^\top\theta)}$} (x_1)
        (x'_1) edge [ ] node[above] {\small $\delta$} (x'_0)
        (x_0) edge [out=210,in=270, loop] node[left] {\small $\frac{(\delta^{-1}-1)\exp(-a^\top \theta)}{1+(\delta^{-1}-1)\exp(-a^\top\theta)}$} ()
        (x'_1) edge [out=330,in=270, loop] node[right] {\small $1-\delta$} ();
\end{tikzpicture}%
\vspace{-0.4in}
\end{center}
\caption{Illustration of the Hard-to-Learn Infinite-Horizon MDP Instance}\label{fig:infinite}
\end{figure*}
The action space is given by $\mathcal{A}=\{-1,1\}^{d-1}$. Let the reward function be given by $r(x_0,a)=0$ and $r(x_1,a)=1$ for any $a\in \mathcal{A}$. Then a higher stationary probability at state $x_1$ means a larger average reward. We set the transition core $\bar\theta$ as
$$\bar \theta = \left(\frac{\theta}{\alpha}, \frac{1}{\beta}\right)\quad\text{where}\quad \theta\in\left\{-\frac{\bar \Delta}{d-1},\frac{\bar\Delta}{d-1}\right\}^{d-1},\ \bar \Delta= \log\left(\frac{(1-\delta)(\delta+\Delta)}{\delta(1-\delta-\Delta)}\right)$$
with $\Delta=(d-1)/(45\sqrt{(2/5)(T/\delta)\log 2})$, $$\delta = \begin{cases}
1/D&\text{for the average-reward case},\\
1-\gamma & \text{for the discounted-reward case},
\end{cases}$$
$\alpha=\sqrt{\bar\Delta/((d-1)(1+\bar\Delta))}$, and $\beta = \sqrt{1/(1+\bar\Delta)}$. The feature vector is given by 
$\varphi(x_0,a,x_0)=(-\alpha a, \beta \log(\delta^{-1}-1))$, $\varphi(x_0,a,x_1)=\varphi(x_1,a,x_0)=(0,0)$, and $\varphi(x_1,a,x_1)=(0, \beta \log(\delta^{-1}-1))$.
We denote this MDP by $M_{\theta}$ to show its dependence on $\theta$.
\begin{theorem}\label{thm:lb-infinite}
Suppose that $d\geq 2$, $D\geq 101$, $T\geq 45(d-1)^2 D$. Then for any algorithm $\mathfrak{A}$, there exists an MDP  $M_\theta$ described as in~\Cref{fig:infinite} such that $L_\theta\leq 100/99$ and $L_{\varphi}\leq 1+\log(D-1)$, 
$$\mathbb{E}\left[\regret(M_\theta, \mathfrak{A},x_0,T)\right]\geq \frac{1}{2025}d\sqrt{DT}$$
where the expectation is taken over the randomness generated by $M_\theta$ and $\mathfrak{A}$.
\end{theorem}
\begin{theorem}\label{thm:lb-infinite-discounted}
Suppose that $d\geq 2$, $\gamma\geq 100/101$, $T\geq 45(d-1)^2/(1-\gamma)$. Then for any policy $\pi$, there exists an MDP  $M_\theta$ described as in~\Cref{fig:infinite} such that $L_\theta\leq 100/99$ and $L_{\varphi}\leq 1+\log(\gamma/(1-\gamma))$, 
\begin{align*}
\mathbb{E}\left[\regret(M_\theta, \mathfrak{A},x_0,T)\right]\geq \frac{\gamma}{3375(1-\gamma)^{3/2}}d\sqrt{T}-\frac{\gamma}{(1-\gamma)^2}
\end{align*}
where the expectation is taken over the randomness generated by $M_\theta$ and $\pi$.
\end{theorem}
\noindent
Note that $L_{\varphi}$ can grow logarithmically in $\delta^{-1}$, which equals $D$ for the average-reward case and $(1-\gamma)^{-1}$ for the discounted-reward setting. Nevertheless, the upper bounds by \Cref{thm:average-ub,thm:discounted-ub} grow polynomially in $L_{\varphi}$. This means that The regret upper bounds on the hard-to-learn MDP remain the same up to additional logarithmic factors in $\delta^{-1}$.

As for the finite-horizon case, our main technique is to connect the MNL function approximation model and linear mixture MDPs. With the multinomial logistic function given in~\eqref{eq:lb-mnl}, we know that $p(x_1\mid x_0, a,\bar\theta)=f(a^\top \theta)$ and $p(x_0\mid x_1,a,\bar\theta) =f(0)$. Moreover, for any $a\in\mathcal{A}$, we have $a^\top\theta\in[-\bar\Delta,\bar\Delta]$. Then we can also argue that for any $x,y\in[-\bar\Delta,\bar\Delta]$ with $x\geq y$, we have
$$0\leq f(x)-f(y)\leq (\delta +\Delta) (x-y).$$
This lays down a bridge between our instance in \Cref{fig:infinite} and the lower bound instance of \cite{yuewu2022}.

\section{Conclusion}

This paper studies infinite-horizon reinforcement learning with multinomial logistic function approximation. We develop an algorithm, \texttt{UCMNLK}, that works for both the average-reward and discounted-reward settings. We provide regret lower bounds for the two settings as well as the finite-horizon setting, which demonstrates that the algorithm achieves tight regret upper bounds. We provide a more comprehensive literature review on online learning of MDPs in the appendix.

\acks{We would like to thank Min-hwan Oh and Taehyun Hwang for helpful discussions. This research is supported by the National Research Foundation of Korea (NRF) grant (No. RS-2024-00350703).}

\bibliography{mybibfile}

\newpage
\appendix

\section{Related Work}

In this section, we provide a more detailed discussion of related work for learning infinite-horizon average-reward Markov decision processes (MDPs).

\paragraph{Infinite-Horizon Average-Reward Tabular MDPs}  \cite{NIPS2008_e4a6222c} initiated the study of online learning of MDPs. They developed an algorithm, UCRL2, that guarantees a regret upper bound of $\widetilde{\mathcal{O}}(DS\sqrt{AT})$ over $T$ time steps where $D$ is the diameter of the underlying MDP, $S$ is the number of states and $A$ is the number of actions. They also provided a regret lower bound of $\Omega(\sqrt{DSAT})$, which shows that UCRL2 is nearly optimal. Then \cite{tewari12} considered the class of weakly communicating MDPs. For the class, they developed an algorithm, REGAL, that guarantees a regret upper bound of $\widetilde{\mathcal{O}}(\mathrm{sp}(v^*)S\sqrt{AT})$ where $\mathrm{sp}(v^*)$ is the span of the optimal associated bias function. After these works, there has been a flurry of activities for closing the gap between regret upper and lower bounds~\citep{Filippi10,pmlr-v83-talebi18a,pmlr-v80-fruit18a,fruit2020improvedanalysisucrl2empirical,pmlr-v119-bourel20a,NEURIPS2019_9e984c10,NIPS2017_3621f145,NIPS2017_51ef186e,pmlr-v97-lazic19a,pmlr-v130-wei21d,pmlr-v195-zhang23b,boone2024achievingtractableminimaxoptimal}. In particular, \cite{NEURIPS2019_9e984c10} developed a nearly minimax optimal algorithm with a regret upper bound of $\widetilde{\mathcal{O}}(\sqrt{\mathrm{sp}(v^*)SAT})$ based on the framework of \cite{fruit2020improvedanalysisucrl2empirical}, but their algorithm is not computationally efficient. Recently,~\cite{boone2024achievingtractableminimaxoptimal} proposed a provably efficient algorithm with a nearly minimax optimal regret upper bound of $\widetilde{\mathcal{O}}(\sqrt{\mathrm{sp}(v^*)SAT})$. 

\paragraph{Reinforcement Learning with Linear Function Approximation} 

There has been notable progress recently in reinforcement learning frameworks that leverage linear function approximations~\citep{jiang17,yang-wang-2019,yang-wang-2020,jin-linear-2020,wang2021optimism,modi-linear-mixture,NEURIPS2018_5f0f5e5f,bilinear,pmlr-v99-sun19a,pmlr-v108-zanette20a,pmlr-v119-zanette20a,pmlr-v119-cai20d,pmlr-v120-jia20a,pmlr-v119-ayoub20a,pmlr-v132-weisz21a,pmlr-v139-zhou21a,zhou-mixture-finite-optimal,pmlr-v139-he21c,NEURIPS2022_ebba182c,Hu2022,He2023,Agarwal23,hong2024provablyefficientreinforcementlearning}. Among these works, the most relevant to our paper are infinite-horizon average-reward linear MDPs and linear mixture MDPs. \cite{pmlr-v130-wei21d} proposed a couple of algorithms for learning infinite-horizon average-reward linear MDPs. First, \texttt{FOPO} is a fixed-point iteration-based algorithm that guarantees a regret upper bound of $\widetilde{\mathcal{O}}(d^{1.5}\mathrm{sp}(v^*)\sqrt{T})$ where $d$ is the dimension of the underlying feature mapping. Although the regret upper bound is currently the best-known upper bound, the algorithm is impractical. They also proposed \texttt{OLSVI.FH} that divides the time horizon into pieces and runs a finite-horizon algorithm for each piece, thereby achieving computational efficiency. However, the regret upper bound of \texttt{OLSVI.FH} is suboptimal. Lastly, they came up with another efficient algorithm \texttt{MDP-EXP2} that guarantees a regret upper bound of order $\widetilde{\mathcal{O}}(\sqrt{T})$, but it applies to only ergodic MDPs. \cite{he2024sampleefficient} studied a general function approximation framework that can be applied to the linear MDP setting. Their algorithm, \texttt{LOOP}, achieves a regret upper bound of $\widetilde{\mathcal{O}}(d^{1.5}\mathrm{sp}(v^*)^{1.5}\sqrt{T})$, but as \texttt{FOPO}, \texttt{LOOP} is not computationally tractable. Recently, \cite{hong2024provablyefficientreinforcementlearning} developed an algorithm that is provably efficient and, at the same time, achieves the best-known upper bound $\widetilde{\mathcal{O}}(d^{1.5}\mathrm{sp}(v^*)\sqrt{T})$. For infinite-horizon average-reward linear mixture MDPs, \cite{yuewu2022} designed an algorithm that is nearly minimax optimal for the class of communicating MDPs. 

\paragraph{Reinforcement Learning with General Function Approximation} 

Reinforcement learning with general function approximation is a recent framework to capture possibly non-linear structures in MDPs, thereby providing an alternative to the linear function approximation framework. \cite{jiang17} studied the Bellman rank as a way of extending the linear class to non-linear function classes. \cite{Wang-Eluder} adopted the notion of eluder dimension due to~\cite{NIPS2013_41bfd20a} for RL with general function approximation. 
\cite{jin2021bellman} proposed the concept of the Bellman eluder (BE) dimension, which combines the eluder dimension and the Bellman error. \cite{bilinear} extended the witness ranking on low-rank structures due to \cite{pmlr-v99-sun19a} and considered the bilinear class. \cite{foster2023statistical} developed a general framework based on the notion of the decision-estimation coefficient. \cite{zhong2023gec} proposed a unified framework with the generalized eluder coefficient (GEC). Recently, \cite{he2024sampleefficient} extended the notion of GEC to the infinite-horizon average-reward setting and introduced the notion of the average-reward generalized eluder coefficient (AGEC).

\paragraph{Reinforcement Learning with Multinomial Logistic Function Approximation} 

\cite{HwangOh2023} initiated the study of the multinomial logistic function approximation framework for learning MDPs. They proposed \texttt{UCRL-MNL} for the finite-horizon setting that achieves a regret upper bound of $\tilde{\mathcal{O}}(\kappa^{-1}d H^{2}\sqrt{K})$ regret where $d$ is the dimension of the transition core, $H$ is the horizon, $K$ is the number of episodes, and $\kappa\in(0,1)$ is a lowed bound of the product of the probability of transitioning to one state and the probability of transitioning to another state. Later, \cite{cho2024randomized} and~\cite{li2024provably} developed algorithms that both guarantee a regret bound of $\tilde{\mathcal{O}}(dH^{2}\sqrt{K}+\kappa^{-1}d^2H^2)$, removing the dependence on $\kappa$ from the previous result. Their algorithms use recently developed online Newton-based parameter estimation methods for logistic bandits due to~\cite{zhang-sugiyama23,lee2024nearlyminimaxoptimalregret}. For the finite-horizon setting, \cite{li2024provably} found a regret lower bound of $\Omega(dH\sqrt{\kappa^* K})$ where $\kappa^*\in(0,1)$ is a quantity similar to $\kappa$. 

\section{Proofs for Section \ref{sec:confidence}}

\subsection{Properties of Multinomial Logistic Function Approximation}\label{sec:basic}

Recall that $\mathcal{S}_t$ denotes $\mathcal{S}_{s_t,a_t}$, the set of reachable states from state $a_t$ in one step after taking action $a_t$. Moreover, the per-time loss function for time step $t\in[T]$, its gradient, and its Hessian are given by
\begin{align*}
\begin{aligned}
\ell_t(\theta) &= - \sum_{s' \in \mathcal{S}_{t}} y_{t,s'} \log p_{t,s'}( {\theta}),\quad \grad_{{\theta}} \left( \ell_{t}({\theta}) \right) 
    = - \sum_{s' \in \mathcal{S}_{t}} \left( y_{t,s'}  - p_{t,s'}( {\theta}) \right) {\varphi}_{t,s'},\\
\grad^2_\theta(\ell_t(\theta))&=\sum_{s' \in \mathcal{S}_t} p_{t,s'} ({\theta}) {\varphi}_{t, s'} {\varphi}_{t, s'}^{\top} - \sum_{s' \in \mathcal{S}_t} \sum\limits_{s'' \in \mathcal{S}_t} p_{t,s'} ({\theta}) p_{t,s''}({\theta}) {\varphi}_{t, s'} {\varphi}^{\top}_{t, s''},
\end{aligned}
\end{align*}
respectively. The following lemma is an immediate consequence of Taylor's theorem. 
\begin{lemma}\label{lemma:taylor}
For any $\theta_1,\theta_2\in \mathbb{R}^d$ and $t\in [T]$, there exists some $\alpha\in[0,1]$ such that $\vartheta:=\alpha \theta_1 + (1-\alpha)\theta_2$ satisfies
$$\ell_t(\theta_2) =\ell_t(\theta_1) +\grad_{{\theta}} \left( \ell_{t}({\theta_1})\right)^\top(\theta_2-\theta_1) + \frac{1}{2}\|\theta_2-\theta_1\|_{\grad^2_\theta(\ell_t(\vartheta))}^2.$$
\end{lemma}

\Cref{ass:recenter} implies that that for each $t\in[T]$, there exists a state $\varsigma_t$ such that $\varphi_{t,\varsigma_t}=0$. This implies that for any $s'\in \mathcal{S}_t$, 
$$p_{t,s'}(\theta)=\frac{\exp \left({\varphi}_{t, s'}^{\top} {\theta}\right) }{\sum\limits_{s'' \in \mathcal{S}_t} \exp \left({\varphi}_{t, s''}^{\top} {\theta}\right)} 
    = \frac{\exp \left({\varphi}_{t, s'}^{\top} {\theta}\right)}{1 + \sum\limits_{s'' \in \mathcal{S}_t \setminus \{ \varsigma_t \}} \exp \left({\varphi}_{t, s''}^{\top} {\theta}\right) }.$$
    \begin{lemma}
    For any $\theta\in \mathbb{R}^d$, we have 
    $$\grad^2_\theta(\ell_t(\theta))\succeq \sum_{s' \in \mathcal{S}_t \setminus \{ \varsigma_t \}} p_{t, \varsigma_t}( {\theta})p_{t,s'}({\theta})  {\varphi}_{t, s'} {\varphi}_{t, s'}^{\top} \succeq \sum_{s' \in \mathcal{S}_t \setminus \{ \varsigma_t \}} \kappa {\varphi}_{t, s'} {\varphi}_{t, s'}^{\top}.$$
    \end{lemma}
    \begin{proof}
   Note that $(x-y)(x-y)^\top = xx^\top + yy^\top - xy^\top - yx^\top\succeq 0$ where $A\succeq B$ means that $A-B$ is positive semidefinite. This implies that ${x} {x}^{\top} + {y} {y}^{\top} \succeq {x} {y}^{\top} + {y} {x}^{\top}$. This implies that
    \begin{align*}
   &\grad^2_\theta(\ell_t(\theta))\\
    &= \sum_{s' \in \mathcal{S}_t} p_{t,s'}( {\theta}) {\varphi}_{t, s'} {\varphi}_{t, s'}^{\top} -\frac{1}{2} \sum_{s' \in \mathcal{S}_t} \sum\limits_{s'' \in \mathcal{S}_t} p_{t,s'}( {\theta}) p_{t,s''}( {\theta}) \left( {\varphi}_{t, s'} {\varphi}^{\top}_{t, s''} + {\varphi}_{t, s''} {\varphi}_{t, s'}^{\top} \right) \\
    &= \sum_{s' \in \mathcal{S}_t\setminus\{\varsigma_t\}} p_{t,s'}( {\theta}) {\varphi}_{t, s'} {\varphi}_{t, s'}^{\top} -\frac{1}{2} \sum_{s' \in \mathcal{S}_t\setminus\{\varsigma_t\}} \sum\limits_{s'' \in \mathcal{S}_t\setminus\{\varsigma_t\}} p_{t,s'}( {\theta}) p_{t,s''}( {\theta}) \left( {\varphi}_{t, s'} {\varphi}^{\top}_{t, s''} + {\varphi}_{t, s''} {\varphi}_{t, s'}^{\top} \right) \\
    &\succeq \sum_{s' \in \mathcal{S}_t\setminus\{\varsigma_t\}} p_{t,s'}( {\theta}) {\varphi}_{t, s'} {\varphi}_{t, s'}^{\top} -\frac{1}{2} \sum_{s' \in \mathcal{S}_t\setminus\{\varsigma_t\}} \sum\limits_{s'' \in \mathcal{S}_t\setminus\{\varsigma_t\}} p_{t,s'}( {\theta}) p_{t,s''}( {\theta}) \left( {\varphi}_{t, s'} {\varphi}_{t, s'}^{\top} + {\varphi}_{t, s''} {\varphi}^{\top}_{t, s''} \right) \\
    &=  \sum_{s' \in \mathcal{S}_t\setminus\{\varsigma_t\}} p_{t,s'}( {\theta}) {\varphi}_{t, s'} {\varphi}_{t, s'}^{\top} - \sum_{s' \in \mathcal{S}_t\setminus\{\varsigma_t\}} \sum\limits_{s'' \in \mathcal{S}_t\setminus\{\varsigma_t\}} p_{t,s'}( {\theta}) p_{t,s''}( {\theta}) {\varphi}_{t, s'} {\varphi}_{t, s'}^{\top}
    \end{align*}
    where the first equality holds because $\varphi_{t,\varsigma_t}=0$. Then it follows that
    \begin{align}\label{eq:psd}
    \begin{aligned}
    \grad^2_\theta(\ell_t(\theta))&\succeq\sum_{s' \in \mathcal{S}_t\setminus\{\varsigma_t\}} \left\{ 1 - \sum\limits_{s'' \in \mathcal{S}_t\setminus\{\varsigma_t\}} p_{t,s''}( {\theta}) \right\} p_{t,s'}( {\theta}) {\varphi}_{t, s'} {\varphi}_{t, s'}^{\top} \\
    &= \sum_{s' \in \mathcal{S}_t \setminus \{ \varsigma_t \}} p_{t, \varsigma_t}({\theta})p_{t,s'}( {\theta})  {\varphi}_{t, s'} {\varphi}_{t, s'}^{\top} \\
    &\succeq \sum_{s' \in \mathcal{S}_t \setminus \{ \varsigma_t \}} \kappa {\varphi}_{t, s'} {\varphi}_{t, s'}^{\top}
    \end{aligned}
    \end{align}
     and the last inequality is from \Cref{ass:kappa bound}. 
\end{proof}

Next we consider
$$p(s'\mid s,a,\theta) = \frac{\exp\left({\varphi}(s,a,s')^{\top} {\theta} \right)}{\sum_{s'' \in \mathcal{S}_{s,a}} \exp \left({\varphi}(s_t,a_t,s'')^\top \theta \right)} = \frac{\exp\left({\varphi}(s,a,s')^{\top} {\theta} \right)}{1+\sum_{s'' \in \mathcal{S}_{s,a}\setminus\{\varsigma_{s,a}\}} \exp \left({\varphi}(s_t,a_t,s'')^\top \theta \right)}$$
where $\varsigma_{s,a}$ is the state where $\varphi(s,a,s')=0$. By \cite[Proposition 1,][]{cho2024randomized}, we deduce that
    \begin{equation}\label{prob-grad}
    \begin{aligned}
    &\grad_{{\theta}} ( p(s'\mid s,a,{\theta}) )\\
    &= p(s'\mid s,a,{\theta}) {\varphi}(s,a,s') - p(s'\mid s,a,{\theta}) \sum\limits_{s'' \in \mathcal{S}_{s,a} \setminus \{ \varsigma_{s,a} \}} p(s''\mid s,a,{\theta}) {\varphi}(s,a,s'')\\
    &= p(s'\mid s,a,{\theta}) {\varphi}(s,a,s') - p(s'\mid s,a,{\theta}) \sum\limits_{s'' \in \mathcal{S}_{s,a} } p(s''\mid s,a, {\theta}) {\varphi}(s,a,s'')
    \end{aligned}
    \end{equation}
    where the second equality holds because $\varphi_{i,\varsigma_i}=0$. Moreover, 
  \begin{equation}\label{prob-hessian}
    \begin{aligned}
    &\grad_{{\theta}}^2 ( p(s'\mid s,a,{\theta}) )\\
    &= p(s'\mid s,a,\theta)\varphi(s,a,s')\varphi(s,a,s')^\top\\
    &\quad - p(s'\mid s,a,\theta)\sum_{s''\in\mathcal{S}_{s,a}}p(s''\mid s,a, \theta)\left(\varphi(s,a,s')\varphi(s,a,s'')^\top+\varphi(s,a,s'')\varphi(s,a,s')^\top\right)\\
    &\quad - p(s'\mid s,a,\theta)\sum_{s''\in\mathcal{S}_{s,a}}p(s''\mid s,a, \theta)\varphi(s,a,s'')\varphi(s,a,s'')^\top\\
    &\quad +2 p(s'\mid s,a,\theta)\left(\sum\limits_{s'' \in \mathcal{S}_{s,a} } p(s''\mid s,a, {\theta}) {\varphi}(s,a,s'')\right)\left(\sum\limits_{s'' \in \mathcal{S}_{s,a} } p(s''\mid s,a, {\theta}) {\varphi}(s,a,s'')\right)^\top.
    \end{aligned}
    \end{equation}

\subsection{Proof of Lemma~\ref{lem:confidence interval}: Concentration of the Transition Core}\label{proof of lem:confidence interval}

In this section, we prove Lemma~\ref{lem:confidence interval} by adapting the proof of \citep[Theorem 3,][]{zhang-sugiyama23}, \citep[Lemma 1,][]{lee2024nearlyminimaxoptimalregret}, \citep[Lemma 3,][]{li2024provably}, and \citep[Lemma 12,][]{cho2024randomized} to our setting of infinite-horizon MDPs. 

Let us consider the second-order Taylor approximation of the per-time loss function $\ell_t$ at $\widehat \theta_t$, given by
$$\widehat \ell_t(\theta) = \ell_t(\widehat \theta_t) + \grad_\theta(\ell_t(\widehat \theta_t))^\top (\theta - \widehat \theta_t) + \frac{1}{2} \|\theta - \widehat \theta_t\|_{\grad^2_\theta(\ell_t(\widehat\theta_t))}^2.$$
Since $\widehat\Sigma_t= \eta \grad^2_\theta(\ell_t(\widehat\theta_t)) + \Sigma_t$, the update rule in~\eqref{update-core} is equivalent to
$$\widehat \theta_{t+1} = \argmin_{\theta\in\Theta}\left\{\widehat \ell_t(\theta) + \frac{1}{2\eta}\|\theta - \widehat\theta_t\|^2_{\Sigma_t}\right\}.$$
\begin{lemma}\label{lee24-lemmaF.1}{\em \citep[Lemma F.1]{lee2024nearlyminimaxoptimalregret}.}
Let $\eta =(1/2) \log\mathcal{U} + (L_\theta L_{\varphi} +1)$. Then 
\begin{equation*}
\begin{aligned}
&\|\widehat \theta_{t+1} - \theta^*\|_{\Sigma_{t+1}}^2\\
&\leq 2\eta \sum_{i=1}^t\left( \ell_i(\theta^*) - \ell_i(\widehat \theta_{i+1})\right)+ 4\lambda L_\theta^2 + 12\sqrt{2}L_\theta L_{\varphi}^3 \eta \sum_{i=1}^t\|\widehat \theta_{i+1} - \widehat \theta_i\|_2^2 - \sum_{i=1}^t \|\widehat \theta_{i+1} - \widehat \theta_i\|_{\Sigma_i}^2.
\end{aligned}
\end{equation*}
\end{lemma}
To prove Lemma~\ref{lem:confidence interval}, we bound the right-hand side of the inequality in Lemma~\ref{lee24-lemmaF.1}. Let $t\in[T]$. For a vector $z\in \mathbb{R}^{|\mathcal{S}_t\setminus\{\varsigma_t\}|}$, we denote by $[z]_{s'}$ the $s'$th coordinate of $z$ for $s'\in \mathcal{S}_t\setminus\{\varsigma_t\}$. Then we define the softmax function $\sigma_t:\mathbb{R}^{|\mathcal{S}_t\setminus\{\varsigma_t\}|}\to \mathbb{R}^{|\mathcal{S}_t\setminus\{\varsigma_t\}|}$ as follows. For $s'\in \mathcal{S}_t\setminus \{\varsigma_t\}$, 
$$[\sigma_t(z)]_{s'} = \frac{\exp([z]_{s'})}{1+\sum_{s''\in \mathcal{S}_t\setminus\{\varsigma_t\}}\exp([z]_{s''})}.$$
For simplicity, we define $$[\sigma_t(z)]_{\varsigma_t}=1- \sum_{s'\in\mathcal{S}_t\setminus \{\varsigma_t\}}[\sigma_t(z)]_{s'}=\frac{1}{1+\sum_{s'\in \mathcal{S}_t\setminus\{\varsigma_t\}}\exp([z]_{s'})},$$
although it is not part of the output of the softmax function $\sigma_t$.
We denote by $\Phi_t\in\mathbb{R}^{d\times |\mathcal{S}_t\setminus\{\varsigma_t\}|}$ the matrix whose columns are $\varphi_{t,s'}\in\mathbb{R}^d$ for $s'\in\mathcal{S}_t\setminus\{\varsigma_t\}$. 
Then $p_{t,s'}(\theta) =[\sigma_t(\Phi_t^\top \theta)]_{s'}$ for $s'\in\mathcal{S}_t\setminus\{\varsigma_t\}$.
Moreover, given $y_t=\{y_{t,s'}:s'\in\mathcal{S}_t\}$, we define 
$$\ell(z, y_t):= \sum_{s'\in \mathcal{S}_t}\mathbf{1}\{y_{t,s'}=1\}\cdot \log\left(\frac{1}{[\sigma_t(z)]_{s'}}\right)\quad \text{for }z\in \mathbb{R}^{|\mathcal{S}_t\setminus\{\varsigma_t\}|}.$$
Then the per-time loss function can be rewritten as
$\ell_t(\theta) = \ell(\Phi_t^\top \theta,y_t).$ Next, we define a pseudo-inverse function $\sigma_t^+:\mathbb{R}^{|\mathcal{S}_t\setminus\{\varsigma_t\}|}\to \mathbb{R}^{|\mathcal{S}_t\setminus\{\varsigma_t\}|}$ of $\sigma_t$ as 
$$[\sigma_t^+(q)]_{s'} = \log\left(\frac{q_{s'}}{1-\|q\|_1}\right)$$
for any $q\in\{p\in[0,1]^{|\mathcal{S}_t\setminus\{\varsigma_t\}|}: \|p\|_1<1\}$. Let $z_t$ be defined as
$$z_t = \sigma_t^{+}\left(\mathbb{E}_{\theta\sim \mathcal{N}_t}\left[\sigma_t(\Phi_t^\top \theta)\right]\right)\quad\text{where}\quad \mathcal{N}_t= \mathcal{N}(\widehat \theta_t, c \Sigma_t^{-1}).$$
Here, $\mathcal{N}_t$ is the Guassian distribution with mean $\widehat \theta_t$ and covariance matrix $c\Sigma_t^{-1}$ where coefficient $c$ is specified later. Having defined $z_t$ for $t\in[T]$, we deduce that
$$\sum_{i=1}^t\left( \ell_i(\theta^*) - \ell_i(\widehat \theta_{i+1})\right) = \underbrace{\sum_{i=1}^t\left( \ell_i(\theta^*) - \ell(z_i, y_i)\right)}_{(a)}+ \underbrace{\sum_{i=1}^t\left( \ell(z_i, y_i) - \ell_i(\widehat \theta_{i+1})\right)}_{(b)}.$$
Term $(a)$ can be bounded based on the following lemma.
\begin{lemma}\label{lee24-lemmaF.2}{\em \citep[Lemma F.2]{lee2024nearlyminimaxoptimalregret}.}
Let $\delta\in(0,1]$ and $\lambda \geq 1$. With probability at least $1-\delta$, for all $t\in[T]$, we have
\begin{align*}
\begin{aligned}
&\sum_{i=1}^t\left( \ell_i(\theta^*) - \ell(z_i, y_i)\right)\\
&\leq (3\log(1+\mathcal{U}t) +2 + L_\theta L_{\varphi})\left(\frac{17}{16}\lambda + 2\sqrt{\lambda} \log\left(\frac{2\sqrt{1+2t}}{\delta}\right)+16\left(\log\left(\frac{2\sqrt{1+2t}}{\delta}\right)\right)^2\right)+2.
\end{aligned}
\end{align*}
\end{lemma}
For term $(b)$, we consider the following lemma.
\begin{lemma}\label{lee24-lemmaF.3}{\em \citep[Lemma F.3]{lee2024nearlyminimaxoptimalregret}.}
For any $c>0$, let $\lambda \geq \max\{2 L_{\varphi}^2, 72cdL_{\varphi}^2\}$. Then, for all $t\in[T]$, we have 
\begin{align*}
\begin{aligned}
&\sum_{i=1}^t\left( \ell(z_i, y_i) - \ell_i(\widehat \theta_{i+1})\right)\leq \frac{1}{2c}\sum_{i=1}^t\|\widehat \theta_i - \widehat \theta_{i+1}\|_{\Sigma_i}^2 + \sqrt{6}cd \log\left(1+ \frac{2tL_{\varphi}^2}{d\lambda}\right).
\end{aligned}
\end{align*}
\end{lemma}
With Lemmas~\ref{lee24-lemmaF.1}--\ref{lee24-lemmaF.3}, we are ready to complete the proof of Lemma~\ref{lem:confidence interval}. It follows from Lemmas~\ref{lee24-lemmaF.1}--\ref{lee24-lemmaF.3} that for $\lambda \geq \max\{2 L_{\varphi}^2, 72cdL_{\varphi}^2\}$,
\begin{equation*}
\begin{aligned}
&\|\widehat \theta_{t+1} - \theta^*\|_{\Sigma_{t+1}}^2\\
&\leq 2\eta (3\log(1+\mathcal{U}t) +2 + L_\theta L_{\varphi})\left(\frac{17}{16}\lambda + 2\sqrt{\lambda} \log\left(\frac{2\sqrt{1+2t}}{\delta}\right)+16\left(\log\left(\frac{2\sqrt{1+2t}}{\delta}\right)\right)^2\right)\\
&\quad +4\eta + 2\eta \sqrt{6}cd \log\left(1+ \frac{2tL_{\varphi}^2}{d\lambda}\right) +4\lambda L_\theta^2 \\
&\quad + 12\sqrt{2}L_\theta L_{\varphi}^3 \eta \sum_{i=1}^t\|\widehat \theta_{i+1} - \widehat \theta_i\|_2^2 +\left(\frac{\eta}{c}-1\right)\sum_{i=1}^t \|\widehat \theta_{i+1} - \widehat \theta_i\|_{\Sigma_i}^2.
\end{aligned}
\end{equation*}
Setting $c= 7\eta/6$ and $\lambda \geq 84\sqrt{2}L_\theta L_\varphi^3\eta$, we have 
\begin{align*}&12\sqrt{2}L_\theta L_{\varphi}^3 \eta \sum_{i=1}^t\|\widehat \theta_{i+1} - \widehat \theta_i\|_2^2 +\left(\frac{\eta}{c}-1\right)\sum_{i=1}^t \|\widehat \theta_{i+1} - \widehat \theta_i\|_{\Sigma_i}^2\\
&\leq 12\sqrt{2}L_\theta L_{\varphi}^3 \eta \sum_{i=1}^t\|\widehat \theta_{i+1} - \widehat \theta_i\|_2^2 -\frac{\lambda}{7}\sum_{i=1}^t \|\widehat \theta_{i+1} - \widehat \theta_i\|_{2}^2\\
&\leq 0.
\end{align*}
Note that
$84\sqrt{2}(L_\theta L_\varphi^3 + dL_\varphi^2)\eta\geq \max\{2 L_{\varphi}^2, 72cdL_{\varphi}^2,84\sqrt{2}L_\theta L_\varphi^3\eta\}$, so we set $\lambda \geq 84\sqrt{2}(L_\theta L_\varphi^3 + dL_\varphi^2)\eta$. As we have
$\eta=(1/2)\log\mathcal{U} + (L_\theta L_\varphi +1)$, we deduce that
$$\|\widehat \theta_{t+1} - \theta^*\|_{\Sigma_{t+1}}\leq C\sqrt{d} (\log(Ut/\delta))^2$$
for some constant $C$ that depends only on $L_\theta, L_\varphi$, as required.

\subsection{Proof of Lemma~\ref{lem:confidence-polytope}: Confidence Polytope for the True Transition Function}\label{appendix:optimistic-value-function}

Since $\left|p(s'\mid s,a,\theta^*)-p(s'\mid s,a,\widehat \theta_t)\right|\leq 1$ for any $s'\in\mathcal{S}_{s,a}$, we may show that
$$\sum_{s'\in \mathcal{S}_{s,a}}\left|p(s'\mid s,a,\theta^*)-p(s'\mid s,a,\widehat \theta_t)\right|\leq B_{s,a}^{1,t}+B_{s,a}^{2,t}$$
for any $(s,a)\in \mathcal{S}\times \mathcal{A}$. By Taylor's theorem, there exists some $\alpha\in[0,1]$ such that $\vartheta=\alpha \widehat \theta_t + (1-\alpha) \theta^*$ satisfies
\begin{align*}
&\sum_{s'\in \mathcal{S}_{s,a}}\left|p(s'\mid s,a,\theta^*)-p(s'\mid s,a,\widehat \theta_t)\right|\\
&=\sum_{s'\in \mathcal{S}_{s,a}}\left|\grad_\theta(p(s'\mid s,a, \widehat \theta_t))^\top (\theta^* - \widehat \theta_t) +\frac{1}{2}\|\theta^*-\widehat \theta_t\|_{\grad_\theta^2(p(s'\mid s,a, \vartheta))}^2\right|\\
&\leq\underbrace{\sum_{s'\in \mathcal{S}_{s,a}}\left|\grad_\theta(p(s'\mid s,a, \widehat \theta_t))^\top (\theta^* - \widehat \theta_t) \right|}_{(a)} +\underbrace{\frac{1}{2}\sum_{s'\in\mathcal{S}_{s,a}}\|\theta^*-\widehat \theta_t\|_{\grad_\theta^2(p(s'\mid s,a, \vartheta))}^2}_{(b)}.
\end{align*}
Term (a) can be bounded as follows.
\begin{align*}
&\sum_{s'\in \mathcal{S}_{s,a}}\left|\grad_\theta(p(s'\mid s,a, \widehat \theta_t))^\top (\theta^* - \widehat \theta_t)\right|\\
&=\sum_{s'\in\mathcal{S}_{s,a}}\left|p(s'\mid s,a,{\widehat\theta_t})\left( {\varphi}(s,a,s') -  \sum_{s'' \in \mathcal{S}_{s,a} } p(s''\mid s,a, {\widehat\theta_t}) {\varphi}(s,a,s'')\right)^\top(\theta^*- \widehat\theta_t) \right|\\
&\leq \sum_{s'\in\mathcal{S}_{s,a}}p(s'\mid s,a,{\widehat\theta_t})\left\| {\varphi}(s,a,s') -  \sum_{s'' \in \mathcal{S}_{s,a} } p(s''\mid s,a, {\widehat\theta_t}) {\varphi}(s,a,s'')\right\|_{\Sigma_t^{-1}}\left\|\theta^*- \widehat\theta_t\right\|_{\Sigma_t} \\
&\leq \beta_t\sum_{s'\in\mathcal{S}_{s,a}}p(s'\mid s,a,{\widehat\theta_t})\left\| {\varphi}(s,a,s') -  \sum_{s'' \in \mathcal{S}_{s,a} } p(s''\mid s,a, {\widehat\theta_t}) {\varphi}(s,a,s'')\right\|_{\Sigma_t^{-1}}
\end{align*}
where the equality is due to~\eqref{prob-grad}, the first inequality is by the Cauchy-Schwarz inequality, and the second inequality holds because we assumed that $\theta^*\in \mathcal{C}_t$. Hence, term $(a)$ is bounded above by $B_{s,a}^{1,t}$. For term $(b)$, note that
\begin{equation*}
    \begin{aligned}
    &\sum_{s'\in\mathcal{S}_{s,a}}\|\theta^*-\widehat \theta_t\|_{\grad_\theta^2(p(s'\mid s,a, \vartheta))}^2\\
    &=\sum_{s'\in\mathcal{S}_{s,a}}\left|(\widehat \theta_t-\theta^*)^\top \grad_{{\theta}}^2 ( p(s'\mid s,a,{\vartheta}) )(\widehat \theta_t-\theta^*)\right|\\
    &\leq \sum_{s'\in\mathcal{S}_{s,a}}p(s'\mid s,a,\vartheta)\Bigg[\left((\widehat \theta_t-\theta^*)^\top\varphi(s,a,s')\right)^2\\
    &\quad +\sum_{s''\in\mathcal{S}_{s,a}}p(s''\mid s,a, \vartheta)\left|2\left((\widehat \theta_t-\theta^*)^\top\varphi(s,a,s')\right)\left((\widehat \theta_t-\theta^*)^\top\varphi(s,a,s'')\right) \right|\\
    &\quad + \sum_{s''\in\mathcal{S}_{s,a}}p(s''\mid s,a, \vartheta)\left((\widehat \theta_t-\theta^*)^\top\varphi(s,a,s'')\right)^2\\
    &\quad \left.+2\left((\widehat \theta_t-\theta^*)^\top \left(\sum\limits_{s'' \in \mathcal{S}_{s,a} } p(s''\mid s,a, {\vartheta}) {\varphi}(s,a,s'')\right)\right)^2\right]
    \end{aligned}
    \end{equation*}
    where the inequality follows from~\eqref{prob-hessian}.
Note that for any $s'\in\mathcal{S}_{s,a}$, the Cauchy-Schwarz inequality implies 
$$(\widehat \theta_t-\theta^*)^\top\varphi(s,a,s')\leq \|\widehat \theta_t-\theta^*\|_{\Sigma_t}\|\varphi(s,a,s')\|_{\Sigma_t^{-1}}\leq \beta_t\|\varphi(s,a,s')\|_{\Sigma_t^{-1}}$$
where the second inequality holds because $\theta^*\in\mathcal{C}_t$.
Then we deduce that
\begin{equation*}
    \begin{aligned}
    &\sum_{s'\in\mathcal{S}_{s,a}}\|\theta^*-\widehat \theta_t\|_{\grad_\theta^2(p(s'\mid s,a, \vartheta))}^2\\
    &\leq 4\beta_t^2\sum_{s'\in\mathcal{S}_{s,a}}p(s'\mid s,a,\vartheta)\|\varphi(s,a,s')\|_{\Sigma_t^{-1}}^2+2\beta_t^2\left(\sum\limits_{s'' \in \mathcal{S}_{s,a} } p(s''\mid s,a, {\vartheta})\|\varphi(s,a,s'')\|_{\Sigma_t^{-1}}\right)^2\\
    &\leq 6\beta_t^2 \max_{s'\in\mathcal{S}_{s,a}}\|\varphi(s,a,s')\|_{\Sigma_t^{-1}}^2.
    \end{aligned}
    \end{equation*}
    Therefore, term $(b)$ is bounded above by $B_{s,a}^{2,t}$, as required.

We also prove the following lemma which will be useful for our analysis.
\begin{lemma}\label{lem:confidence-polytope'}
Suppose that $\theta^*\in \mathcal{C}_t$ where $\mathcal{C}_t$ is defined as in~\eqref{eq:confidence-set}. Let $p^*\in\mathbb{R}^{S\times A\times S}$ be the vector representation of the true transition function. Then for $t\in[T]$,
\begin{equation}\label{eq:confidence-polytope'}
p^*\in \mathcal{P}_t':=\left\{p\in [0,1]^{S\times A\times S}: 
    p\text{ satisfies }\eqref{constraint1'},\eqref{constraint2'}\right\}
\end{equation}
where 
\begin{align}
&\sum_{s'\in\mathcal{S}_{s,a}} p_{s,a,s'}=1,\label{constraint1'}\\ 
&\sum_{s'\in\mathcal{S}_{s,a}}\left|p_{s,a,s'}- p(s'\mid s,a,\widehat \theta_t)\right|\leq R_{t,s,a}\label{constraint2'}
\end{align}
with $R_{t,s,a} = 2\beta_t \max\limits_{s'\in\mathcal{S}_{s,a}}\| {\varphi}(s,a,s')\|_{\Sigma_{t}^{-1}}$ for all $(s,a)\in\mathcal{S}\times\mathcal{A}$.
    \end{lemma}
    \begin{proof}
Since $\left|p(s'\mid s,a,\theta^*)-p(s'\mid s,a,\widehat \theta_t)\right|\leq 1$ for any $s'\in\mathcal{S}_{s,a}$, it is sufficient to show that
$$\sum_{s'\in \mathcal{S}_{s,a}}\left|p(s'\mid s,a,\theta^*)-p(s'\mid s,a,\widehat \theta_t)\right|\leq 2\beta_t \max_{s'\in\mathcal{S}_{s,a}}\| {\varphi}(s,a,s')\|_{\Sigma_{t}^{-1}}$$
for any $(s,a)\in \mathcal{S}\times \mathcal{A}$. By Taylor's theorem, for any $s'\in\mathcal{S}_{s,a}$, there exists some $\alpha_{s'}\in[0,1]$ such that $\vartheta_{s'}=\alpha_{s'} \widehat \theta_t + (1-\alpha_{s'}) \theta^*$ satisfies
\begin{align*}
&\sum_{s'\in \mathcal{S}_{s,a}}\left|p(s'\mid s,a,\theta^*)-p(s'\mid s,a,\widehat \theta_t)\right|\\
&=\sum_{s'\in \mathcal{S}_{s,a}}\left|\grad_\theta(p(s'\mid s,a, \vartheta_{s'}))^\top (\theta^* - \widehat \theta_t) \right|\\
&=\sum_{s'\in\mathcal{S}_{s,a}}\left|p(s'\mid s,a,{\vartheta_{s'}})\left( {\varphi}(s,a,s') -  \sum_{s'' \in \mathcal{S}_{s,a} } p(s''\mid s,a, {\vartheta_{s'}}) {\varphi}(s,a,s'')\right)^\top(\theta^*- \widehat\theta_t) \right|\\
&\leq \sum_{s'\in\mathcal{S}_{s,a}}p(s'\mid s,a,{\vartheta_{s'}})\left\|{\varphi}(s,a,s') -  \sum_{s'' \in \mathcal{S}_{s,a} } p(s''\mid s,a, {\vartheta_{s'}}) {\varphi}(s,a,s'')\right\|_{\Sigma_{t}^{-1}}\left\|\theta^*- \widehat\theta_t\right\|_{\Sigma_t}\\
&\leq \beta_t\sum_{s'\in\mathcal{S}_{s,a}}p(s'\mid s,a,{\vartheta_{s'}})\left(\|{\varphi}(s,a,s')\|_{\Sigma_{t}^{-1}} +  \sum_{s'' \in \mathcal{S}_{s,a} } p(s''\mid s,a, {\vartheta_{s'}})\| {\varphi}(s,a,s'')\|_{\Sigma_{t}^{-1}}\right)\\
&\leq  2\beta_t \max_{s'\in\mathcal{S}_{s,a}}\| {\varphi}(s,a,s')\|_{\Sigma_{t}^{-1}}.
\end{align*}
where the second equality is due to~\eqref{prob-grad}, the first inequality is by the Cauchy-Schwarz inequality, and the second inequality holds because we assumed that $\theta^*\in \mathcal{C}_t$.
\end{proof}

\section{Proof of Theorem \ref{thm:average-ub}: Performance Analysis of \texttt{UCMNLK} for the Average-Reward Setting}\label{sec:UCMNLK-proof-average}

Let $K_T$ denote the total number of distinct episodes over the horizon of $T$ time steps. For simplicity, we assume that the last time step of the last episode and that time step $T+1$ is the beginning of the $(K_T+1)$th episode, i.e., $t_{K_T+1} = T+1$. Then we have 
\begin{align}\label{eq:average-regret-decomposition}
\begin{aligned}
\regret(T)&=\sum_{k=1}^{K_T}\sum_{t=t_k}^{t_{k+1}-1}\left(J^*-r(s_t,a_t)\right)\\
&\leq \sum_{k=1}^{K_T}\sum_{t=t_k}^{t_{k+1}-1}\left(J^*-Q_k(s_t,a_t) +\gamma \max_{p\in\mathcal{P}_{t_k}}\left\{\sum_{s'\in\mathcal{S}_t} p_{s_t,a_t,s'} V_k(s')\right\}+\gamma^N\right)\\
&=\underbrace{\sum_{k=1}^{K_T}\sum_{t=t_k}^{t_{k+1}-1}\left(J^* -(1-\gamma)V_k(s_{t+1})\right)}_{(a)} + \underbrace{\sum_{k=1}^{K_T}\sum_{t=t_k}^{t_{k+1}-1}\left(V_k(s_{t+1})-Q_k(s_t,a_t)\right)}_{(b)}\\
&\quad + \underbrace{\gamma\sum_{k=1}^{K_T}\sum_{t=t_k}^{t_{k+1}-1}\left(\sum_{s'\in\mathcal{S}_t}p^*_{s_t,a_t,s'}V_k(s') - V_k(s_{t+1})\right)}_{(c)}\\
&\quad+\underbrace{\gamma\sum_{k=1}^{K_T}\sum_{t=t_k}^{t_{k+1}-1}\max_{p\in\mathcal{P}_{t_k}}\left\{\sum_{s'\in\mathcal{S}_t}\left(p_{s_t,a_t,s'}-p^*_{s_t,a_t,s'}\right)V_k(s')\right\}}_{(d)}+T\gamma^N.
\end{aligned}
\end{align}
where the inequality comes from Lemma~\ref{convergence-devi}. We provide upper bounds on terms $(a)$--$(d)$ in \Cref{sec:term-a,sec:term-b,sec:term-c,sec:term-d}. Based on them, we prove a regret upper bound on \texttt{UCMNLK} for the average-reward setting in \Cref{sec:term-final}.

\subsection{Term $(a)$}\label{sec:term-a}

Recall that $J^*$ is the optimal average reward of the MDP. It is known that any communicating MDP satisfies the following Bellman optimality condition~\cite[See][]{puterman2014markov}.
The condition states that there exist $v^*:\mathcal{S}\to\mathbb{R}$ and $q^*:\mathcal{S}\times\mathcal{A}\to\mathbb{R}$ such that for all $(s,a)\in\mathcal{S}\times \mathcal{A}$,
\begin{equation}\label{bellman-average}
J^* + q^*(s,a)= r(s,a) + \sum_{s'\in\mathcal{S}}p(s'\mid s,a)v^*(s')\quad\text{and}\quad v^*(s)=\max_{a\in\mathcal{A}}q^*(s,a).
\end{equation}
Here, $v^*:\mathcal{S}\to\mathbb{R}$ is referred to as the optimal bias function. For any function $h:\mathcal{S}\to\mathbb{R}$, we define its span as $\mathrm{sp}(h):=\max_{s\in\mathcal{S}}h(s) - \min_{s\in\mathcal{S}}h(s)$. In particular, it is known that for a communicating MDP with diameter at most $D$, we have $\mathrm{sp}(v^*)\leq D$~\citep{Auer_Jaksch2010,puterman2014markov}. Furthermore, we also have the following lemma that compares the span of the optimal discounted value function $V^*$ and that of the optimal bias function $v^*$.
\begin{lemma}\label{wei-lemma2}{\em \citep[Lemma 2]{pmlr-v119-wei20c}}
For any $\gamma\in[0,1)$, $\mathrm{sp}(V^*)\leq 2 \cdot\mathrm{sp}(v^*)$ and $|(1-\gamma)V^*(s) - J^*|\leq (1-\gamma) \mathrm{sp}(v^*)$ for all $s\in\mathcal{S}$.
\end{lemma}

We can provide an upper bound on term $(a)$ using Lemma~\ref{wei-lemma2}. Note that for any $t$,
\begin{align*}
J^*- (1-\gamma)V_k(s_{t+1})\leq J^*-(1-\gamma)V^*(s_{t+1})\leq (1-\gamma)\mathrm{sp}(v^*)\leq (1-\gamma)D
\end{align*}
where the first inequality follows from Lemma~\ref{average-optimism} and the second inequality is due to Lemma~\ref{wei-lemma2}. Therefore, it follows that
\begin{equation}\label{eq:average-a}
\text{Term $(a)$}=\sum_{k=1}^{K_T}\sum_{t=t_k}^{t_{k+1}-1}\left(J^* -(1-\gamma)V_k(s_{t+1})\right)\leq T(1-\gamma)D.
\end{equation}

\subsection{Term $(b)$}\label{sec:term-b}

Note that term $(b)$ can be rewritten as follows.
\begin{align*}
&\sum_{k=1}^{K_T}\sum_{t=t_k}^{t_{k+1}-1}\left(V_k(s_{t+1})-Q_k(s_t,a_t)\right)\\
&=\sum_{k=1}^{K_T}\sum_{t=t_k}^{t_{k+1}-2}\left(Q_k(s_{t+1},a_{t+1})-Q_k(s_t,a_t)\right)+ \sum_{k=1}^{K_T}\left(V_k(s_{t_{k+1}}) - Q_k(s_{t_{k+1}-1},a_{t_{k+1}-1})\right)\\
&=\sum_{k=1}^{K_T}\left(V_k(s_{t_{k+1}}) - Q_k(s_{t_k},a_{t_k})\right).
\end{align*}
Here, for any $k\geq 1$, we know from Lemma~\ref{average-optimism} that
$$V_k(s_{t_{k+1}}) - Q_k(s_{t_k},a_{t_k})\leq V_k(s_{t_{k+1}})\leq \frac{1}{1-\gamma}.$$
Therefore, term $(b)$ can be bounded above as
\begin{equation}\label{eq:average-b}
\text{Term $(b)$}=\sum_{k=1}^{K_T}\sum_{t=t_k}^{t_{k+1}-1}\left(V_k(s_{t+1})-Q_k(s_t,a_t)\right)\leq \frac{K_T}{1-\gamma}.
\end{equation}

\subsection{Term $(c)$}\label{sec:term-c}

For $t\in[T]$, recall that $p^*_{s_t,a_t,s'}$ denotes the true transition probability of transitioning to state $s'$ given that the state-action pair for $t$ is $(s_t,a_t)$. Take $Y_t$ as $$Y_t=\sum_{s'\in \mathcal{S}_t}p^*_{s_t,a_t,s'}V_{k}(s') -V_{k}(s_{t+1})$$ 
for $t_k\leq t\leq t_{k+1}-1$ and $k\in[K_T]$. For $t\in [T]$, let $\mathcal{F}_t$ be the $\sigma$-algebra generated by the randomness up to time step $t$. Then we have $\mathbb{E}\left[Y_t\mid \mathcal{F}_t\right]=0$, which means that $Y_1,\ldots, Y_T$ gives rise to a martingale difference sequence. Then $\text{Term $(c)$}$, which is essentially the summation of $Y_1,\ldots, Y_T$, can be bounded by Azuma's inequality given as follows.
\begin{lemma}[Azuma's inequality] \label{lem:azuma}
Let $Y_1,\ldots, Y_T$ be a martingale difference sequence with respect to a filtration $\mathcal{F}_1,\ldots, \mathcal{F}_T$. Assume that $|Y_t|\leq B$ for $t\in[T]$. Then with probability at least $1-\delta$, we have
$\sum_{t=1}^T Y_t\leq B\sqrt{2T\log\left({1}/{\delta}\right)}.$
\end{lemma}
To apply Azuma's inequality, we need a global bound on $Y_t$ terms. Note that
$$|Y_t|\leq \sum_{s' \in \mathcal{S}_t} p^*_{s_t,a_t,s'}\left|V_k(s') - V_k(s_{t+1})\right|\leq  D$$
where the last inequality follows from Lemma~\ref{bounded-span}. Then it follows from Lemma \ref{lem:azuma} that with probability at least $1-\delta$, 
\begin{equation}\label{eq:average-c}\text{Term $(c)$}=\gamma\sum_{k=1}^{K_T}\sum_{t=t_k}^{t_{k+1}-1}\left(\sum_{s'\in\mathcal{S}_t}p^*_{s_t,a_t,s'}V_k(s') - V_k(s_{t+1})\right)\leq\sum_{t=1}^TY_T\leq D\sqrt{2T\log (1/\delta)}.
\end{equation}

\subsection{Term $(d)$}\label{sec:term-d}

Let $p\in \mathcal{P}_{t_k}$. 
Then we have
\begin{align*}
\sum_{s'\in\mathcal{S}_t}\left(p_{s_t,a_t,s'}-p^*_{s_t,a_t,s'}\right)V_k(s')&=\sum_{s'\in\mathcal{S}_t}\left(p_{s_t,a_t,s'}-p^*_{s_t,a_t,s'}\right)\left(V_k(s')-\min_{s''\in\mathcal{S}}V_k(s'')\right)\\
&\leq \sum_{s'\in\mathcal{S}_t}\left|p_{s_t,a_t,s'}-p^*_{s_t,a_t,s'}\right|\mathrm{sp}(V_k)\\
&\leq D\sum_{s'\in\mathcal{S}_t}\left|p_{s_t,a_t,s'}-p^*_{s_t,a_t,s'}\right|\\
&\leq D \sum_{s'\in\mathcal{S}_t}\left(\left|p_{s_t,a_t,s'}-p_t(s',\widehat \theta_{t_k})\right|+\left|p_t(s',\widehat \theta_{t_k})-p^*_{s_t,a_t,s'}\right|\right)\\
&\leq 2D \left(B_{s_t,a_t}^{1,{t_k}}+B_{s_t,a_t}^{2,{t_k}}\right)
\end{align*}
where the equality holds because $\sum_{s'\in\mathcal{S}_t}(p_{s_t,a_t,s'}-p^*_{s_t,a_t,s'})=0$, the second inequality follows from Lemma~\ref{bounded-span}, and the last inequality is implied by Lemma~\ref{lem:confidence-polytope} as $p,p^*\in\mathcal{P}_{t_k}$. Therefore, it follows that
\begin{align*}
\text{Term $(d)$}\leq 2D\underbrace{\sum_{k=1}^{K_T}\sum_{t=t_k}^{t_{k+1}-1}B_{s_t,a_t}^{1,{t_k}}}_{(\star)}+2D\underbrace{\sum_{k=1}^{K_T}\sum_{t=t_k}^{t_{k+1}-1}B_{s_t,a_t}^{2,{t_k}}}_{(\star\star)}.
\end{align*}
We may provide an upper bound on the right-most side of this inequality based on the following lemma.
\begin{lemma}\label{lem:term-star}
Suppose that $\theta^*\in\mathcal{C}_t$ for all $t\in[T]$. Let $\lambda \geq L_{\varphi}^2$. Then
$$\sum_{k=1}^{K_T}\sum_{t=t_k}^{t_{k+1}-1}B_{s_t,a_t}^{1,t_k}\leq \beta_T\left(\left(\frac{32\beta_T}{\kappa}+\frac{128\sqrt{2} L_{\varphi}\eta}{\kappa\sqrt{\lambda}} \right)d\log\left(1+\frac{T\mathcal{U}L_{\varphi}^2}{\lambda d}\right)+ 2\sqrt{dT\log\left(1+\frac{T\mathcal{U}L_{\varphi}^2}{\lambda d}\right)}\right).$$
\end{lemma}
\begin{proof}
See \Cref{sec:error-cumulative}.
\end{proof}

\begin{lemma}\label{lem:term-starstar}
Let $\lambda \geq L_{\varphi}^2$. Then
$$\sum_{k=1}^{K_T}\sum_{t=t_k}^{t_{k+1}-1}B_{s_t,a_t}^{2,t_k}\leq \frac{12d}{\kappa}\beta_T^2 \log\left(1+\frac{T\mathcal{U}L_{\varphi}^2}{\lambda d}\right).$$
\end{lemma}
\begin{proof}
See \Cref{sec:error-cumulative}.
\end{proof}

Hence, we deduce that
\begin{equation}\label{eq:average-d}
\text{Term $(d)$}\leq \left(\frac{88dD}{\kappa}\beta_T^2+ \frac{256\sqrt{2} L_{\varphi}\eta dD}{\kappa\sqrt{\lambda}} \beta_T\right)\log\left(1+\frac{T\mathcal{U}L_{\varphi}^2}{\lambda d}\right) + 4D\beta_T\sqrt{dT\log\left(1+\frac{T\mathcal{U}L_{\varphi}^2}{\lambda d}\right)}.
\end{equation}

\subsection{Completing the Regret Bound for the Average-Reward Case}\label{sec:term-final}

Combining~\eqref{eq:average-a},~\eqref{eq:average-b},~\eqref{eq:average-c}, and~\eqref{eq:average-d}, we deduce that
\begin{align*}
&\regret(T)\\
&\leq T(1-\gamma)D + \frac{K_T}{1-\gamma}+ D\sqrt{2T\log(1/\delta))}+ T\gamma^N\\
&\quad + \left(\frac{88dD}{\kappa}\beta_T^2+ \frac{256\sqrt{2} L_{\varphi}\eta dD}{\kappa\sqrt{\lambda}} \beta_T\right)\log\left(1+\frac{T\mathcal{U}L_{\varphi}^2}{\lambda d}\right) + 4D\beta_T\sqrt{dT\log\left(1+\frac{T\mathcal{U}L_{\varphi}^2}{\lambda d}\right)}
\end{align*}
Here, $K_T$ can be bounded by the following lemma.
\begin{lemma} \label{lem:Bound of the number of episodes}
$K_T\leq 1+d\log_2\left(1 + 2TL_{\varphi}^2/\lambda\right)$.
\end{lemma}
\begin{proof}
        Since $\Sigma_1 = \lambda I_d$, we have $\det ({\Sigma}_1)= \lambda^d$. Note that for any $\theta$ and $t$,
\begin{align*}\left\|\grad^2_\theta(\ell_t(\theta))\right\|_2&\leq \sum_{s' \in \mathcal{S}_t} p_{t,s'} ({\theta})\left\| {\varphi}_{t, s'} {\varphi}_{t, s'}^{\top} \right\|_2 +\sum_{s' \in \mathcal{S}_t} \sum\limits_{s'' \in \mathcal{S}_t} p_{t,s'} ({\theta}) p_{t,s''}({\theta}) \left\|{\varphi}_{t, s'} {\varphi}^{\top}_{t, s''}\right\|_2\\
&\leq \sum_{s' \in \mathcal{S}_t} p_{t,s'} ({\theta})\left\| {\varphi}_{t, s'}\right\|_2^2+\sum_{s' \in \mathcal{S}_t} \sum\limits_{s'' \in \mathcal{S}_t} p_{t,s'} ({\theta}) p_{t,s''}({\theta}) \left\|{\varphi}_{t, s'}\right\|_2 \left\|{\varphi}_{t, s''}\right\|_2\\
&\leq 2L_{\varphi}^2
\end{align*}
where the first two inequalities are by the triangle inequality and the last is due to \Cref{ass:L bound}.
Then it follows that
    \begin{align*}
    \left\| {\Sigma}_T \right\|_2 
    = \left\|\lambda I_d+\sum_{t=1}^{T-1} \grad^2_\theta(\ell_t(\widehat\theta_{t+1}))\right\|_2\leq \lambda + 2T L_{\varphi}^2.
    \end{align*}
This implies that $\det({\Sigma}_T) \le (\lambda + 2TL_{{\varphi}}^2)^d$.
Therefore, we have
    \begin{align} \label{eq:Bound of episode K}
    (\lambda + 2T L_{{\varphi}}^2)^d \ge \det({\Sigma}_T) \ge \det ({\Sigma}_{t_{K_T}}) \ge 2^{K_T-1} \det ({\Sigma}_{t_1}) = 2^{K_T-1} \lambda^d,
    \end{align}
where the second inequality holds because ${\Sigma}_T \succeq {\Sigma}_{t_{K_T}}$ and the last  holds due to $\det(\Sigma_{t_{k+1}})\geq 2\det(\Sigma_{t_k})$. Then it follows from \eqref{eq:Bound of episode K} that $K_T\leq 1+d\log_2\left(1 + 2TL_{\varphi}^2/\lambda\right)$, as required.
    \end{proof}

Setting $\gamma$ and $N$ as
$$\gamma = 1- \sqrt{\frac{d}{DT}}\quad\text{and}\quad N = \frac{1}{1-\gamma}\log\left(\frac{\sqrt{T}}{dD}\right)=\sqrt{\frac{DT}{d}}\log\left(\frac{\sqrt{T}}{dD}\right),$$ 
we have $$N\geq \frac{\log\left({\sqrt{T}}/{dD}\right)}{\log(1/\gamma)},$$
in which case we get $T \gamma^N\leq dD\sqrt{T}$. Therefore,
\begin{align*}
\regret(T) &\leq 3\sqrt{dDT}\log_2\left(1 + \frac{2TL_{\varphi}^2}{\lambda}\right)+ D\sqrt{2T\log(1/\delta))}+ T\gamma^N \\
&\quad + 8f(L_\theta,L_\varphi) \left(\frac{88dD}{\kappa}\beta_T^2+ \frac{256\sqrt{2} L_{\varphi}\eta dD}{\kappa\sqrt{\lambda}} \beta_T\right)\log\left(1+\frac{T\mathcal{U}L_{\varphi}^2}{\lambda d}\right)(\log(\mathcal{U}t/\delta))^2\\
&\quad + 32f(L_\theta,L_\varphi)D\beta_T\sqrt{dT\log\left(1+\frac{T\mathcal{U}L_{\varphi}^2}{\lambda d}\right)}(\log(\mathcal{U}t/\delta))^2\\
&=\widetilde{\mathcal{O}}\left(dD \sqrt{T} + \kappa^{-1}d^2D\right),
\end{align*}
as required.

\section{Proof of Theorem \ref{thm:discounted-ub}: Performance Analysis of \texttt{UCMNLK} for the Discounted-Reward Setting}\label{sec:UCMNLK-proof-discounted}

Recall that $K_T$ denotes the total number of distinct episodes over $T$ time steps and that $t_{K_T+1}$ is defined as $T+1$ for simplicity. Let $\pi$ denote the non-stationary policy taken by \texttt{UCMNLK}. Then by Lemma~\ref{average-optimism}, we have
\begin{equation}\label{eq:discounted-1}
\regret(\pi,T) =\sum_{k=1}^{K_T}\sum_{t=t_k}^{t_{k+1}-1} \left( V^*(s_t)- V_t^\pi(s_t)\right)\leq \underbrace{\sum_{k=1}^{K_T}\sum_{t=t_k}^{t_{k+1}-1} \left( V_k(s_t)- V_t^\pi(s_t)\right)}_{\regret'(\pi,T)}.
\end{equation}
Note that for $t_k\leq t\leq t_{k+1}-1$,
\begin{align}\label{eq:discounted-2}
\begin{aligned}
V_k(s_t)&= Q_k(s_t,a_t)\leq r(s_t,a_t)+\gamma \max_{p\in\mathcal{P}_{t_k}}\left\{\sum_{s'\in\mathcal{S}_t} p_{s_t,a_t,s'} V_k(s')\right\}+\gamma^N
\end{aligned}
\end{align}
where the inequality follows from Lemma~\ref{convergence-devi}. Moreover, by the Bellman equation, 
\begin{align}\label{eq:discounted-3}
\begin{aligned}
V_t^\pi(s_t)&= r(s_t,a_t) + \gamma \sum_{s'\in\mathcal{S}_t}p^*_{s_t,a_t,s'}V_{t+1}^\pi(s').
\end{aligned}
\end{align}
Combining~\eqref{eq:discounted-1},~\eqref{eq:discounted-2}, and~\eqref{eq:discounted-3}, we deduce that
\begin{align}\label{eq:discounted-4}
\begin{aligned}
&\regret(\pi,T) - T\gamma^N\\
&\leq \regret'(\pi,T) - T\gamma^N\\
&\leq \underbrace{\gamma \sum_{k=1}^{K_T}\sum_{t=t_k}^{t_{k+1}-1} \max_{p\in\mathcal{P}_{t_k}}\left\{\sum_{s'\in\mathcal{S}_t} \left(p_{s_t,a_t,s'}-p^*_{s_t,a_t,s'}\right) V_k(s')\right\}}_{(i)}\\
&\quad +   \underbrace{\gamma\sum_{k=1}^{K_T}\sum_{t=t_k}^{t_{k+1}-1} \left(\sum_{s'\in\mathcal{S}_t} p^*_{s_t,a_t,s'} \left(V_k(s')- V_{t+1}^\pi(s')\right)- \left(V_k(s_{t+1})-V_{t+1}^\pi(s_{t+1})\right)\right)}_{(ii)}\\
&\quad +   \underbrace{\gamma\sum_{k=1}^{K_T}\sum_{t=t_k}^{t_{k+1}-1}\left(V_k(s_{t+1})-V_{t+1}^\pi(s_{t+1})\right)}_{(iii)}.
\end{aligned}
\end{align}
Note that term $(i)$ is the same as term $(d)$ in~\eqref{eq:average-regret-decomposition}. Following the same argument in \Cref{sec:term-d} and using the fact that $V_k(s')\leq 1/(1-\gamma)$ for any $s'\in\mathcal{S}_t$ due to Lemma~\ref{average-optimism}, we deduce that
\begin{equation}\label{eq:term1}
\text{Term $(i)$}\leq \left(\frac{88d}{\kappa(1-\gamma)}\beta_T^2+ \frac{256\sqrt{2} L_{\varphi}\eta d}{\kappa\sqrt{\lambda}(1-\gamma)} \beta_T\right)\log\left(1+\frac{T\mathcal{U}L_{\varphi}^2}{\lambda d}\right) + \frac{4\beta_T}{(1-\gamma)}\sqrt{dT\log\left(1+\frac{T\mathcal{U}L_{\varphi}^2}{\lambda d}\right)}.
\end{equation}
For term $(ii)$, we observe that taking $Y_t$ for $t_k\leq t\leq t_{k+1}-1$ and $k\in[K_T]$ as
$$Y_t=\sum_{s'\in\mathcal{S}_t} p^*_{s_t,a_t,s'} \left(V_k(s')- V_{t+1}^\pi(s')\right)- \left(V_k(s_{t+1})-V_{t+1}^\pi(s_{t+1})\right)$$
gives rise to a martingale difference sequence. Moreover, we have $|Y_t|\leq 2/(1-\gamma)$. Then applying Azuma's inequality (Lemma~\ref{lem:azuma}), we deduce that
\begin{equation}\label{eq:term2}
    \text{Term $(ii)$}\leq \frac{2}{1-\gamma}\sqrt{2T\log(1/\delta)}.
    \end{equation}
For term $(iii)$, observe that
\begin{equation}\label{eq:term3}
\begin{aligned}
\text{term $(iii)$}&= \gamma\sum_{k=1}^{K_T}\sum_{t=t_k}^{t_{k+1}-1}\left(V_k(s_{t})-V_{t+1}^\pi(s_{t})\right) \\
&\quad + \gamma\sum_{k=1}^{K_T}\left(-(V_k(s_{t_k}) - V_{t_k}^\pi(s_{t_k}))+ (V_k(s_{t_{k+1}}) - V_{t_{k+1}}^\pi(s_{t_{k+1}}))\right)\\
&\leq \gamma \cdot \regret'(\pi,T) + \frac{\gamma}{1-\gamma}K_T.
\end{aligned}
\end{equation}
Combining~\eqref{eq:discounted-4},~\eqref{eq:term1},~\eqref{eq:term2}, and~\eqref{eq:term3}, it follows that
\begin{align*}
    \begin{aligned}
    &\regret(\pi,T)\\
    &\leq \frac{T\gamma^N}{1-\gamma} +\frac{\gamma}{(1-\gamma)^2}K_T+ \frac{2}{(1-\gamma)^2}\sqrt{2T\log(1/\delta)} \\
    &\quad +\left(\frac{88d}{\kappa(1-\gamma)^2}\beta_T^2+ \frac{256\sqrt{2} L_{\varphi}\eta d}{\kappa\sqrt{\lambda}(1-\gamma)^2} \beta_T\right)\log\left(1+\frac{T\mathcal{U}L_{\varphi}^2}{\lambda d}\right) + \frac{4\beta_T}{(1-\gamma)^2}\sqrt{dT\log\left(1+\frac{T\mathcal{U}L_{\varphi}^2}{\lambda d}\right)}.
    \end{aligned}
\end{align*}
Setting $N$ as
$$N\geq \frac{1}{1-\gamma}\log\left(\frac{\sqrt{T}}{d}\right),$$
we obtain
$$\regret(\pi, T)=\widetilde{\mathcal{O}}\left(\frac{1}{(1-\gamma)^2}d\sqrt{T} + \frac{1}{\kappa(1-\gamma)^2}d^2\right),$$
as required. 

\section{Proofs for Section \ref{sec:UCMNLK-analysis}}

In this section, we prove Lemmas~\ref{average-optimism},~\ref{convergence-devi}, and~\ref{bounded-span} given in \Cref{sec:UCMNLK-analysis}.

\subsection{Proof of Lemma~\ref{average-optimism}: Optimistic Value Functions for \texttt{UCMNLK}}\label{sec:lemma-average-optimism}

Let $V^{(0)},\ldots, V^{(N-1)}$ denote the sequence of value functions generated by \texttt{DEVI} for episode $k$, and let $Q^{(0)},\ldots, Q^{(N)}$ be the sequence of action-value functions generated by \texttt{DEVI} for episode $k$. Then $Q_k$ equals $Q^{(N)}$, and $V_k(s)$ is given by $\max_{a\in\mathcal{A}} Q^{(N)}(s,a)$ for  $s\in\mathcal{S}$. For simplicity, we define $V^{(N)}$ as $V_k$.

To prove that $1/(1-\gamma)\geq V_k(s)\geq V^*(s)$ and $1/(1-\gamma)\geq Q_k(s,a)\geq Q^*(s,a)$ for $(s,a)\in\mathcal{S}\times\mathcal{A}$, we will argue by induction on $n$ that for each $n\in\{0,\ldots, N\}$, $1/(1-\gamma)\geq V^{(n)}(s)\geq V^*(s)$ and $1/(1-\gamma)\geq Q^{(n)}(s,a)\geq Q^*(s,a)$ for $(s,a)\in\mathcal{S}\times\mathcal{A}$. For $n=0$, $1/(1-\gamma)=Q^{(0)}(s,a)\geq Q^*(s,a)$ for all $(s,a)\in\mathcal{S}\times \mathcal{A}$. Moreover, note that $V^{(0)}(s)=\max_{a\in\mathcal{A}}Q^{(0)}(s,a)=1/(1-\gamma)\geq V^*(s)$ for all $s\in\mathcal{S}$. 

Assume that $1/(1-\gamma)\geq V^{(n)}(s)\geq V^*(s)$ and $1/(1-\gamma)\geq Q^{(n)}(s,a)\geq Q^*(s,a)$ for $(s,a)\in\mathcal{S}\times\mathcal{A}$ for some $n\in\{0,\ldots, N-1\}$. Note that
\begin{align*}
Q^{(n+1)}(s,a)  &=r(s,a)+\gamma \max_{p\in\mathcal{P}_{t_k}}\left\{\sum_{s'\in\mathcal{S}_{s,a}}p_{s,a,s'}V^{(n)}(s')\right\}\\
&\geq r(s,a)+\gamma \max_{p\in\mathcal{P}_{t_k}}\left\{\sum_{s'\in\mathcal{S}_{s,a}}p_{s,a,s'}V^*(s')\right\}\\
&\geq r(s,a)+\gamma \sum_{s'\in\mathcal{S}_{s,a}}p(s'\mid s,a,\theta^*)V^*(s')\\
&= Q^*(s,a)
\end{align*}
where the first inequality follows from the induction hypothesis that $V^{(n)}(s)\geq V^*(s)$, the second inequality is by Lemma~\ref{lem:confidence-polytope}, and the last equality is due to the Bellman optimality equation. Moreover, 
$$Q^{(n+1)}(s,a)  =r(s,a)+\gamma \max_{p\in\mathcal{P}_{t_k}}\left\{\sum_{s'\in\mathcal{S}_{s,a}}p_{s,a,s'}V^{(n)}(s')\right\}\leq r(s,a) + \frac{\gamma}{1-\gamma} \leq \frac{1}{1-\gamma}$$
holds because the induction hypothesis implies that $V^{(n)}(s')\leq 1/(1-\gamma)$ for $s'\in \mathcal{S}_{s,a}$ and $\sum_{s'\in\mathcal{S}_{s,a}}p_{s,a,s'}=1$ for any $p\in \mathcal{P}_{t_k}$. 

Next, we argue that $1/(1-\gamma)\geq V^{(n+1)}(s)\geq V^*(s)$ for $s\in \mathcal{S}$. Since $V^{(n+1)}(s) = \max_{a\in\mathcal{A}} Q^{(n+1)}(s,a)$ and $Q^{(n+1)}(s,a)\leq 1/(1-\gamma)$ for any $(s,a)\in\mathcal{S}\times \mathcal{A}$, it follows that
$V^{(n+1)}(s)\leq 1/(1-\gamma)$ for any $s\in \mathcal{S}$. Furthermore, we have
$$V^{(n+1)}(s)  = \max_{a\in\mathcal{A}}Q^{(n+1)}(s,a) \geq \max_{a\in\mathcal{A}}Q^*(s,a)=V^*(s).$$
By the induction argument, we have just proved that
$$\frac{1}{1-\gamma}\geq V^{(n)}(s)\geq V^*(s),\quad \frac{1}{1-\gamma}\geq Q^{(n)}(s,a)\geq Q^*(s,a)$$
for any $(s,a)\in\mathcal{S}\times\mathcal{A}$ and $n\in\{0,\ldots, N\}$, as required.

\subsection{Proof of Lemma~\ref{convergence-devi}: Convergence of Discounted Extended Value Iteration}\label{sec:appendix-devi}

We will first show the following lemma.
\begin{lemma}\label{lem:convergence-devi-1}
Let $N$ be the number of rounds for discounted extended value iteration (\texttt{DEVI}). Then
$Q^{(N-1)}(s,a)-Q^{(N)}(s,a)\leq \gamma^{N-1}$ for any $(s,a)\in \mathcal{S}\times\mathcal{A}$.
\end{lemma}
\begin{proof}
Note that for $n\geq 2$, we have
\begin{align*}
Q^{(n)}(s,a)  &=r(s,a)+\gamma \max_{p\in\mathcal{P}}\left\{\sum_{s'\in\mathcal{S}_{s,a}}p_{s,a,s'}V^{(n-1)}(s')\right\},\\
Q^{(n-1)}(s,a)  &=r(s,a)+\gamma \max_{p\in\mathcal{P}}\left\{\sum_{s'\in\mathcal{S}_{s,a}}p_{s,a,s'}V^{(n-2)}(s')\right\}.
\end{align*}
This implies that for any $(s,a)\in\mathcal{S}\times\mathcal{A}$, there exists some $\tilde p\in \mathcal{P}$ such that
\begin{align*}
&Q^{(n-1)}(s,a)-Q^{(n)}(s,a)\\
&=\gamma \left(\max_{p\in\mathcal{P}}\left\{\sum_{s'\in\mathcal{S}_{s,a}}p_{s,a,s'}V^{(n-2)}(s')\right\}-\max_{p\in\mathcal{P}}\left\{\sum_{s'\in\mathcal{S}_{s,a}}p_{s,a,s'}V^{(n-1)}(s')\right\}\right)\\
&\leq \gamma \max_{p\in\mathcal{P}}\sum_{s'\in\mathcal{S}_{s,a}}p_{s,a,s'}\left(V^{(n-2)}(s')-V^{(n-1)}(s')\right)\\
&= \gamma \sum_{s'\in\mathcal{S}_{s,a}}\tilde p_{s,a,s'}\left(V^{(n-2)}(s')-V^{(n-1)}(s')\right)
\end{align*}
where the inequality holds because $\max_p\{f(p)+g(p)\}\leq \max_p\{f(p)\} + \max_p\{g(p)\}$.
The right-most side can be further bounded as follows.
\begin{align}\label{eq:convergence-devi-0}
\begin{aligned}
&\gamma\sum_{s'\in\mathcal{S}_{s,a}}\tilde p_{s,a,s'}\left(V^{(n-2)}(s')-V^{(n-1)}(s')\right)\\
&\leq \gamma \max_{s'\in\mathcal{S}}\left(V^{(n-2)}(s')-V^{(n-1)}(s')\right)\\
&=  \gamma\max_{s'\in\mathcal{S}}\left(\max_{a'\in\mathcal{A}}Q^{(n-2)}(s',a')-\max_{a'\in\mathcal{A}}Q^{(n-1)}(s',a')\right)\\
&\leq  \gamma\max_{(s',a')\in\mathcal{S}\times\mathcal{A}}\left(Q^{(n-2)}(s',a')-Q^{(n-1)}(s',a')\right)
\end{aligned}
\end{align}
where the first inequality holds because the left-most side is a convex combination of $V^{(n-2)}(s')-V^{(n-1)}(s')$ for $s'\in\mathcal{S}_{s,a}$ and the second inequality is due to $\max_{a'}\{f(a')+g(a')\}\leq \max_{a'}\{f(a')\} + \max_{a'}\{g(a')\}$ as before. 
Therefore, it follows that for any $n\geq 2$,
$$\max_{(s,a)\in\mathcal{S}\times\mathcal{A}}\left(Q^{(n-1)}(s,a)-Q^{(n)}(s,a)\right)\leq \gamma \max_{(s,a)\in\mathcal{S}\times\mathcal{A}}\left(Q^{(n-2)}(s,a)-Q^{(n-1)}(s,a)\right).$$
In particular, this implies that
\begin{align*}
\begin{aligned}
\max_{(s,a)\in\mathcal{S}\times\mathcal{A}}\left(Q^{(N-1)}(s,a)-Q^{(N)}(s,a)\right)&\leq \gamma^{N-1}\max_{(s,a)\in\mathcal{S}\times\mathcal{A}}\left(Q^{(0)}(s,a)-Q^{(1)}(s,a)\right)\\
&= \gamma^{N-1}\max_{(s,a)\in\mathcal{S}\times\mathcal{A}}\left(\frac{1}{1-\gamma}- r(s,a) -\frac{\gamma}{1-\gamma}\right)\\
&\leq \gamma^{N-1}
\end{aligned}
\end{align*}
where the last inequality holds because $r(s,a)\leq 1$.
\end{proof}

Based on Lemma~\ref{lem:convergence-devi-1}, we complete the proof of Lemma~\ref{convergence-devi}. Note that
\begin{align*}
&Q^{(N)}(s_t,a_t)\\
&= r(s_t,a_t) + \gamma \max_{p\in\mathcal{P}_{t_k}}\left\{\sum_{s'\in\mathcal{S}_{t}}p_{s_t,a_t,s'}V^{(N-1)}(s')\right\}\\
&= r(s_t,a_t)+ \gamma \max_{p\in\mathcal{P}_{t_k}}\left\{\sum_{s'\in\mathcal{S}_{t}}p_{s_t,a_t,s'}V^{(N)}(s')\right\} \\
&\quad + \gamma \max_{p\in\mathcal{P}_{t_k}}\left\{\sum_{s'\in\mathcal{S}_{t}}p_{s_t,a_t,s'}\left(V^{(N-1)}(s')-V^{(N)}(s')\right)\right\}\\
&\leq r(s_t,a_t)+ \gamma \max_{p\in\mathcal{P}_{t_k}}\left\{\sum_{s'\in\mathcal{S}_{t}}p_{s_t,a_t,s'}V^{(N)}(s')\right\} + \gamma \max_{(s,a)\in\mathcal{S}\times\mathcal{A}}\left(Q^{(N-1)}(s,a)-Q^{(N)}(s,a)\right)\\
&\leq r(s_t,a_t)+ \gamma \max_{p\in\mathcal{P}_{t_k}}\left\{\sum_{s'\in\mathcal{S}_{t}}p_{s_t,a_t,s'}V^{(N)}(s')\right\} + \gamma^N
\end{align*}
where $V^{(N)}$ is given by $V_k$, the first inequality can be established following the same argument as in~\eqref{eq:convergence-devi-0}, the second inequality is implied by Lemma~\ref{lem:convergence-devi-1}. Since $Q^{(N)}$ equals $Q_k$ and $V^{(N)}$ equals $V_k$, we have
$$Q_k(s_t,a_t)\leq r(s_t,a_t)+ \gamma \max_{p\in\mathcal{P}_{t_k}}\left\{\sum_{s'\in\mathcal{S}_{t}}p_{s_t,a_t,s'}V_k(s')\right\} + \gamma^N,$$
as required.

\subsection{Proof of Lemma~\ref{bounded-span}: Bound on the Span of the Optimistic Value Function}\label{sec:bounded-span}

First, we prove the following lemma.
\begin{lemma}\label{lem:value-monotone}
For any $n\geq 1$ and $s\in\mathcal{S}$,
$V^{(n)}(s)\leq V^{(n-1)}(s).$
\end{lemma}
\begin{proof}
We argue by induction on $n$. Since $V^{(0)}(s)=1/(1-\gamma)$ and $V^{(1)}(s)\leq 1/(1-\gamma)$ by Lemma~\ref{average-optimism}, it is clear that $V^{(1)}(s)\leq V^{(0)}(s)$. We assume that for some $n\geq 1$, $V^{(n)}(s)\leq V^{(n-1)}(s)$ for any $s\in\mathcal{S}$. Then it follows from the induction hypothesis that
\begin{align*}
Q^{(n+1)}(s,a) &= r(s,a) + \gamma \max_{p\in \mathcal{P}_{t_k}}\left\{\sum_{s'\in\mathcal{S}_{s,a}}p_{s,a,s'}V^{(n)}(s')\right\}\\
&\leq r(s,a) + \gamma \max_{p\in \mathcal{P}_{t_k}}\left\{\sum_{s'\in\mathcal{S}_{s,a}}p_{s,a,s'}V^{(n-1)}(s')\right\}\\
&= Q^{(n)}(s,a)
\end{align*}
for any $(s,a)\in\mathcal{S}\times \mathcal{A}$. This implies that for any $s\in\mathcal{S}$,
$$V^{(n+1)}(s)=\max_{a\in\mathcal{A}} Q^{(n+1)}(s,a)\leq \max_{a\in\mathcal{A}} Q^{(n)}(s,a)= V^{(n)}(s),$$
as required.
\end{proof}

Using Lemma~\ref{lem:value-monotone}, we complete the proof of Lemma~\ref{bounded-span}. Let $\tau(\pi)$ denote the number of steps after which state $s$ is reached from state $s'$ for the first time under some policy $\pi$. As the diameter of the MDP is at most $D$, there exists a policy $\widetilde\pi$ such that $\mathbb{E}\left[\tau(\widetilde\pi)\right]\leq D$. For discounted extended value iteration, we may think of the following non-stationary policy. First, starting from the initial state $s'$, we run the policy $\widetilde \pi$ under the true transition $p^*$ until we reach state $s$. If we reach state $s$ within $n$ steps, i.e., $n\geq \tau(\widetilde \pi)$, then we take the non-stationary policy and the non-stationary transition function that give rise to $V^{(n-\tau(\widetilde \pi))}(s)$.
Let $V(s')$ denote the total expected discounted reward under this procedure. Note that $$V(s')\geq  \gamma^n\cdot \frac{1}{1-\gamma}$$
since $V^{(0)}(s'')=1/(1-\gamma)$ for any $s''\in\mathcal{S}$. Then it follows that
$$V(s')\geq \mathbb{P}\left[n<\tau(\widetilde \pi)\right]\cdot \frac{\gamma^n}{1-\gamma} + \mathbb{P}\left[n\geq\tau(\widetilde \pi)\right]\cdot \mathbb{E}\left[\gamma^{\tau(\widetilde \pi)}V^{(n-\tau(\widetilde \pi))}(s)\mid n\geq \tau(\widetilde \pi)\right].$$
Here, as $V^{(n)}(s)\leq 1/(1-\gamma)$ by Lemma~\ref{average-optimism}, we have
$$\gamma^n\cdot \frac{1}{1-\gamma}\geq \mathbb{E}\left[\gamma^{\tau(\widetilde \pi)}V^{(n)}(s)\mid n<\tau(\widetilde \pi)\right].$$
Moreover,
\begin{align*}
\mathbb{E}\left[\gamma^{\tau(\widetilde \pi)}V^{(n-\tau(\widetilde \pi))}(s)\mid n\geq \tau(\widetilde \pi)\right]&\geq \mathbb{E}\left[\gamma^{\tau(\widetilde \pi)}V^{(n)}(s)\mid n\geq \tau(\widetilde \pi)\right].
\end{align*}
Therefore, it follows that
\begin{align*}
V(s')
&\geq \mathbb{P}\left[n<\tau(\widetilde \pi)\right]\cdot \mathbb{E}\left[\gamma^{\tau(\widetilde \pi)}V^{(n)}(s)\mid n< \tau(\widetilde \pi)\right] \\
&\quad + \mathbb{P}\left[n\geq\tau(\widetilde \pi)\right]\cdot \mathbb{E}\left[\gamma^{\tau(\widetilde \pi)}V^{(n)}(s)\mid n\geq \tau(\widetilde \pi)\right]\\
&=\mathbb{E}\left[\gamma^{\tau(\widetilde \pi)}V^{(n)}(s)\right]\\
&\geq \gamma^{\mathbb{E}\left[\tau(\widetilde \pi)\right]}V^{(n)}(s)\\
&\geq \gamma^D V^{(n)}(s)
\end{align*}
where the second inequality is by Jensen's inequality and the third inequality holds because $\gamma<1$ and $\mathbb{E}\left[\tau(\widetilde\pi)\right]\leq D$. Furthermore, it is clear that $V(s')\leq V^{(n)}(s')$. This is because $V^{(n)}(s')$ is the largest possible total expected discounted reward achievable by a policy that maximizes the $n$-step discounted reward for the discounted extended value iteration procedure. Consequently, we have just proved that
$$\gamma^D V^{(n)}(s)\leq V(s')\leq V^{(n)}(s').$$
This implies that
$$V^{(n)}(s)- V^{(n)}(s')\leq (1-\gamma^D)V^{(n)}(s)\leq \frac{1-\gamma^D}{1-\gamma}\leq D$$
where the second inequality comes from $V^{(n)}(s)\leq 1/(1-\gamma)$ by Lemma~\ref{average-optimism} and the second inequality holds because $(1-\gamma^D)/(1-\gamma)=1+\gamma+\cdots+\gamma^{D-1}$.

\section{Cumulative Error Bounds}\label{sec:error-cumulative}

In this section, we prove Lemmas~\ref{lem:term-star} and~\ref{lem:term-starstar}, providing a tight upper bound on the following.
$$\underbrace{\sum_{k=1}^{K_T}\sum_{t=t_k}^{t_{k+1}-1}B_{s_t,a_t}^{1,{t_k}}}_{(\star)}+\underbrace{\sum_{k=1}^{K_T}\sum_{t=t_k}^{t_{k+1}-1}B_{s_t,a_t}^{2,{t_k}}}_{(\star\star)}.$$
The following lemmas are useful for our analysis.
\begin{lemma}\label{lem:determinant1}{\em \citep[Lemma 12]{abbasi2011}}
Let $A,B\in\mathbb{R}^{d\times d}$ be positive semidefinite matrices such that $A\succeq B$. Then for any $x\in\mathbb{R}^d$, we have 
$\|x\|_A\leq \|x\|_B\sqrt{\det(A)/\det(B)}$.
\end{lemma}
\begin{lemma}\label{lem:determinant2}{\em \citep[Lemma 7]{Oh2019}}
Let $x_1,\ldots, x_n\in\mathbb{R}^d$. Then 
$$\det\left(I_d + \sum_{i=1}^n x_ix_i^\top\right)\geq 1 + \sum_{i=1}^n \|x_i\|_2^2.$$
\end{lemma}
\begin{lemma}\label{lem:abbasi lemma 10}{\em \citep[Lemma 10]{abbasi2011}.}
Suppose $x_1,\ldots, x_t\in \mathbb{R}^d$ and $\|x_s\|_2\leq L$ for any $1\leq s\leq t$. 
Let ${V}_t = \lambda {I}_d + \sum_{i=1}^{t} x_i x_i^\top$ for some $\lambda > 0$.
Then $\det ({V}_t)$ is increasing with respect to $t$ and 
    \begin{align*}
    \det({V}_t) \le \left( \lambda + \frac{t  L^2}{d} \right)^d.
    \end{align*}
\end{lemma}

Based on the lemmas, we prove the following technical lemma that analyzes several summation terms.
\begin{lemma}\label{lem:sums}
Let $\lambda\geq L_{\varphi}^2$. For $t\geq 1$, let $\widehat \varphi_{t,s'}$ be defined as $$\widehat \varphi_{t,s'}=\varphi_{t,s'}-\sum_{s''\in\mathcal{S}_{t}} p_{t,s''}(\widehat \theta_{t+1})\varphi_{t,s''}.$$
Then the following statements hold.
\begin{align*}
(1)&\quad \sum_{t=1}^T \max_{s'\in\mathcal{S}_t}\|\varphi_{t,s'}\|_{\Sigma_{t}^{-1}}^2\leq \frac{2d}{\kappa}\log\left(1+ \frac{T\mathcal{U}L_{\varphi}^2}{\lambda d}\right).\\
(2)&\quad \sum_{t=1}^T \sum_{s'\in\mathcal{S}_t}p_{t,s'}(\widehat \theta_{t+1})\|\widehat\varphi_{t,s'}\|_{\Sigma_{t}^{-1}}^2\leq 2d\log\left(1+\frac{T\mathcal{U}L_{\varphi}^2}{\lambda d}\right) \\
(3) &\quad \sum_{t=1}^T \max_{s'\in\mathcal{S}_t}\|\widehat\varphi_{t,s'}\|_{\Sigma_{t}^{-1}}^2\leq \frac{8d}{\kappa}\log\left(1+ \frac{T\mathcal{U}L_{\varphi}^2}{\lambda d}\right).
\end{align*}
\end{lemma}
\begin{proof}
 \noindent
 \textbf{Statement (1):} 
In~\eqref{eq:psd}, we argued that for any $t$ and $\theta$,
\begin{equation}\label{eq:statement1-0}
\grad^2_\theta(\ell_t(\theta))
\succeq \kappa\sum_{s' \in \mathcal{S}_t \setminus \{ \varsigma_t \}}  {\varphi}_{t, s'} {\varphi}_{t, s'}^{\top}.
\end{equation}
Note that if two matrices $A,B\in\mathbb{R}^{d\times d}$ satisfy $A\succeq B$, then for any $x\in \mathbb{R}^d$, $\|x\|_A\geq \|x\|_B$ holds because $A-B$ is positive semidefinite and thus $x^\top (A-B)x\geq 0$. Then Lemma~\ref{lem:determinant1} implies that $\det(A)\geq \det (B)$. Note that
\begin{equation}\label{eq:statement1-1}
\begin{aligned}
\det(\Sigma_{t+1})&\geq \det\left(\Sigma_t+\kappa\sum_{s' \in \mathcal{S}_t \setminus \{ \varsigma_t \}}  {\varphi}_{t, s'} {\varphi}_{t, s'}^{\top}\right)\\
&= \det(\Sigma_t)\cdot\det\left(1+\sum_{s'\in\mathcal{S}_t\setminus\{\varsigma_t\}}
\sqrt{\kappa}\Sigma_t^{-1/2}\varphi_{t,s'}\left(\sqrt{\kappa}\Sigma_t^{-1/2}\varphi_{t,s'}\right)^\top\right)\\
&\geq \det(\Sigma_t)\left(1+ \kappa \sum_{s'\in\mathcal{S}_t\setminus\{\varsigma_t\}}\left\|\varphi_{t,s'}\right\|_{\Sigma_t^{-1}}^2\right)
\end{aligned}
\end{equation}
where the first inequality holds due to $\Sigma_{t+1}=\Sigma_t +\grad^2_\theta(\ell_t(\widehat\theta_{t+1}))$ and~\eqref{eq:statement1-0}, the equality holds because $\Sigma_t$ is positive definite, and the second inequality is from Lemma~\ref{lem:determinant2}. 

Moreover, since $\Sigma_t\succeq \Sigma_1=\lambda I_d$, we have $(1/\lambda)I_d=\Sigma_1^{-1}\succeq \Sigma_t^{-1}$. Then
$$\left\|\varphi_{t,s'}\right\|_{\Sigma_t^{-1}}^2\leq \left\|\varphi_{t,s'}\right\|_{\Sigma_1^{-1}}^2=\frac{1}{\lambda} \left\|\varphi_{t,s'}\right\|_2^2\leq 1$$
where the first inequality holds as $\Sigma_1^{-1}\succeq \Sigma_t^{-1}$ while the second inequality holds because $\lambda \geq L_{\varphi}^2$. In particular, as $\kappa\leq 1$, we deduce that
$$\kappa \cdot\max_{s'\in\mathcal{S}_t}\left\|\varphi_{t,s'}\right\|_{\Sigma_t^{-1}}^2\leq 1.$$
Note that for any $z\in[0,1]$, we have $z \leq 2\log(1+z)$, which implies that
\begin{equation}\label{eq:statement1-2}
\kappa \sum_{t=1}^T \max_{s'\in\mathcal{S}_t}\left\|\varphi_{t,s'}\right\|_{\Sigma_t^{-1}}^2\leq 2\sum_{t=1}^T\log \left(1+ \kappa\cdot\max_{s'\in\mathcal{S}_t}\left\|\varphi_{t,s'}\right\|_{\Sigma_t^{-1}}^2\right).
\end{equation}
Furthermore, as $\det(\Sigma_1)=\lambda^d$, Lemma~\ref{lem:abbasi lemma 10} implies that
\begin{equation}\label{eq:statement1-3}
\sum_{t=1}^T\log\left(1+\kappa \sum_{s'\in\mathcal{S}_t\setminus\{\varsigma_t\}}\left\|\varphi_{t,s'}\right\|_{\Sigma_t^{-1}}^2 \right)   \leq \log\left(\frac{\det(\Sigma_{T+1})}{\det(\Sigma_1)}\right)\leq d\log\left(1+ \frac{T\mathcal{U}L_{\varphi}^2}{\lambda d}\right).
\end{equation}
Combining~\eqref{eq:statement1-2} and~\eqref{eq:statement1-3}, it follows that
$$\sum_{t=1}^T \max_{s'\in\mathcal{S}_t}\left\|\varphi_{t,s'}\right\|_{\Sigma_t^{-1}}^2\leq \frac{2d}{\kappa}\log\left(1+ \frac{T\mathcal{U}L_{\varphi}^2}{\lambda d}\right),$$
as required.

\noindent
\textbf{Statement (2):} Note that
 \begin{align*}
   \grad^2_\theta(\ell_t(\widehat\theta_{t+1}))
    &= \sum_{s' \in \mathcal{S}_t} p_{t,s'}( {\widehat\theta_{t+1}}) {\varphi}_{t, s'} {\varphi}_{t, s'}^{\top} -\sum_{s' \in \mathcal{S}_t} \sum\limits_{s'' \in \mathcal{S}_t} p_{t,s'}( {\widehat\theta_{t+1}}) p_{t,s''}( {\widehat\theta_{t+1}}) {\varphi}_{t, s'} {\varphi}^{\top}_{t, s''}  \\
     &= \sum_{s' \in \mathcal{S}_t} p_{t,s'}( {\widehat\theta_{t+1}}) {\varphi}_{t, s'} \widehat {\varphi}_{t, s'}^\top\\
     &=\sum_{s' \in \mathcal{S}_t} p_{t,s'}( {\widehat\theta_{t+1}}) \widehat {\varphi}_{t, s'} \widehat {\varphi}_{t, s'}^\top + \sum_{s'\in\mathcal{S}_t}  p_{t,s'}( {\widehat\theta_{t+1}})\sum_{s''\in\mathcal{S}_t}p_{t,s''}( {\widehat\theta_{t+1}})\varphi_{t,s''} \widehat \varphi_{t,s'}^\top\\
     &=\sum_{s' \in \mathcal{S}_t} p_{t,s'}( {\widehat\theta_{t+1}}) \widehat {\varphi}_{t, s'} \widehat {\varphi}_{t, s'}^\top
    \end{align*}
    where the last equality holds because
    \begin{align*}\sum_{s'\in\mathcal{S}_t}  p_{t,s'}( {\widehat\theta_{t+1}})\widehat \varphi_{t,s'}&=\sum_{s'\in\mathcal{S}_t}  p_{t,s'}( {\widehat\theta_{t+1}})\varphi_{t,s'} -\sum_{s'\in\mathcal{S}_t}  p_{t,s'}( {\widehat\theta_{t+1}})\sum_{s''\in\mathcal{S}_t}  p_{t,s''}( {\widehat\theta_{t+1}})\varphi_{t,s''}\\
    &=\sum_{s'\in\mathcal{S}_t}  p_{t,s'}( {\widehat\theta_{t+1}})\varphi_{t,s'} -\sum_{s''\in\mathcal{S}_t}  p_{t,s''}( {\widehat\theta_{t+1}})\varphi_{t,s''}\\
    &=0.
    \end{align*}
Therefore, it follows that
\begin{equation}\label{eq:statement2-1}
\begin{aligned}
&\det(\Sigma_{t+1})\\
&=\det\left(\Sigma_t+\sum_{s' \in \mathcal{S}_t } p_{t,s'}( {\widehat\theta_{t+1}}) \widehat{\varphi}_{t, s'} \widehat{\varphi}_{t, s'}^{\top}\right)\\
&= \det(\Sigma_t)\cdot \det\left(I_d+\sum_{s'\in\mathcal{S}_t}
\sqrt{p_{t,s'}( {\widehat\theta_{t+1}}) }\Sigma_t^{-1/2}\widehat\varphi_{t,s'}\left(\sqrt{p_{t,s'}( {\widehat\theta_{t+1}}) }\Sigma_t^{-1/2}\widehat\varphi_{t,s'}\right)^\top\right)\\
&\geq \det(\Sigma_t)\left(1+ \sum_{s'\in\mathcal{S}_t}p_{t,s'}( {\widehat\theta_{t+1}}) \left\|\widehat\varphi_{t,s'}\right\|_{\Sigma_t^{-1}}^2\right)
\end{aligned}
\end{equation}
where the first equality holds due to $\Sigma_{t+1}=\Sigma_t +\grad^2_\theta(\ell_t(\widehat\theta_{t+1}))$ and~\eqref{eq:statement1-0}, the second equality holds because $\Sigma_t$ is positive definite, and the inequality is from Lemma~\ref{lem:determinant2}. Moreover,
\begin{align*}
&\sum_{s'\in\mathcal{S}_t}p_{t,s'}( {\widehat\theta_{t+1}}) \left\|\widehat\varphi_{t,s'}\right\|_{\Sigma_t^{-1}}^2\\
&\leq\frac{1}{\lambda}\sum_{s'\in\mathcal{S}_t}p_{t,s'}( {\widehat\theta_{t+1}}) \left\|\widehat\varphi_{t,s'}\right\|_{2}^2\\
&=\frac{1}{\lambda}\sum_{s'\in\mathcal{S}_t}p_{t,s'}( {\widehat\theta_{t+1}})\left(\varphi_{t,s'}-\sum_{s''\in\mathcal{S}_t}p_{t,s''}(\widehat\theta_{t+1})\varphi_{t,s''}\right)^\top\left(\varphi_{t,s'}-\sum_{s''\in\mathcal{S}_t}p_{t,s''}(\widehat\theta_{t+1})\varphi_{t,s''}\right)\\
&=\frac{1}{\lambda}\sum_{s'\in\mathcal{S}_t}p_{t,s'}( {\widehat\theta_{t+1}})\left\|\varphi_{t,s'}\right\|_2^2- \frac{1}{\lambda}\left\|\sum_{s'\in\mathcal{S}_t}p_{t,s'}( {\widehat\theta_{t+1}})\varphi_{t,s'}\right\|_2^2\\
&\leq\frac{1}{\lambda}\sum_{s'\in\mathcal{S}_t}p_{t,s'}( {\widehat\theta_{t+1}})\left\|\varphi_{t,s'}\right\|_2^2\\
&\leq \frac{1}{\lambda}L_{\varphi}^2
\end{align*}
where the first inequality holds because $(1/\lambda)I_d=\Sigma_1^{-1}\succeq \Sigma_t^{-1}$ . Since $\lambda \geq L_{\varphi}^2$, we have
$$\sum_{s'\in\mathcal{S}_t}p_{t,s'}( {\widehat\theta_{t+1}}) \left\|\widehat\varphi_{t,s'}\right\|_{\Sigma_t^{-1}}^2\leq 1.$$ Then we deduce that
\begin{align*}
\sum_{t=1}^T\sum_{s'\in\mathcal{S}_t}p_{t,s'}( {\widehat\theta_{t+1}}) \left\|\widehat\varphi_{t,s'}\right\|_{\Sigma_t^{-1}}^2&\leq 2\sum_{t=1}^T \log\left(1+\sum_{s'\in\mathcal{S}_t}p_{t,s'}( {\widehat\theta_{t+1}}) \left\|\widehat\varphi_{t,s'}\right\|_{\Sigma_t^{-1}}^2\right)\\
&\leq 2\log\left(\frac{\det(\Sigma_{T+1})}{\det(\Sigma_1)}\right)\\
&\leq 2d\log\left(1+\frac{T\mathcal{U}L_{\varphi}^2}{\lambda d}\right)
\end{align*}
where the first inequality holds because $z\leq 2\log(1+z)$ for any $z\in[0,1]$, the second inequality follows from~\eqref{eq:statement2-1}, and the third inequality is due to Lemma~\ref{lem:abbasi lemma 10}.

\noindent
\textbf{Statement (3):} Note that
$$\|\widehat \varphi_{t,s'}\|_{\Sigma_t^{-1}}\leq \|\varphi_{t,s'}\|_{\Sigma_t^{-1}} + \sum_{s''\in\mathcal{S}_t}p_{t,s''}(\widehat\theta_{t+1})\|\varphi_{t,s''}\|_{\Sigma_t^{-1}}\leq 2\cdot\max_{s''\in\mathcal{S}_t} \| \varphi_{t,s''}\|_{\Sigma_t^{-1}},$$
which implies that
$$\|\widehat \varphi_{t,s'}\|_{\Sigma_t^{-1}}^2 \leq 4\cdot \max_{s''\in\mathcal{S}_t} \| \varphi_{t,s''}\|_{\Sigma_t^{-1}}^2.$$
Then statement (3) follows from statement (1), as required.
\end{proof}

we prove Lemmas~\ref{lem:term-star} and~\ref{lem:term-starstar} which provide upper bounds on terms $(\star)$ and $(\star\star)$, respectively.

\subsection{Proof of Lemma~\ref{lem:term-star}}

Consider the $k$th episode for $k\in\{1,\ldots, K_T\}$. For $t_k\leq t\leq t_{k+1}-1$, let us use notations $\bar \varphi_{t,s'}$ and $\widehat \varphi_{t,s'}$ given by
$$
\bar \varphi_{t,s'}=\varphi_{t,s'}-\sum_{s''\in\mathcal{S}_{t}} p_{t,s''}(\widehat \theta_{t_k})\varphi_{t,s''},\quad
\widehat \varphi_{t,s'}=\varphi_{t,s'}-\sum_{s''\in\mathcal{S}_{t}} p_{t,s''}(\widehat \theta_{t+1})\varphi_{t,s''}.$$
Then we have
\begin{align}\label{eq:star-0}
\begin{aligned}
&\sum_{t=t_k}^{t_{k+1}-1}B_{s_t,a_t}^{1,t_k}\\
&= \sum_{t=t_k}^{t_{k+1}-1}\beta_t\sum_{s'\in\mathcal{S}_t}p_{t,s'}(\widehat \theta_{t_k}) \|\bar\varphi_{t,s'}\|_{\Sigma_{t_k}^{-1}}\\
&\leq \beta_T\sum_{t=t_k}^{t_{k+1}-1}\sum_{s'\in\mathcal{S}_t}p_{t,s'}(\widehat \theta_{t_k}) \|\bar\varphi_{t,s'}\|_{\Sigma_{t_k}^{-1}}\\
&\leq \beta_T\sum_{t=t_k}^{t_{k+1}-1}\sum_{s'\in\mathcal{S}_t}\left(p_{t,s'}(\widehat \theta_{t_k}) \|\bar\varphi_{t,s'}-\widehat \varphi_{t,s'}\|_{\Sigma_{t_k}^{-1}}+\left(p_{t,s'}(\widehat \theta_{t_k})-p_{t,s'}(\widehat \theta_{t+1})\right) \|\widehat \varphi_{t,s'}\|_{\Sigma_{t_k}^{-1}}\right)\\
&\quad +\beta_T{\sum_{t=t_k}^{t_{k+1}-1}\sum_{s'\in\mathcal{S}_t}p_{t,s'}(\widehat \theta_{t+1}) \|\widehat \varphi_{t,s'}\|_{\Sigma_{t_k}^{-1}}}\\
&\leq 2\beta_T\underbrace{\sum_{t=t_k}^{t_{k+1}-1}\sum_{s'\in\mathcal{S}_t}\left|p_{t,s'}(\widehat \theta_{t_k})-p_{t,s'}(\widehat \theta_{t+1})\right| \|\widehat \varphi_{t,s'}\|_{\Sigma_{t_k}^{-1}}}_{(a)}+\beta_T\underbrace{\sum_{t=t_k}^{t_{k+1}-1}\sum_{s'\in\mathcal{S}_t}p_{t,s'}(\widehat \theta_{t+1}) \|\widehat \varphi_{t,s'}\|_{\Sigma_{t_k}^{-1}}}_{(b)}
\end{aligned}
\end{align}
where the first inequality holds because $\beta_t$ increases as $t$ gets large and the last inequality follows from
\begin{equation*}
\begin{aligned}
\sum_{s'\in\mathcal{S}_t}p_{t,s'}(\widehat \theta_{t_k}) \|\bar\varphi_{t,s'}-\widehat \varphi_{t,s'}\|_{\Sigma_{t_k}^{-1}}
&=\sum_{s'\in\mathcal{S}_t}p_{t,s'}(\widehat \theta_{t_k})\left\|\sum_{s''\in\mathcal{S}_t}\left(p_{t,s''}(\widehat\theta_{t+1})-p_{t,s''}(\widehat\theta_{t_k})\right)\varphi_{t,s''}\right\|_{\Sigma_{t_k}^{-1}}\\
&=\left\|\sum_{s''\in\mathcal{S}_t}\left(p_{t,s''}(\widehat\theta_{t+1})-p_{t,s''}(\widehat\theta_{t_k})\right)\varphi_{t,s''}\right\|_{\Sigma_{t_k}^{-1}}\\
&=\left\|\sum_{s''\in\mathcal{S}_t}\left(p_{t,s''}(\widehat\theta_{t+1})-p_{t,s''}(\widehat\theta_{t_k})\right)\widehat\varphi_{t,s''}\right\|_{\Sigma_{t_k}^{-1}}
\end{aligned}
\end{equation*}
as we have
$$\sum_{s''\in\mathcal{S}_t}\left(p_{t,s''}(\widehat\theta_{t+1})-p_{t,s''}(\widehat\theta_{t_k})\right)\sum_{s''''\in\mathcal{S}_t}p_{t,s''''}(\widehat \theta_{t+1})\varphi_{t,s''''}=0.$$
Let us first consider term $(a)$. Note that 
$$\left|p_{t,s'}(\widehat \theta_{t_k})-p_{t,s'}(\widehat \theta_{t+1})\right| \leq \left|p_{t,s'}(\widehat \theta_{t_k})-p_{t,s'}(\theta^*)\right|+\left|p_{t,s'}(\theta^*)-p_{t,s'}(\widehat \theta_t)\right|+\left|p_{t,s'}(\widehat\theta_t)-p_{t,s'}(\widehat \theta_{t+1})\right|.$$
Moreover, by the triangle inequality, 
$$\|\widehat \varphi_{t,s'}\|_{\Sigma_{t_k}^{-1}}\leq \|\varphi_{t,s'}\|_{\Sigma_{t_k}^{-1}}+\sum_{s''\in\mathcal{S}_t}p_{t,s''}(\widehat \theta_{t+1})\| \varphi_{t,s''}\|_{\Sigma_{t_k}^{-1}}\leq 2\max_{s'\in\mathcal{S}_t}\| \varphi_{t,s'}\|_{\Sigma_{t_k}^{-1}}.$$
Then it follows that
\begin{align}\label{doublstar-a}
\begin{aligned}
    (a)&\leq 2\sum_{t=t_k}^{t_{k+1}-1} \max_{s'\in\mathcal{S}_t}\| \varphi_{t,s'}\|_{\Sigma_{t_k}^{-1}}\sum_{s'\in\mathcal{S}_t}\left(\left|p_{t,s'}(\widehat \theta_{t_k})-p_{t,s'}(\theta^*)\right|+\left|p_{t,s'}(\theta^*)-p_{t,s'}(\widehat \theta_t)\right|\right)\\
    &\quad + \sum_{t=t_k}^{t_{k+1}-1} \sum_{s'\in\mathcal{S}_t}\left|p_{t,s'}(\widehat\theta_t)-p_{t,s'}(\widehat \theta_{t+1})\right|\|\widehat \varphi_{t,s'}\|_{\Sigma_{t_k}^{-1}}\\
    &\leq \underbrace{4\beta_T\sum_{t=t_k}^{t_{k+1}-1}\max_{s'\in\mathcal{S}_t}\| \varphi_{t,s'}\|_{\Sigma_{t_k}^{-1}}^2}_{(a1)} + \underbrace{\sum_{t=t_k}^{t_{k+1}-1} \sum_{s'\in\mathcal{S}_t}\left|p_{t,s'}(\widehat\theta_t)-p_{t,s'}(\widehat \theta_{t+1})\right|\|\widehat \varphi_{t,s'}\|_{\Sigma_{t_k}^{-1}}}_{(a2)}
    \end{aligned}
\end{align}
{where the second inequality follows because \Cref{lem:confidence-polytope'} holds and $\beta_t\leq \beta_T$ for any $t\leq T$.} Here, let us consider term $(a2)$. By Taylor's theorem, for any $s''\in\mathcal{S}_t$, there exists some $\alpha_{s''}\in[0,1]$ such that $\vartheta_{t,s''}=\alpha_{s''}\widehat\theta_{t+1}+(1-\alpha_{s''})\widehat\theta_{t}$ satisfying
\begin{align*}
&p_{t,s''}(\widehat\theta_{t+1})-p_{t,s''}(\widehat\theta_{t})\\
&=\grad_\theta (p_{t,s''}(\vartheta_{t,s''}))^\top(\widehat\theta_{t+1}-\widehat\theta_{t})\\
&=\left(p_{t,s''}(\vartheta_{t,s''})\varphi_{t,s''}- p_{t,s''}(\vartheta_{t,s''})\sum_{s'''\in\mathcal{S}_t}p_{t,s'''}(\vartheta_{t,s''})\varphi_{t,s'''}\right)^\top(\widehat\theta_{t+1}-\widehat\theta_{t})\\
&=\left(p_{t,s''}(\vartheta_{t,s''})\widehat\varphi_{t,s''}- p_{t,s''}(\vartheta_{t,s''})\sum_{s'''\in\mathcal{S}_t}p_{t,s'''}(\vartheta_{t,s''})\widehat\varphi_{t,s'''}\right)^\top(\widehat\theta_{t+1}-\widehat\theta_{t})
\end{align*}
where the last equality holds because
$$p_{t,s''}(\vartheta_{t,s''})\left(1-\sum_{s'''\in\mathcal{S}_t}p_{t,s'''}(\vartheta_{t,s''})\right)\sum_{s''''\in\mathcal{S}_t}p_{t,s''''}(\widehat \theta_{t+1})\varphi_{t,s''''}=0.$$
This implies that
\begin{align}\label{eq:star-3}
\begin{aligned}
&\sum_{s'\in\mathcal{S}_t}\left|p_{t,s'}(\widehat \theta_{t})-p_{t,s'}(\widehat \theta_{t+1})\right| \|\widehat \varphi_{t,s'}\|_{\Sigma_{t_k}^{-1}}\\
&\leq\sqrt{2}\sum_{s'\in\mathcal{S}_t}\left|p_{t,s'}(\widehat \theta_{t})-p_{t,s'}(\widehat \theta_{t+1})\right| \|\widehat \varphi_{t,s'}\|_{\Sigma_{t}^{-1}}\\
&\leq \sqrt{2}\sum_{s''\in\mathcal{S}_t}p_{t,s''}(\vartheta_{t,s''})\left(\|\widehat\varphi_{t,s''}\|_{\Sigma_{t}^{-1}}+\sum_{s'''\in\mathcal{S}_t}p_{t,s'''}(\vartheta_{t,s''})\|\widehat\varphi_{t,s'''}\|_{\Sigma_{t}^{-1}}\right)\left\|\widehat\theta_{t+1}-\widehat\theta_{t}\right\|_{\Sigma_{t}}\|\widehat\varphi_{t,s''}\|_{\Sigma_{t}^{-1}}\\
&\leq 2\sqrt{2}\sum_{s''\in\mathcal{S}_t}p_{t,s''}(\vartheta_{t,s''})\left\|\widehat\theta_{t+1}-\widehat\theta_{t}\right\|_{\Sigma_{t}}\max_{s'''\in\mathcal{S}_t}\|\widehat\varphi_{t,s'''}\|_{\Sigma_{t}^{-1}}^2
\end{aligned}
\end{align}
where the first inequality is implied by Lemma~\ref{lem:determinant1} because $\Sigma_{t_k}^{-1}\succeq \Sigma_t^{-1}$ and
$$\|\widehat\varphi_{t,s''}\|_{\Sigma_{t_k}^{-1}}^2\leq \|\widehat\varphi_{t,s''}\|_{\Sigma_{t}^{-1}}^2\frac{\det(\Sigma_{t_k}^{-1})}{\det(\Sigma_t^{-1})}=\|\widehat\varphi_{t,s''}\|_{\Sigma_{t}^{-1}}^2\frac{\det(\Sigma_{t})}{\det(\Sigma_{t_k})}\leq 2\|\widehat\varphi_{t,s''}\|_{\Sigma_{t}^{-1}}^2$$ 
and the second inequality is due to the Cauchy-Schwarz inequality. Here, due to our choice of $\widehat \theta_{t+1}$ in \eqref{update-core}, we have
$$ \grad_\theta(\ell_{t}(\widehat\theta_{t}))^\top\widehat\theta_{{t}+1}+ \frac{1}{2\eta}\|\widehat \theta_{{t}+1}-\widehat \theta_{t}\|_{\widehat \Sigma_{t}}^2\leq \grad_\theta(\ell_{t}(\widehat\theta_{t}))^\top\widehat\theta_{t},$$
implying in turn that
$$\|\widehat \theta_{{t}+1}-\widehat \theta_{t}\|_{\widehat \Sigma_{t}}^2\leq 2\eta\grad_\theta(\ell_{t}(\widehat\theta_{t}))^\top\left(\widehat\theta_{t}-\widehat\theta_{{t}+1}\right)\leq 2\eta\|\grad_\theta(\ell_{t}(\widehat\theta_{t}))\|_{\widehat \Sigma_{t}^{-1}}\|\widehat \theta_{{t}+1}-\widehat \theta_{t}\|_{\widehat \Sigma_{t}}.$$
Therefore, it follows that
$$\|\widehat \theta_{{t}+1}-\widehat \theta_{t}\|_{\widehat \Sigma_{t}}\leq 2\eta\|\grad_\theta(\ell_{t}(\widehat\theta_{t}))\|_{\widehat \Sigma_{t}^{-1}}.$$
Moreover, recall that for $t\geq 1$
$$\widehat \Sigma_{t}= \Sigma_{t} + \eta\grad_\theta^2(\ell_{t}(\widehat \theta_{t}))\succeq \Sigma_{t} +\eta\kappa \sum_{s'\in\mathcal{S}_{t}}\varphi_{{t},s'}\varphi_{{t},s'}\succeq \Sigma_{t} \succeq \Sigma_1=\lambda I_d$$
where the first inequality is given as in~\eqref{eq:statement1-0}. Hence, we have $\widehat \Sigma_{t}\succeq \Sigma_{t}$ and $(1/\lambda)I_d=\Sigma_1^{-1}\succeq \widehat \Sigma_{t}^{-1}$. Then it follows that
\begin{equation}\label{eq:star-2}\|\widehat \theta_{{t}+1}-\widehat \theta_{t}\|_{ \Sigma_{t}}\leq \|\widehat \theta_{{t}+1}-\widehat \theta_{t}\|_{\widehat \Sigma_{t}}\leq 2\eta \|\grad_\theta(\ell_{t}(\widehat\theta_{t}))\|_{\widehat \Sigma_{t}^{-1}}\leq \frac{2\eta}{\sqrt{\lambda}}\|\grad_\theta(\ell_{t}(\widehat\theta_{t}))\|_2.\end{equation}
Here, we have
\begin{align}\label{eq:star-1}
\begin{aligned}
\|\grad_\theta(\ell_{t}(\widehat\theta_{t}))\|_2&=\left\|- \sum_{s' \in \mathcal{S}_{{t}}} \left( y_{{t},s'}  - p_{{t},s'}( {\widehat\theta_{t}}) \right) {\varphi}_{{t},s'}\right\|_2\\
&\leq \left\|\sum_{s' \in \mathcal{S}_{{t}}}y_{{t},s'}{\varphi}_{{t},s'}\right\|_2+\left\|\sum_{s' \in \mathcal{S}_{{t}}}p_{{t},s'}( {\widehat\theta_{t}}) {\varphi}_{{t},s'}\right\|_2\\
&\leq 2\cdot \max_{s'\in\mathcal{S}_{t}}\|\varphi_{{t},s'}\|_2\\
&\leq 2L_\varphi.
\end{aligned}
\end{align}
Combining \eqref{eq:star-3}, \eqref{eq:star-2}, and \eqref{eq:star-1}, we may provide an upper bound on term $(a)$ as follows.
\begin{equation}\label{eq:star-term-a}
\begin{aligned}
&\sum_{t=t_k}^{t_{k+1}-1}\sum_{s'\in\mathcal{S}_t}\left|p_{t,s'}(\widehat \theta_t)-p_{t,s'}(\widehat \theta_{t+1})\right| \|\widehat \varphi_{t,s'}\|_{\Sigma_{t}^{-1}}\\
&\leq \frac{8\sqrt{2}L_{\varphi}\eta }{\sqrt{\lambda}}\sum_{t=t_k}^{t_{k+1}-1}\sum_{s''\in\mathcal{S}_t}p_{t,s''}(\vartheta_{t,s''})\max_{s'''\in\mathcal{S}_t}\|\widehat\varphi_{t,s'''}\|_{\Sigma_{t}^{-1}}^2\\
&= \frac{8\sqrt{2}L_{\varphi}\eta }{\sqrt{\lambda}}\sum_{t=t_k}^{t_{k+1}-1}\max_{s''\in\mathcal{S}_t}\|\widehat\varphi_{t,s''}\|_{\Sigma_{t}^{-1}}^2
\end{aligned}
\end{equation}
Moreover, \Cref{lem:determinant1} implies that for term $(a1)$,
\begin{equation}\label{eq:star-term-a1}4\beta_T\sum_{t=t_k}^{t_{k+1}-1}\max_{s'\in\mathcal{S}_t}\| \varphi_{t,s'}\|_{\Sigma_{t_k}^{-1}}^2\leq 8\beta_T\sum_{t=t_k}^{t_{k+1}-1}\max_{s'\in\mathcal{S}_t}\| \varphi_{t,s'}\|_{\Sigma_{t}^{-1}}^2
\end{equation}
and for term $(b)$,
\begin{equation}\label{eq:star-term-b}
    \begin{aligned}
    \sum_{t=t_k}^{t_{k+1}-1}\sum_{s'\in\mathcal{S}_t}p_{t,s'}(\widehat \theta_{t+1}) \|\widehat \varphi_{t,s'}\|_{\Sigma_{t_k}^{-1}}&\leq   \sqrt{2}\sum_{t=t_k}^{t_{k+1}-1}\sum_{s'\in\mathcal{S}_t}p_{t,s'}(\widehat \theta_{t+1}) \|\widehat \varphi_{t,s'}\|_{\Sigma_{t}^{-1}}.
    \end{aligned}
\end{equation}
By \eqref{eq:star-0}, we may deduce the following upper bound on term $(\star)$.
\begin{align*}
&\sum_{k=1}^{K_T}\sum_{t=t_k}^{t_{k+1}-1}B_{s_t,a_t}^{1,t_k}\\
&\leq \beta_T\left(16\beta_T\sum_{t=1}^{T}\max_{s''\in\mathcal{S}_t}\|\varphi_{t,s''}\|_{\Sigma_{t}^{-1}}^2+\frac{16\sqrt{2}L_{\varphi}\eta }{\sqrt{\lambda}}\sum_{t=1}^{T}\max_{s''\in\mathcal{S}_t}\|\widehat\varphi_{t,s''}\|_{\Sigma_{t}^{-1}}^2 + \sqrt{2}\sum_{t=1}^T \sum_{s'\in\mathcal{S}_t}p_{t,s'}(\widehat \theta_{t+1}) \|\widehat \varphi_{t,s'}\|_{\Sigma_{t}^{-1}}\right)\\
&\leq \beta_T\left(\left(\frac{32\beta_T}{\kappa}+\frac{128\sqrt{2} L_{\varphi}\eta}{\kappa\sqrt{\lambda}}\right) d\log\left(1+\frac{T\mathcal{U}L_{\varphi}^2}{\lambda d}\right)+ \sqrt{2 \sum_{t=1}^T \sum_{s'\in\mathcal{S}_t}p_{t,s'}(\widehat \theta_{t+1}) \cdot \sum_{t=1}^T \sum_{s'\in\mathcal{S}_t}p_{t,s'}(\widehat \theta_{t+1}) \|\widehat \varphi_{t,s'}\|_{\Sigma_{t}^{-1}}^2 }\right)\\
&\leq \beta_T\left(\left(\frac{32\beta_T}{\kappa}+\frac{128\sqrt{2} L_{\varphi}\eta}{\kappa\sqrt{\lambda}}\right)d\log\left(1+\frac{T\mathcal{U}L_{\varphi}^2}{\lambda d}\right)+ 2\sqrt{dT\log\left(1+\frac{T\mathcal{U}L_{\varphi}^2}{\lambda d}\right)}\right)
\end{align*}
where the first inequality is due to~\eqref{eq:star-0},~\eqref{doublstar-a},~\eqref{eq:star-term-a},~\eqref{eq:star-term-a1}, and~\eqref{eq:star-term-b}, the second inequality is by the Cauchy-Schwarz inequality, and the third inequality follows from Lemma~\ref{lem:sums}.

\subsection{Proof of Lemma~\ref{lem:term-starstar}}
Note that
\begin{align*}
\sum_{k=1}^{K_T}\sum_{t=t_k}^{t_{k+1}-1}B_{s_t,a_t}^{2,t_k} &= 3\sum_{k=1}^{K_T}\sum_{t=t_k}^{t_{k+1}-1} \beta_t^2\max_{s'\in\mathcal{S}_t}\|\varphi_{t,s'}\|_{\Sigma_{t_k}^{-1}}^2\\
&\leq 3\beta_T^2\sum_{k=1}^{K_T}\sum_{t=t_k}^{t_{k+1}-1}\max_{s'\in\mathcal{S}_t}\|\varphi_{t,s'}\|_{\Sigma_{t_k}^{-1}}^2\\
&\leq 6\beta_T^2\sum_{k=1}^{K_T}\sum_{t=t_k}^{t_{k+1}-1}\max_{s'\in\mathcal{S}_t}\|\varphi_{t,s'}\|_{\Sigma_{t}^{-1}}^2\\
&\leq \frac{12d}{\kappa}\beta_T^2 \log\left(1+\frac{T\mathcal{U}L_{\varphi}^2}{\lambda d}\right)
\end{align*}
where the first inequality holds because $\beta_t$ increases as $t$ gets large, the second inequality is by \Cref{lem:determinant1}, and the third inequality follows from Lemma~\ref{lem:sums}.

\section{Lower Bound Proof for the Finite-Horizon Episodic Setting}\label{sec:lb-finite}

Recall that the transition core $\bar \theta_h$ for each step $h\in[H]$ is given by
$$\bar \theta_h = \left(\frac{\theta_h}{\alpha}, \frac{1}{\beta}\right)\quad\text{where}\quad \theta_h\in\left\{-\bar \Delta,\bar\Delta\right\}^{d-1},\quad \bar \Delta= \frac{1}{d-1}\log\left(\frac{(1-\delta)(\delta+(d-1)\Delta)}{\delta(1-\delta-(d-1)\Delta)}\right),$$
and $\delta=1/H$ and $\Delta=1/(4\sqrt{2HK})$.

\subsection{Linear Approximation of the Multinomial Logistic Model}

We consider a multinomial logistic function given by $f:\mathbb{R}\to \mathbb{R}$ as
$$f(x) = \frac{1}{1+ \frac{1-\delta}{\delta}\exp(-x)}.$$
In contrast to the infinite-horizon average-reward case, we take $\delta=1/H$ where $H$ is the horizon of each episode. Recall that the derivative of $f$ is given by
$$f'(x) = \frac{\frac{1-\delta}{\delta}\exp(-x)}{\left(1+ \frac{1-\delta}{\delta}\exp(-x)\right)^2}=f(x) -f(x)^2.$$
For simplicity, for $h\in[H]$, we use notation $p_{\theta_h}$ given by
$$p_{\theta_h}(x_i\mid x_h,a):=p(x_i\mid x_h,a,\bar\theta_h)=\begin{cases}
f(a^\top \theta_h),&\text{if $i=H+2$}\\
1-f(a^\top \theta_h),&\text{if $i=h+1$}.
\end{cases}.$$
Note that $-(d-1)\bar\Delta\leq a^\top\theta_h\leq (d-1)\bar\Delta$ for any $a\in\mathcal{A}$, which means that $f(-(d-1)\bar\Delta)\leq p_{\theta_h}(x_{H+2}\mid x_h, a) \leq f((d-1)\bar\Delta)$.
The following lemma is analogous to~Lemma \ref{mvt-mnl}. 
\begin{lemma}\label{mvt-mnl-finite}
For any $x,y\in[-(d-1)\bar\Delta,(d-1)\bar\Delta]$ with $x\geq y$, we have
$$0\leq f(x)-f(y)\leq (\delta +(d-1)\Delta) (x-y).$$
\end{lemma}
\begin{proof}
By the mean value theorem, there exists $y\leq z\leq x$ such that $f(x)-f(y)= f'(z)(x-y)$. Note that
$f'(z)= f(z) - f(z)^2\leq f(z)\leq f((d-1)\bar\Delta) = \delta + (d-1)\Delta$ where the last equality holds by our choice of $\bar\Delta$.
\end{proof}
\noindent

\subsection{Basic Properties of the Hard Finite-Horizon Episodic MDP Instance}

Recall that $\delta$ and $\Delta$ are given by $$\delta=\frac{1}{D}\quad\text{and}\quad\Delta =\frac{1}{45\sqrt{(2/5)\log 2}}\cdot \frac{(d-1)}{\sqrt{DT}},$$
respectively. The following lemma characterizes the sizes of parameters $\delta$ and $\Delta$ under the setting of our hard-t0-learn MDP.
\begin{lemma}\label{parameter-bound1-finite}
Suppose that $T\geq H^3(d-1)^2/32$. Then $(d-1)\Delta\leq \delta/H$
\end{lemma}
\begin{proof}
Note that $(d-1)\Delta\leq \delta/H$ if and only if $K\geq H^3(d-1)^2/32$.
\end{proof}

The following lemma provides upper bounds on $L_\varphi$ and $L_\theta$. Moreover, 
\begin{lemma}\label{parameter-bound2-finite}
Suppose that $H\geq 3$. For any $\bar\theta=(\theta/\alpha,1/\beta)$, we have $\|\bar\theta\|_2\leq 3/2$. Moreover, for any $a\in\mathcal{A}$ and $(i,j)\in\{(h,h+1):h\in[H]\}\cup \{(h,H+2):h\in[H]\}$, $\|\varphi(x_i,a,x_j)\|_2\leq 1+ \log(H-1)$.
\end{lemma}
\begin{proof}
Recall that $\alpha=\sqrt{\bar\Delta/(1+(d-1)\bar\Delta)}$ and $\beta = \sqrt{1/(1+(d-1)\bar\Delta)}$. Moreover,
$$\|\bar\theta\|_2^2 = \frac{\|\theta\|_2^2}{\alpha^2}+\frac{1}{\beta^2}=(1+(d-1)\bar\Delta)^2.$$
Note that
$$(d-1)\bar\Delta =\log\left(\frac{(d-1)\Delta + \delta(1-\delta-(d-1)\Delta)}{\delta(1-\delta-(d-1)\Delta)}\right)\leq \frac{(d-1)\Delta }{\delta(1-\delta-(d-1)\Delta)}.$$
Since $(d-1)\Delta\leq \delta/H$ by Lemma \ref{parameter-bound1-finite}, it follows that
$$(d-1)\bar\Delta\leq \frac{1}{H}\cdot\frac{1}{1-\frac{H+1}{H}\delta}=\frac{1}{H-(H+1)/H}\leq \frac{1}{H-1},$$
which implies that $\|\bar\theta\|_2\leq 1+\bar\Delta\leq 3/2$.
Moreover, for any $(i,j)\in\{(h,h+1):h\in[H]\}\cup \{(h,H+2):h\in[H]\}$,
\begin{align*}
\|\varphi(x_i,a,x_j)\|^2&\leq \alpha^2\|a\|_2^2 +\beta^2(\log(H-1))^2\\
&=\frac{(d-1)\bar\Delta}{1+(d-1)\bar\Delta}+ \frac{(\log(H-1))^2}{1+(d-1)\bar\Delta}\\
&\leq (1+\log(H-1))^2,
\end{align*}
as required.
\end{proof}

\subsection{Proof of \Cref{thm:lb-finite}}

Let $\pi=\{\pi_h\}_{h=1}^H$ be a policy for the $H$-horizon MDP. Recall that the value function $V_1^\pi$ under policy $\pi$ is given by
$$V_1^\pi(x_1)=\mathbb{E}_{\theta,\pi}\left[\sum_{h=1}^Hr(s_h, a_h)\mid s_1=x_1\right]$$
where the expectation is taken with respect to the distribution that has dependency on the transition core $\theta$ and the policy $\pi$. Let $N_h$ denote the event that the process visits state $x_h$ in step $h$ and then enters $x_{H+2}$, i.e., $N_h=\{s_h=x_h, x_{h+1}=x_{H+2}\}$.
Then we have that
$$V_1^\pi(x_1)=\sum_{h=1}^{H-1}(H-h)\mathbb{P}_{\theta,\pi}(N_h\mid s_1=x_1).$$
Moreover, note that
\begin{align*}
&\mathbb{P}_{\theta,\pi}(s_{h+1}=x_{H+2}\mid s_h=x_h, s_1=x_1)\\
&=\sum_{a\in\mathcal{A}}\mathbb{P}_{\theta,\pi}(s_{h+1}=x_{H+2}\mid s_h=x_h,a_h=a)\mathbb{P}_{\theta,\pi}(a_h=a\mid s_h=x_h, s_1=x_1)\\
&=\sum_{a\in\mathcal{A}}f(a^\top \theta_h) \mathbb{P}_{\theta,\pi}(a_h=a\mid s_h=x_h, s_1=x_1)\\
&=\delta + \underbrace{\sum_{a\in\mathcal{A}}(f(a^\top\theta_h)-\delta)\mathbb{P}_{\theta,\pi}(a_h=a\mid s_h=x_h,s_1=x_1)}_{a_h}.
\end{align*}
Then it follows that
$$\mathbb{P}_{\theta,\pi}(s_{h+1}=x_{h+1}\mid s_h=x_h, s_1=x_1)= 1-\delta - a_h,$$
which implies that
$$\mathbb{P}_{\theta,\pi}(N_h) = (\delta + a_h)\prod_{j=1}^{h-1}(1-\delta-a_j).$$
Therefore, we deduce that
$$V_1^\pi(x_1)=\sum_{h=1}^H (H-h)(\delta+a_h)\prod_{j=1}^{h-1}(1-\delta-a_j).$$
Note that the optimal policy $\pi^*=\{\pi_h^*\}_{h=1}^H$ deterministically chooses the action maximizing $a^\top \theta_h$ at each step $h$. Recall that the maximum value of $a^\top\theta_h$ is $(d-1)\bar \Delta$ for any $h$, and moreover, $f((d-1)\bar\Delta)=\delta + (d-1)\Delta$. Therefore, under the optimal policy,
$$\mathbb{P}_{\theta,\pi^*}(s_{h+1}=x_{H+2}\mid s_h=x_h, s_1=x_1) = \delta + (d-1)\Delta$$
This further implies that the value function under the optimal policy is given by
$$V_1^*(x_1)=\sum_{h=1}^H (H-h)(\delta+(d-1)\Delta)(1-\delta-(d-1)\Delta)^{h-1}.$$
Next, let us define $S_i$ and $T_i$ for $i\in[H]$ as follows.
$$S_i=\sum_{h=i}^H(H-h)(\delta+a_h)\prod_{j=i}^{h-1}(1-\delta-a_j)\ \text{and}\  T_i = \sum_{h=i}^H(H-h)(\delta+(d-1)\Delta)(1-\delta-(d-1)\Delta)^{h-i}.$$
Following the induction argument of~\citep[Equation (C.25)]{zhou-mixture-finite-optimal} we may deduce that
$$T_1-S_1=\sum_{h=1}^{H-1}((d-1)\Delta - a_h)(H-h-T_{h+1})\prod_{j=1}^{h-1}(1-\delta-a_j).$$
Moreover, since $3(d-1)\Delta\leq \delta=1/H$ and $H\geq 3$ by Lemma \ref{parameter-bound1-finite}, it follows from~\citep[Equations (C.26)]{zhou-mixture-finite-optimal} that
$H-h-T_{h+1}\geq H/3$ for $h\leq H/2$. Moreover, as $a_j\leq(d-1)\Delta\leq \delta/3$, we have $\delta+a_j\leq 4\delta/3$. Since $H\geq 3$, it holds that 
$$\prod_{j=1}^{h-1}(1-\delta-a_j)\geq \left(1-\frac{4\delta}{3}\right)^H\geq \frac{1}{3}.$$
Consequently, we deduce that
\begin{equation}\label{eq:finite-bound-1}
V_1^*(x_1)-V_1^\pi(x_1) = T_1-S_1\geq \frac{H}{10}\sum_{h=1}^{H/2}((d-1)\Delta-a_h).
\end{equation}
From the right-hand side of~\eqref{eq:finite-bound-1}, we have that
$$(d-1)\Delta = \max_{a\in\mathcal{A}}\mu_h^\top a\quad\text{where}\quad \mu_h = \frac{\Delta}{\bar\Delta}\theta_h\in\{-\Delta, \Delta\}^{d-1}.$$
Moreover, note that
\begin{align*}
f(\theta_h^\top a) -\delta &\leq (\delta+(d-1)\Delta)\theta_h^\top a=\frac{\bar\Delta(\delta+(d-1)\Delta)}{\Delta}\mu_h^\top a\leq \frac{\delta+(d-1)\Delta}{\delta(1-\delta-(d-1)\Delta)}\mu_h^\top a
\end{align*}
where the first inequality is due to Lemma \ref{mvt-mnl-finite} and the second inequality holds because
$$\bar\Delta=\frac{1}{d-1}\log\left(1+ \frac{(d-1)\Delta}{\delta(1-\delta-(d-1)\Delta)}\right)\leq \frac{1}{d-1}\cdot \frac{(d-1)\Delta}{\delta(1-\delta-(d-1)\Delta)}= \frac{\Delta}{\delta(1-\delta-(d-1)\Delta)}.$$
Furthermore, as $(d-1)\Delta \leq \delta/H$ by Lemma \ref{parameter-bound1-finite}, we have
$$\frac{\delta+(d-1)\Delta}{\delta(1-\delta-(d-1)\Delta)}\leq \frac{(1+1/H)\delta}{\delta(1-(1+1/H)\delta)}=\frac{H^2+H}{H^2-H-1}=1+\frac{2H+1}{H^2-H-1}\leq 1+ \frac{3}{H}$$
where the first inequality holds because $(d-1)\Delta\leq \delta/H$, the first equality holds due to $\delta=1/H$, and the last inequality is by $H\geq 3$.
Then it follows that
$$f(\theta_h^\top a) -\delta \leq \frac{\delta+(d-1)\Delta}{\delta(1-\delta-(d-1)\Delta)}\mu_h^\top a\leq \mu_h^\top a + \frac{3}{H}\mu_h^\top a \leq \mu_h^\top a + \frac{3(d-1)\Delta}{H}$$
where the last inequality holds because $\mu_h\in\{-\Delta,\Delta\}^{d-1}$ and thus $\mu_h^\top a\leq (d-1)\Delta$. This in turn implies that
$$a_h\leq \frac{3(d-1)\Delta}{H} + \mu_h^\top \underbrace{\sum_{a\in\mathcal{A}}\mathbb{P}_{\theta,\pi}(a_h=a\mid s_h=x_h,s_1=x_1)\cdot a}_{\bar a_h^\pi}$$
Based on~\eqref{eq:finite-bound-1}, we get
\begin{equation}\label{eq:finite-bound-2}
V_1^*(x_1)-V_1^\pi(x_1) = T_1-S_1\geq \frac{H}{10}\sum_{h=1}^{H/2}\left(\max_{a\in\mathcal{A}}\mu_h^\top a - \mu_h^\top \bar a_h^\pi\right) - \frac{H(d-1)\Delta}{20}.
\end{equation}
Let $\mathfrak{A}$ be an algorithm that takes policy $\pi^k=\{\pi_h^k\}_{h=1}^H$ for episodes $k\in[K]$. Then we deduce from~\eqref{eq:finite-bound-2} that
\begin{equation}\label{eq:finite-bound-3}
    \begin{aligned}
        \mathbb{E}\left[\regret(M_\theta, \mathfrak{A},K)\right]&=\mathbb{E}\left[\sum_{k=1}^K \left(V_1^*(x_1)-V_1^{\pi^k}(x_1)\right)\right]\\
        &\geq \frac{H}{10}\sum_{h=1}^{H/2}\underbrace{\mathbb{E}\left[\sum_{k=1}^K\left(\max_{a\in\mathcal{A}}\mu_h^\top a - \mu_h^\top \bar a_h^{\pi^k}\right)\right]}_{I_h(\theta,\pi)} - \frac{H(d-1)}{20}K\Delta.
    \end{aligned}
\end{equation}
Here, we now argue that the term $I_h(\theta,\pi)$ corresponds to the regret under a bandit algorithm for a linear bandit problem. Let $\mathcal{L}_{\mu_h}$ denote the linear bandit problem parameterized by $\mu_h\in\{-\Delta,\Delta\}^{d-1}$ where the action set is $\mathcal{A}=\{-1,1\}^{d-1}$ and the reward distribution for taking action $a\in\mathcal{A}$ is a Bernoulli distribution $B(\delta + \mu_h^\top a)$. Recall that $\bar a_h^{\pi^k}$ is given by
$$\bar a_h^{\pi^k}=\sum_{a\in\mathcal{A}}\mathbb{P}_{\theta,\pi^k}(a_h=a\mid s_h=x_h,s_1=x_1)\cdot a.$$
Basically, $\mathfrak{A}$ corresponds to a bandit algorithm that takes action $a\in\mathcal{A}$ with probability $\mathbb{P}_{\theta,\pi^k}(a_h=a\mid s_h=x_h,s_1=x_1)$ in episode $k$. Let $a_h^{\pi^k}$ denote the random action taken by $\mathfrak{A}$. Then by linearity of expectation, 
$$I_h(\theta,\pi)=\mathbb{E}\left[\sum_{k=1}^K\left(\max_{a\in\mathcal{A}}\mu_h^\top a - \mu_h^\top  a_h^{\pi^k}\right)\right]$$
where the expectation is taken with respect to the randomness generated by $\mathfrak{A}$ and 
which is the expected pseudo-regret under $\mathfrak{A}$. The following lemma provides a lower bound on the expected pseudo-regret for the particular linear bandit instance.
\begin{lemma}\label{lemma:linear-bandit}{\em \citep[Lemma C.8]{zhou-mixture-finite-optimal}.} Suppose that $0<\delta\leq 1/3$ and $K\geq (d-1)^2/(2\delta)$. Let $\Delta = 4\sqrt{2\delta/K}$ and consider the linear bandit problems $\mathcal{L}_{\mu_h}$ described above. Then for any bandit algorithm $\mathfrak{A}$, there exists a parameter $\mu_h^*\in\{-\Delta,\Delta\}^{d-1}$ such that the expected pseudo-regret of $\mathfrak{A}$ over the first $K$ steps on $\mathcal{L}_{\mu_h^*}$ is at least $(d-1)\sqrt{K\delta}/(8\sqrt{2})$.
\end{lemma}
Applying Lemma~\ref{lemma:linear-bandit} to~\eqref{eq:finite-bound-3}, we deduce that
\begin{align*}
        \mathbb{E}\left[\regret(M_\theta, \mathfrak{A},K)\right]&\geq \frac{H^{3/2}(d-1)\sqrt{K}}{160\sqrt{2}} - \frac{H^{1/2}(d-1)\sqrt{K}}{80\sqrt{2}}\\
        &\geq \frac{H^{3/2}(d-1)\sqrt{K}}{160\sqrt{2}} - \frac{H^{3/2}(d-1)\sqrt{K}}{240\sqrt{2}}\\
        &=\frac{H^{3/2}(d-1)\sqrt{K}}{480\sqrt{2}}
\end{align*}
where the second inequality holds because $H\geq 3$.

\section{Lower Bound Proofs for the Infinite-Horizon Setting}\label{sec:lb-infinite}

Recall that the transition core $\bar\theta$ is given by 
$$\bar \theta = \left(\frac{\theta}{\alpha}, \frac{1}{\beta}\right)\quad\text{where}\quad \theta\in\left\{-\frac{\bar \Delta}{d-1},\frac{\bar\Delta}{d-1}\right\}^{d-1},\quad \bar \Delta= \log\left(\frac{(1-\delta)(\delta+\Delta)}{\delta(1-\delta-\Delta)}\right),$$
and $\Delta=(d-1)/(45\sqrt{(2/5)(T/\delta)\log 2})$. Moreover, we set $\delta$ as
$$\delta = \begin{cases}
1/D&\text{for the average-reward case},\\
1-\gamma & \text{for the discounted-reward case}.
\end{cases}$$
The following lemma characterizes the sizes of parameters $\delta$ and $\Delta$ under the setting of our hard-to-learn MDP.
\begin{lemma}\label{parameter-bound1}
Suppose that $d\geq 2$, $\delta \leq 1/101$, $T\geq 45(d-1)^2/\delta$. Then the following statements hold.
$$100\Delta\leq \delta,\quad 2\delta+ \Delta\leq1, \quad  \Delta\leq \delta (1-\delta),\quad \frac{1}{\delta}\leq  \left(\frac{3}{2}\cdot \frac{4}{5}\cdot \left(\frac{99}{101}\right)^4-1\right)T.$$
\end{lemma}
\begin{proof}
If $T\geq 45(d-1)^2/\delta$, then $T\geq (100/15)^2(d-1)^2/\delta$. Note that $\sqrt{(2/5)\log 2}>1/3$. Then $100\Delta<(100/15)(d-1)/\sqrt{T/\delta}$, and as $T\geq (100/15)^2(d-1)^2/\delta$, we get that $100\Delta<\delta$. Moreover, since $\delta\leq1/3$, we also have that $2\delta +\Delta\leq 1$ and $\Delta\leq \delta(1-\delta)$. Moreover, we know that 
$$\frac{3}{2}\cdot \frac{4}{5}\cdot \left(\frac{99}{101}\right)^4>\frac{11}{10}.$$
Since $T\geq 45(d-1)^2/\delta\geq 10/\delta$, the last inequality holds.
\end{proof}

The following lemma provides upper bounds on $L_\varphi$ and $L_\theta$.
\begin{lemma}\label{parameter-bound2}
For any $\bar\theta=(\theta/\alpha,1/\beta)$, we have $\|\bar\theta\|_2\leq 100/99$. Moreover, for any $a\in\mathcal{A}$ and $i,j\in\{0,1\}$, $\|\varphi(x_i,a,x_j)\|_2\leq 1+ \log((1/\delta)-1)$.
\end{lemma}
\begin{proof}
Recall that $\alpha=\sqrt{\bar\Delta/((d-1)(1+\bar\Delta))}$ and $\beta = \sqrt{1/(1+\bar\Delta)}$. Moreover,
$$\|\bar\theta\|_2^2 = \frac{\|\theta\|_2^2}{\alpha^2}+\frac{1}{\beta^2}=(1+\bar\Delta)^2.$$
Note that
$$\bar\Delta = \log\left(\frac{(1-\delta)(\delta+\Delta)}{\delta(1-\delta-\Delta)}\right)=\log\left(\frac{\Delta + \delta(1-\delta-\Delta)}{\delta(1-\delta-\Delta)}\right)\leq \frac{\Delta }{\delta(1-\delta-\Delta)}.$$
Then it follows from Lemma \ref{parameter-bound1} that
$$\bar\Delta\leq \frac{1}{100}\cdot\frac{1}{1-\frac{101}{100}\delta}=\frac{1}{100-101\delta}\leq \frac{1}{99},$$
which implies that $\|\bar\theta\|_2\leq 1+\bar\Delta\leq 100/99$.
Moreover, for any $i,j\in\{0,1\}$,
$$\|\varphi(x_i,a,x_j)\|^2\leq \alpha^2\|a\|_2^2 +\beta^2(\log((1/\delta)-1))^2=\frac{\bar\Delta}{1+\bar\Delta}+ \frac{(\log((1/\delta)-1))^2}{1+\bar\Delta},$$
in which case, we have
$(1+\log((1/\delta)-1))^2\leq (1+\log((1/\delta)-1))^2$, as required.
\end{proof}

\subsection{Linear Approximation of the Multinomial Logistic Model}

Let us define a function $f:\mathbb{R}\to \mathbb{R}$ as
$$f(x) = \frac{1}{1+ \frac{1-\delta}{\delta}\exp(-x)}.$$
The derivative of $f$ is given by
$$f'(x) = \frac{\frac{1-\delta}{\delta}\exp(-x)}{\left(1+ \frac{1-\delta}{\delta}\exp(-x)\right)^2}=f(x) -f(x)^2.$$
The following lemma bridges the multinomial logistic function $x$ and a linear function based on the mean value theorem.
\begin{lemma}\label{mvt-mnl}
For any $x,y\in[-\bar\Delta,\bar\Delta]$ with $x\geq y$, we have
$$0\leq f(x)-f(y)\leq (\delta +\Delta) (x-y).$$
\end{lemma}
\begin{proof}
By the mean value theorem, there exists $y\leq z\leq x$ such that $f(x)-f(y)= f'(z)(x-y)$. Note that
$f'(z)= f(z) - f(z)^2\leq f(z)\leq f(\bar\Delta) = \delta + \Delta$ where the last equality holds by our choice of $\bar\Delta$.
\end{proof}
\noindent
By our choice of feature vector $\varphi$ and transition core $\bar \theta=(\theta/\alpha, 1/\beta)$, we have
$$p(x_1\mid x_0, a) = \frac{1}{1+((1/\delta)-1)\exp(-a^\top\theta)}=f(a^\top\theta) \quad\text{and}\quad p(x_0\mid x_1, a)=\delta=f(0).$$

\subsection{Upper Bound on the Number of Visits to State 1}
For simplicity, we introduce notation $p_{\theta}$ given by
$$p_{ \theta}(x_j\mid x_i, a) := p(x_j\mid x_i, a, \bar\theta)$$
for any $i,j\in\{0,1\}$.
Note that inducing a higher probability of transitioning to $x_1$ from $x_0$ results in a larger reward. This means that the optimal policy chooses action $a$ that maximizes $a^\top \theta$ so that $p(x_1\mid x_0, a)$ is maximized. Then under the optimal policy, we take action $a_\theta$ such that $a_\theta^\top \theta =  \bar\Delta$. Hence, the transition probability under the optimal policy is given by
$$p_\theta(x_1\mid x_0, a_\theta)= f(\bar\Delta) = \delta + \Delta$$
where the second equality follows from our choice of $\bar\Delta$. 

To provide a lower bound, it is sufficient to consider deterministic stationary policies~\citep{Auer2002,puterman2014markov}. Let $\pi$ be a deterministic (non-stationary) policy. Let $\mathcal{P}_\theta$ denote the distribution over $\mathcal{S}^T$ where $s_1=x_0$, $a_t$ is determined by $\pi$, and $s_{t+1}$ is sampled from $p_{\theta}(\cdot\mid s_t,a_t)$. Let $\mathbb{E}_{\theta}$ denote the expectation taken over $\mathcal{P}_\theta$. Moreover, we define $N_i$ for $i\in\{0,1\}$ and $N_0^a$ as the number of times $x_i$ is visited for $i\in\{0,1\}$ and the number of time steps in which state $x_0$ is visited and action $a$ is chosen. We also define $N_0^{\mathcal{V}}$ for $\mathcal{V}\subseteq \mathcal{A}$ as the number of time steps in which state $x_0$ is visited and an action from the set $\mathcal{V}$ is chosen.

In this section, we analyze term $\mathbb{E}_{\theta}N_1$ and provide an upper bound on it, which is crucial for coming up the desired lower bounds for the both average-reward and discounted-reward settings. We prove the following lemma that is analogous to \citep[Lemma C.2]{yuewu2022}.
\begin{lemma}\label{lemma:lb-infinite-1}
Suppose that $2\delta+\Delta\leq 1$, $\Delta\leq \delta(1-\delta)$, and $$\frac{1}{\delta}\leq \left(\frac{3}{2}\cdot \frac{4}{5}\cdot \left(\frac{99}{101}\right)^4-1\right)T.$$ Then
$$
\mathbb{E}_\theta N_1\leq\frac{T}{2}+ \frac{\delta+\Delta}{2\delta} \sum_{a\in\mathcal{A}}a^\top\theta\cdot \mathbb{E}_\theta N_0^a\quad\text{and}\quad \mathbb{E}_\theta N_0\leq\left(\frac{99}{101}\right)^4\cdot\frac{4}{5}T.$$
\end{lemma}
\begin{proof}
See Lemma \ref{sec:lemma:lb-infinite-1}.
\end{proof}
Note that since $a\in\{-1,1\}^{d-1}$,
$$(\delta+\Delta)a^\top\theta \leq (\delta+\Delta)\frac{\bar\Delta}{d-1}\sum_{j=1}^{d-1}\mathbf{1}\left\{\mathrm{sign}(a_j)=\mathrm{sign}(\theta_j)\right\}.$$
Moreover, 
$$\bar\Delta = \log\left(\frac{(1-\delta)(\delta+\Delta)}{\delta(1-\delta-\Delta)}\right)=\log\left(\frac{\Delta + \delta(1-\delta-\Delta)}{\delta(1-\delta-\Delta)}\right)\leq \frac{\Delta }{\delta(1-\delta-\Delta)}$$
where the inequality holds because $1+x\leq \exp(x)$ for any $x\in\mathbb{R}$. Moreover, since $100\Delta\leq \delta$ and $\delta\leq 1/101$ by Lemma~\ref{parameter-bound1}, we have 
\begin{equation}\label{lb-infinite-auxiliary}(\delta+\Delta)\bar \Delta\leq \frac{(\delta+\Delta) }{\delta(1-\delta-\Delta)}\leq \frac{101}{100}\cdot \frac{1}{1-\frac{101}{100}\delta}\cdot\Delta\leq \frac{101}{99}\Delta.
\end{equation}
Then it follows from Lemma \ref{lemma:lb-infinite-1} that
\begin{align}\label{lb-infinite-intermediate}
\begin{aligned}
\frac{1}{|\Theta|}\sum_{\theta\in\Theta}\mathbb{E}_\theta N_1&\leq \frac{T}{2}+ \frac{1}{|\Theta|}\sum_{\theta\in\Theta}\frac{\Delta}{\delta(d-1)}\sum_{a\in\mathcal{A}} \sum_{j=1}^{d-1}\mathbf{1}\left\{\mathrm{sign}(a_j)=\mathrm{sign}(\theta_j)\right\}\frac{101\mathbb{E}_{\theta}N_0^a}{198}\\
&\leq \frac{T}{2}+ \frac{101\Delta}{198\delta|\Theta|(d-1)} \sum_{j=1}^{d-1}\sum_{\theta\in\Theta}\sum_{a\in\mathcal{A}}\mathbb{E}_{\theta}\left[\mathbf{1}\left\{\mathrm{sign}(a_j)=\mathrm{sign}(\theta_j)\right\}N_0^a\right].
\end{aligned}
\end{align}

For a given $\theta$ and a coordinate  $j\in[d-1]$, we consider $\theta'$ that differs from $\theta$ only in the $j$th coordinate. Then we have
\begin{align*}&\mathbb{E}_{\theta}\left[\mathbf{1}\left\{\mathrm{sign}(a_j)=\mathrm{sign}(\theta_j)\right\}N_0^a\right]+\mathbb{E}_{\theta'}\left[\mathbf{1}\left\{\mathrm{sign}(a_j)=\mathrm{sign}(\theta_j')\right\}N_0^a\right]\\
&=\mathbb{E}_{\theta'}N_0^a + \mathbb{E}_{\theta}\left[\mathbf{1}\left\{\mathrm{sign}(a_j)=\mathrm{sign}(\theta_j)\right\}N_0^a\right]-\mathbb{E}_{\theta'}\left[\mathbf{1}\left\{\mathrm{sign}(a_j)=\mathrm{sign}(\theta_j)\right\}N_0^a\right]
\end{align*}
because $\mathbf{1}\left\{\mathrm{sign}(a_j)=\mathrm{sign}(\theta_j)\right\}+\mathbf{1}\left\{\mathrm{sign}(a_j)=\mathrm{sign}(\theta_j')\right\}=1$. Summing up this equality for $\theta\in\Theta$ and $a\in\mathcal{A}$, we obtain
\begin{align*}
&2\sum_{\theta\in \Theta}\sum_{a\in\mathcal{A}}\mathbb{E}_{\theta}\left[\mathbf{1}\left\{\mathrm{sign}(a_j)=\mathrm{sign}(\theta_j)\right\}N_0^a\right]\\
&=\sum_{\theta\in \Theta}\mathbb{E}_{\theta'}N_0 + \sum_{\theta\in \Theta}\mathbb{E}_{\theta}\left[\sum_{a\in\mathcal{A}}\mathbf{1}\left\{\mathrm{sign}(a_j)=\mathrm{sign}(\theta_j)\right\}N_0^a\right]\\
&\quad - \sum_{\theta\in \Theta}\mathbb{E}_{\theta'}\left[\sum_{a\in\mathcal{A}}\mathbf{1}\left\{\mathrm{sign}(a_j)=\mathrm{sign}(\theta_j)\right\}N_0^a\right]\\
&=\sum_{\theta\in \Theta}\mathbb{E}_{\theta'}N_0+ \sum_{\theta\in \Theta}\left(\mathbb{E}_{\theta}\left[N_0^{\mathcal{A}_j}\right]-\mathbb{E}_{\theta'}\left[N_0^{\mathcal{A}_j}\right]\right)
\end{align*}
where $\mathcal{A}_j$ is the set of all actions $a$ which satisfy $\mathbf{1}\{ \mathrm{sign}(a_j)=\mathrm{sign}(\theta_j)\}$. Here, to provide an upper bound on the term $\mathbb{E}_{\theta}[N_0^{\mathcal{A}_j}]-\mathbb{E}_{\theta'}[N_0^{\mathcal{A}_j}]$, we apply the following version of Pinsker's inequality due to~\cite{Auer_Jaksch2010}.
\begin{lemma}\label{lemma:pinsker}{\em \citep[Equation (49)]{Auer_Jaksch2010}}. 
Let $s = \{s_1,\dots,s_T\} \in \mathcal{S}^{T}$ denote the sequence of the observed states from time step $1$ to $T$. Then for any two distributions $\mathcal{P}_1$ and $\mathcal{P}_2$ over $\mathcal{S}^{T}$ and any bounded function $f: \mathcal{S}^{T}\rightarrow [0, B]$, we have
\begin{align*}
    \mathbb{E}_{\mathcal{P}_1} f(s) - \mathbb{E}_{\mathcal{P}_2} f(s) \leq \sqrt{\log 2/2}B\sqrt{\mathrm{\mathrm{KL}}(\mathcal{P}_2||\mathcal{P}_1)}
\end{align*}
where $\mathrm{KL}(\mathcal{P}_2||\mathcal{P}_1)$ is the Kullback–Leibler divergence of $\mathcal{P}_2$ from $\mathcal{P}_1$.
\end{lemma}
By~Lemma \ref{lemma:pinsker}, it holds that
\begin{align*}
2\sum_{\theta\in \Theta}\sum_{a\in\mathcal{A}}\mathbb{E}_{\theta}\left[\mathbf{1}\left\{\mathrm{sign}(a_j)=\mathrm{sign}(\theta_j)\right\}N_0^a\right]\leq \sum_{\theta\in \Theta}\mathbb{E}_{\theta'}N_0+ \sum_{\theta\in \Theta}\sqrt{\log 2/2}T\sqrt{\mathrm{KL}(\mathcal{P}_{\theta'}\parallel\mathcal{P}_{\theta})}.
\end{align*}
Here, we need to provide an upper bound on the KL divergence term $\mathrm{KL}(\mathcal{P}_{\theta'}\parallel\mathcal{P}_{\theta})$. For this, we prove the following lemma which is analogous to \citep[Lemma C.4]{yuewu2022}.
\begin{lemma}\label{lemma:kl-infinite}
    Suppose that $\theta$ and $\theta'$ only differ in the $j$th coordinate and $100\Delta\leq\delta \leq 1/101$. Then we have the following bound for the KL divergence of $\mathcal{P}_{\theta'}$ from $\mathcal{P}_{\theta}$.
\begin{align*}
   \mathrm{KL}(\mathcal{P}_{\theta'}\parallel\mathcal{P}_{\theta}) \leq \left(\frac{101}{99}\right)^2\frac{16 \Delta^2}{(d-1)^2 \delta} \mathbb{E}_{\theta'}{N}_0
\end{align*}
\end{lemma}
\begin{proof}
See Lemma \ref{sec:lemma:kl-infinite}.
\end{proof}

By~Lemma \ref{lemma:kl-infinite}, we deduce that
\begin{align}\label{lb-infinite-almost-final}
\begin{aligned}
&2\sum_{\theta\in \Theta}\sum_{a\in\mathcal{A}}\mathbb{E}_{\theta}\left[\mathbf{1}\left\{\mathrm{sign}(a_j)=\mathrm{sign}(\theta_j)\right\}N_0^a\right]\\
&\leq \sum_{\theta\in \Theta}\mathbb{E}_{\theta'}N_0+ \sum_{\theta\in \Theta}\frac{202}{99}\sqrt{2\log 2}\frac{T\Delta}{(d-1)\sqrt{\delta}}\sqrt{\mathbb{E}_{\theta'}N_0}\\
&\leq \sum_{\theta\in \Theta}\mathbb{E}_{\theta}N_0+ \sum_{\theta\in \Theta}\frac{202}{99}\sqrt{2\log 2}\frac{T\Delta}{(d-1)\sqrt{\delta}}\sqrt{\mathbb{E}_{\theta}N_0}.
\end{aligned}
\end{align}
Combining~\eqref{lb-infinite-intermediate} and~\eqref{lb-infinite-almost-final}, we deduce that
\begin{align}\label{eq:N1}
\begin{aligned}
\frac{1}{|\Theta|}\sum_{\theta\in\Theta}\mathbb{E}_\theta N_1&\leq \frac{T}{2}+ \frac{101\Delta}{396\delta|\Theta|} \sum_{\theta\in\Theta}\left(\mathbb{E}_{\theta}N_0+ \frac{202}{99}\sqrt{2\log 2}\frac{T\Delta}{(d-1)\sqrt{\delta}}\sqrt{\mathbb{E}_{\theta}N_0}\right)\\
&\leq \frac{T}{2}+ \frac{\Delta}{4\delta|\Theta|} \sum_{\theta\in\Theta}\left(\frac{4}{5}T+ 2\sqrt{2\log 2}\frac{T\Delta}{(d-1)\sqrt{\delta}}\frac{2\sqrt{T}}{\sqrt{5}}\right)\\
&\leq \frac{T}{2}+ \frac{\Delta T}{5\delta}+ \sqrt{\frac{2}{5}\log 2}\frac{\Delta^2 T^{3/2}}{(d-1)\delta^{3/2}}
\end{aligned}
\end{align}
where the second inequality follows from Lemma \ref{lemma:lb-infinite-1}.

\subsection{Proof of \Cref{thm:lb-infinite}}

Let us first argue that the diameter of $M_\theta$ is $D$.
\begin{lemma}\label{lem:lb-diameter}
The diameter of $M_\theta$ is $1/\delta$.
\end{lemma}
\begin{proof}
    Note that the expected travel time from state $x_1$ to state $x_0$ is $1/(\delta+\Delta)$ which is less than $1/\delta$, while the expected travel time from state $x_0$ to state $x_1$ is $1/\delta$. Hence, the diameter of our hard-to-learn MDP $M_\theta$ is $1/\delta$.
\end{proof}
As we set $\delta = 1/D$ for the average-reward setting, the diameter of $M_\theta$ equals $D$ by Lemma~\ref{lem:lb-diameter}. 

Recall that under the optimal policy, we have $p^*(x_1\mid x_0,a)=\delta + \Delta$. This means that under the optimal policy, the stationary distribution over states $x_0$ and $x_1$ is given by
$$\mu=\left(\frac{\delta}{2\delta+\Delta},\ \frac{\delta + \Delta}{2\delta +\Delta}\right).$$
As $r(x_0,a)=0$ and $r(x_1,a)=1$ for any $a\in\mathcal{A}$, it follows that the optimal average reward equals 
\begin{align}\label{eq:average-opt}
   J^*(M_\theta)=\frac{\delta+\Delta}{2\delta +\Delta}.
\end{align}
For simplicity, we refer to the regret of policy $\pi$ as $\regret_\theta(T)$. Then we have
$$\mathbb{E}_{\theta}\left[\regret_\theta(T)\right]=TJ^*(M_\theta) - \mathbb{E}_{\theta}\left[\sum_{t=1}^T r(s_t,a_t)\right]=TJ^*(M_\theta) - \mathbb{E}_{\theta}N_1.$$
Taking $\Theta = \{-\bar\Delta/(d-1),\bar\Delta/(d-1)\}^{d-1}$, we deduce that
\begin{equation}
\label{lb-infinite-beginning}\frac{1}{|\Theta|}\sum_{\theta\in\Theta}\mathbb{E}_{\theta}\left[\regret_\theta(T)\right]=TJ^*(M_\theta) - \frac{1}{|\Theta|}\sum_{\theta\in\Theta}\mathbb{E}_{\theta}N_1.
\end{equation}
Then it follows from~\eqref{eq:N1} that 
\begin{align*}
\frac{1}{|\Theta|}\sum_{\theta\in\Theta}\mathbb{E}_{\theta}\left[\regret_\theta(T)\right]&\geq \frac{(\delta+\Delta)T}{2\delta+\Delta} -\frac{T}{2}-\frac{\Delta T}{5\delta}- \sqrt{\frac{2}{5}\log 2}\frac{\Delta^2 T^{3/2}}{(d-1)\delta^{3/2}}\\
&=\frac{\Delta(\delta-2\Delta)T}{10\delta(2\delta+\Delta)} - \sqrt{\frac{2}{5}\log 2}\frac{\Delta^2 T^{3/2}}{(d-1)\delta^{3/2}}\\
&\geq \frac{2\Delta}{45\delta}T - \sqrt{\frac{2}{5}\log 2}\frac{\Delta^2 T^{3/2}}{(d-1)\delta^{3/2}}
\end{align*}
where the second inequality holds because $0<4\Delta\leq \delta$. Setting $\Delta$ as
$$\Delta =\frac{1}{45\sqrt{(2/5)\log 2}}\cdot \frac{(d-1)}{\sqrt{DT}},$$
the rightmost side equals
$$\frac{1}{2025\sqrt{(2/5)\log2}}(d-1)\sqrt{DT}.$$
When $d\geq 2$, we have $2(d-1)\geq d$. Moreover, $\sqrt{(2/5)\log 2}\leq 1/2$. So, we get that
$$
\frac{1}{|\Theta|}\sum_{\theta\in\Theta}\mathbb{E}_{\theta}\left[\regret_\theta(T)\right]\geq \frac{1}
{2025}d\sqrt{DT},$$
as required.

\subsection{Proof of \Cref{thm:lb-infinite-discounted}}

As in the discounted-reward setting, we refer to the regret of policy $\pi$ as $\regret_\theta(T)$. Let us show the following lemma that is useful to provide a lower bound on the regret. 
\begin{lemma} \label{lem:lower bound discounted regret transformation}
We have
  $$\mathbb{E}_{\theta} [\regret_\theta(T)]\geq\mathbb{E}_{\theta} \left[ \sum^T_{t=1} V^*(s_t) - \frac{1}{1-\gamma} \sum^T_{t=1} r(s_{t}, a_{t}) - \frac{\gamma}{(1-\gamma)^2} \right].$$
\end{lemma}
\begin{proof}
By the definition of $V_t^\pi$ and the regret in \Cref{sec:infinite-discounted}, we deduce that
    \begin{align*}
       \mathbb{E}_{\theta} [\regret_\theta(T)]  &= \mathbb{E}_{\theta} \left[ \sum^T_{t=1} V^*(s_t) - \sum^T_{t=1}\sum^{\infty}_{t^{\prime}=0} \gamma^{t^{\prime}} r(s_{t+t^{\prime}}, a_{t+t^{\prime}}) \right]\\ 
       &= \mathbb{E}_{\theta} \left[ \sum^T_{t=1} V^*(s_t) - \underbrace{\sum^{T}_{t=1} r(s_t, a_t) \sum^{t-1}_{t^{\prime}=0} \gamma^{t^{\prime}}}_{I_1} - \underbrace{\sum^{\infty}_{t=T+1} r(s_t, a_t) \sum^{t-1}_{t^{\prime}=t-T} \gamma^{t^{\prime}}}_{I_2} \right].
    \end{align*}
Note that
    \begin{align*}
        I_1 = \sum^{T}_{t=1} r(s_t, a_t) \sum^{t-1}_{t^{\prime}=0} \gamma^{t^{\prime}} \leq \sum^{T}_{t=1} r(s_t, a_t) \sum^{\infty}_{t^{\prime}=0} \gamma^{t^{\prime}} = \sum^{T}_{t=1} \frac{r(s_t, a_t)}{1-\gamma}.
    \end{align*}
Moreover,
    \begin{align*}
        I_2 = \sum^{\infty}_{t=T+1} r(s_t, a_t) \sum^{t-1}_{t^{\prime}=t-T} \gamma^{t^{\prime}} \leq \sum^{\infty}_{t=T+1} 1 \sum^{\infty}_{t^{\prime}=t-T} \gamma^{t^{\prime}} = \sum^{\infty}_{t=T+1} 1 \cdot \frac{\gamma^{t-T}}{1-\gamma} = \frac{\gamma}{(1-\gamma)^2},
    \end{align*}
    where the first equality holds by $r(s_t, a_t) \leq 1$. These bounds on $I_1$ and $I_2$ lead to the desired lower bound on the expected regret. 
\end{proof}

Recall that in state $x_0$, the optimal policy always takes action $a_\theta$ such that $a_\theta^\top \theta =  \bar\Delta$. Hence, the transition probability under the optimal policy is given by
\begin{align*}
      p_{\theta}(x_0 | x_0, a_{\theta}) &= 1 - f(\bar{\Delta}) = 1 - \delta - \Delta,\\
       p_{\theta}(x_1 | x_0, a_{\theta}) &= f(\bar{\Delta}) = \delta + \Delta,\\
     p_{\theta}(x_0 | x_1, a_{\theta}) &= f(0) = \delta,\\
        p_{\theta}(x_1 | x_1, a_{\theta}) &= 1 - f(0) = 1 - \delta.
    \end{align*}
Then it follows from the Bellman optimality equation~\eqref{bellman-discounted} that
\begin{align*}
        V^*(x_0) &= 0 + \gamma(1 - \delta - \Delta)V^*(x_0) + \gamma(\delta + \Delta)V^*(x_1),
        \\ V^*(x_1) &= 1 + \gamma \delta V^*(x_0) + \gamma (1 - \delta)V^*(x_1).
    \end{align*}
Therefore, the optimal value functions are given by
\begin{align}\label{eq:discounted-value-opt}
        V^*(x_0) = \frac{\gamma(\Delta + \delta)}{(1-\gamma)(\gamma(2\delta + \Delta - 1) + 1)}, \quad
        V^*(x_1) = \frac{\gamma(\Delta + \delta) + 1 - \gamma}{(1-\gamma)(\gamma(2\delta + \Delta - 1) + 1)}.
    \end{align}
Note that
\begin{align}\label{eq:discounted-regret-lb-1}
 \begin{aligned}
        &\frac{1}{| \Theta |} \sum_{\theta \in \Theta} \left[ \mathbb{E}_{\theta} [\regret_{\theta}(T)] + \frac{\gamma}{(1-\gamma)^2} \right]
        \\ &\geq \frac{1}{| \Theta |} \sum_{\theta \in \Theta} \mathbb{E}_{\theta} \left[ N_0 V^*(x_0) + N_1 V^*(x_1) - \frac{1}{1-\gamma} N_1 \right]
         \\ &= \frac{1}{(1-\gamma)| \Theta |} \sum_{\theta \in \Theta} \mathbb{E}_{\theta} \left[ N_0 \frac{\gamma(\Delta + \delta)}{\gamma(2\delta + \Delta - 1) + 1} + N_1 \frac{-\gamma \delta}{\gamma(2\delta + \Delta - 1) + 1} \right]
        \\ &= \frac{1}{(1-\gamma)| \Theta |} \sum_{\theta \in \Theta} \mathbb{E}_{\theta} \left[ T \frac{\gamma(\Delta + \delta)}{\gamma(2\delta + \Delta - 1) + 1} - N_1 \frac{\gamma(\Delta + 2\delta)}{\gamma(2\delta + \Delta - 1) + 1} \right]
        \\&=\frac{\gamma}{(1-\gamma)(\gamma(2\delta+\Delta-1)+1)}\left((\Delta + \delta)T-\frac{\Delta+2\delta}{|\Theta|}\sum_{\theta \in \Theta} \mathbb{E}_{\theta} [N_1]\right)
    \end{aligned}
    \end{align}
    where the the inequality is implied by $r(x_0, a)=0$ and $r(x_1, a)=1$ for any $a$ and Lemma~\ref{lem:lower bound discounted regret transformation}, the first equality is due to~\eqref{eq:discounted-value-opt}, and the second equality holds because $T = N_0 + N_1$. Moreover, 
    \begin{align}\label{eq:discounted-regret-lb-2}
    \begin{aligned}
    &(\Delta + \delta)T-\frac{\Delta+2\delta}{|\Theta|}\sum_{\theta \in \Theta} \mathbb{E}_{\theta} [N_1]\\
    &\geq (\Delta + \delta)T - \frac{(\Delta+2\delta)T}{2} - (\Delta+2\delta)\left(\frac{\Delta T}{5\delta} + \sqrt{\frac{2}{5}\log 2}\frac{\Delta^2 T^{3/2}}{(d-1)\delta^{3/2}}\right)\\
    &=\left(\frac{\Delta}{2} - \frac{(\Delta+2\delta)\Delta}{5\delta}\right) T - (\Delta+2\delta)\sqrt{\frac{2}{5}\log 2}\frac{\Delta^2 T^{3/2}}{(d-1)\delta^{3/2}}
    \end{aligned}
    \end{align}
    Here, by Lemma~\ref{parameter-bound1}, we know that $100\Delta\leq \delta$, which implies that
    $$\Delta+2\delta \leq \frac{201}{100}\delta,\quad \frac{\Delta}{2} - \frac{(\Delta+2\delta)\Delta}{5\delta}\geq \frac{49}{500}\Delta.$$
    Then it follows that
     \begin{align}\label{eq:discounted-regret-lb-3}
     \begin{aligned}
    &\left(\frac{\Delta}{2} - \frac{(\Delta+2\delta)\Delta}{5\delta}\right) T - (\Delta+2\delta)\sqrt{\frac{2}{5}\log 2}\frac{\Delta^2 T^{3/2}}{(d-1)\delta^{3/2}}\\
    &\geq \left(\frac{49}{500} - \frac{201}{100} \sqrt{\frac{2T}{5\delta}\log 2}\frac{\Delta}{(d-1)}\right) \Delta T\\
    &\geq \left(\frac{49}{500} -\frac{201}{4500}\right)\Delta T\\
    &\geq \frac{240}{4500}\Delta T
    \end{aligned}
    \end{align}
    where the second inequality is due to our choice of $\Delta$. Moreover, since $\Delta\leq 100\Delta\leq \delta$, we have 
    \begin{equation}\label{eq:discounted-regret-lb-4}
    \gamma(2\delta+ \Delta -1) +1 \leq 1-\gamma + 3\delta\gamma= 1-\gamma + 3(1-\gamma)\gamma\leq 4 (1-\gamma).
    \end{equation}
    Combining \eqref{eq:discounted-regret-lb-1}, \eqref{eq:discounted-regret-lb-2}, \eqref{eq:discounted-regret-lb-3}, and \eqref{eq:discounted-regret-lb-4}, we obtain 
    \begin{align*}
      \frac{1}{| \Theta |} \sum_{\theta \in \Theta}  \mathbb{E}_{\theta} [\regret_{\theta}(T)]
        &\geq \frac{24}{1800(1-\gamma)^2}\gamma\Delta T -\frac{\gamma}{(1-\gamma)^2}\\
        &=\frac{\gamma (d-1)\sqrt{T}}{3375(1-\gamma)^{3/2}\sqrt{(2/5)\log 2}}-\frac{\gamma}{(1-\gamma)^2}\\
        &\geq \frac{\gamma}{3375(1-\gamma)^{3/2}}d\sqrt{T}-\frac{\gamma}{(1-\gamma)^2}
    \end{align*}
    where the last inequality holds because $2(d-1)\geq 1$ and $\sqrt{(2/5)\log 2}\leq 1/2$.

\subsection{Proof of Lemma \ref{lemma:lb-infinite-1}}\label{sec:lemma:lb-infinite-1}

We have that
\begin{equation}\label{lb-infinite-1}
\begin{aligned}
    \mathbb{E}_{{\theta}} {N}_{1}&= \sum_{t=2}^T\mathcal{P}_\theta(s_t=x_1)\\
    &= \underbrace{\sum_{t=2}^{T}   \mathcal{P}_{{\theta}}(s_t=x_{1} \mid s_{t-1}=x_{1})\mathcal{P}_{{\theta}} (s_{t-1} = x_{1}) }_{I_1} 
    + \underbrace{\sum_{t=2}^{T} \mathcal{P}_{{\theta}} (s_{t}=x_{1}, s_{t-1}=x_{0}) }_{I_2}. 
    \end{aligned}
    \end{equation}
For $I_1$, note that $\mathcal{P}_{{\theta}}(s_t=x_{1} \mid s_{t-1}=1-\delta$ regardless of action $a_{t-1}$, so we have
\begin{equation}\label{lb-infinite-I1}
I_1=(1-\delta)\sum_{t=2}^{T}   \mathcal{P}_{{\theta}} (s_{t-1} = x_{1})= (1-\delta)\mathbb{E}_{{\theta}} {N}_{1} - (1-\delta)\mathcal{P}_{{\theta}} (s_T = x_{1}).
\end{equation}
For $I_2$, note that
\begin{equation}\label{lb-infinite-I2}
\begin{aligned}
I_2&=\sum_{t=2}^T\sum_{a\in\mathcal{A}}\mathcal{P}_{{\theta}} (s_{t}=x_{1}\mid s_{t-1}=x_{0}, a_{t-1}=a)\mathcal{P}_{{\theta}} (s_{t-1}=x_{0},a_{t-1}=a)\\
&=\sum_{t=2}^T\sum_{a\in\mathcal{A}}f(a^\top\theta) \mathcal{P}_{{\theta}} (s_{t-1}=x_{0},a_{t-1}=a)\\
&=\sum_{a\in\mathcal{A}}f(a^\top\theta) \left(\mathbb{E}N_0^a -\mathcal{P}_{\theta}(s_T=x_0, a_T=a)\right).
\end{aligned}
\end{equation}
Plugging~\eqref{lb-infinite-I1} and~\eqref{lb-infinite-I2} to~\eqref{lb-infinite-1}, we deduce that
\begin{align}\label{lb-infinite-N1-0}
\begin{aligned}
\mathbb{E}_\theta N_1 &= \sum_{a\in\mathcal{A}} \frac{f(a^\top\theta)}{\delta} \mathbb{E}_\theta N_0^a -\underbrace{\left(\frac{1-\delta}{\delta}\mathcal{P}_{\theta}(x_T=x_1)+\sum_{a\in\mathcal{A}} \frac{f(a^\top\theta)}{\delta}\mathcal{P}_{\theta}(s_T=x_0,a_T=a)\right)}_{\psi_\theta}\\
&=\mathbb{E}_\theta N_0 + \frac{1}{\delta}\sum_{a\in\mathcal{A}}(f(a^\top\theta)-\delta) \mathbb{E}_\theta N_0^a -\psi_\theta.
\end{aligned}
\end{align}
Since $T=\mathbb{E}_\theta N_0+\mathbb{E}_\theta N_1$, it follows that
\begin{equation}\label{lb-infinite-N1}
\mathbb{E}_\theta N_1\leq \frac{T}{2}+ \frac{1}{2\delta}\sum_{a\in\mathcal{A}}(f(a^\top\theta)-\delta) \mathbb{E}_\theta N_0^a.
\end{equation}
Note that
$$f(a^\top\theta)-\delta = f(a^\top\theta)-f(0)\leq (\delta+\Delta)a^\top\theta$$
where the first inequality is from~Lemma \ref{mvt-mnl}. 

Next, for $\mathbb{E}_{\theta}N_0$, since $f(-\bar\Delta)\leq f(a^\top\theta)\leq f(\bar\Delta)=\delta+\Delta$, we have from \eqref{lb-infinite-N1-0} that
\begin{align*}
\begin{aligned}
\mathbb{E}_{\theta}N_1 &\geq \left(1+ \frac{f(-\bar\Delta)-f(0)}{\delta}\right)\mathbb{E}_{\theta}N_0 - \frac{1-\delta}{\delta}\mathcal{P}_{\theta}(x_T=x_1)-\frac{\delta+\Delta}{\delta}\mathcal{P}_{\theta}(s_T=x_0)\\
&\geq  \left(1+ \frac{f(-\bar\Delta)-f(0)}{\delta}\right)\mathbb{E}_{\theta}N_0 - \frac{1-\delta}{\delta}+\frac{1-2\delta-\Delta}{\delta}\mathcal{P}_{\theta}(s_T=x_0)\\
&\geq  \left(1+ \frac{f(-\bar\Delta)-f(0)}{\delta}\right)\mathbb{E}_{\theta}N_0 - \frac{1-\delta}{\delta}
\end{aligned}
\end{align*}
where the second inequality holds because $2\delta+\Delta\leq 1$. This implies that
$$\mathbb{E}_\theta N_0\leq \frac{T+\frac{1-\delta}{\delta}}{2-\frac{1}{\delta}\left(\delta - f(-\bar\Delta)\right)}.$$
The following lemma provides a lower bound on $f(-\bar\Delta)$.
\begin{lemma}\label{lb-min-prob}
$f(-\bar\Delta)\geq \delta/2$ if and only if $\Delta\leq \delta(1-\delta)$.
\end{lemma}
\begin{proof}
$f(-\bar\Delta)\geq \delta/2$ if and only if  $1+\frac{1-\delta}{\delta}\exp({\bar\Delta})\leq 2/\delta$, which is equivalent to $\exp(-\bar\Delta)\geq (1-\delta)/(2-\delta)$. By plugging in the definition of $\bar \Delta$ to the inequality, we get that $f(-\bar\Delta)\geq \delta/2$ if and only if $\delta(1-\delta-\Delta)/((1-\delta)(\delta+\Delta)) \geq (1-\delta)/(2-\delta)$, which is equivalent to $\Delta\leq \delta(1-\delta)$.
\end{proof}
\noindent By simple algebra, we may derive from $f(-\bar\Delta)\geq \delta/2$ that $2(\delta-f(-\bar \Delta))\leq \delta$ holds. Since we assumed that $\Delta\leq \delta(1-\delta)$, it follows that
$$\mathbb{E}_\theta N_0\leq \frac{T+ \frac{1-\delta}{\delta}}{3/2}=\left(\frac{99}{101}\right)^4\cdot \frac{4}{5}T$$
where the inequality holds because  
$$\frac{1-\delta}{\delta}\leq \frac{1}{\delta}\leq \left(\frac{3}{2}\cdot \frac{4}{5}\cdot \left(\frac{99}{101}\right)^4-1\right)T,$$
as required.

\subsection{Proof of Lemma \ref{lemma:kl-infinite}}\label{sec:lemma:kl-infinite}

First of all, we consider the following lemma.
    \begin{lemma}\label{lem:Lemma 20 in Auer_Jacsch2010}{\em \citep[Lemma 20] {Auer_Jaksch2010}}. 
    Suppose $0 \le \delta' \le 1/2$ and $\epsilon' \le 1-2\delta'$, then
        \begin{align*}
        \delta' \log \frac{\delta'}{\delta' + \epsilon'} + (1-\delta')\log \frac{(1-\delta')}{1-\delta'-\epsilon'} \le \frac{2 (\epsilon')^2}{\delta'}.
        \end{align*}
    \end{lemma}

Let $\bm{s}_t$ denote the sequence of states $\{ s_1,\ldots,s_t \}$ from time step 1 to $T$. By the Markovian property of MDP, we may decompose the KL divergence term of $\mathcal{P}_{\theta'}$ from $\mathcal{P}_{\theta}$ as follows.
    \begin{align*}
    \mathrm{KL} \left( \mathcal{P}_{ {\theta'} } \parallel \mathcal{P}_{ {\theta} } \right)
    = \sum_{t=1}^{T-1} \mathrm{KL} \left(\mathcal{P}_{ {\theta'} } \left( s_{t+1} \mid \bm{s}_t \right) \parallel \mathcal{P}_{ {\theta} } \left( s_{t+1} \mid \bm{s}_t \right) \right)
    \end{align*}
where the KL divergence of $\mathcal{P}_{ {\theta}' } \left( s_{t+1} \mid \bm{s}_t \right)$ from $\mathcal{P}_{ {\theta} } \left( s_{t+1} \mid \bm{s}_t \right)$ is given by
    \begin{align*}
    \mathrm{KL} \left(\mathcal{P}_{ {\theta'} } \left( s_{t+1} \mid \bm{s}_t \right) \parallel \mathcal{P}_{ {\theta} } \left( s_{t+1} \mid \bm{s}_t \right) \right)
    = \sum_{\bm{s}_{t+1} \in \mathcal{S}^{t+1}} \mathcal{P}_{ {\theta}' } \left( \bm{s}_{t+1} \right) \log \frac{\mathcal{P}_{ {\theta}' } \left( s_{t+1} \mid \bm{s}_t \right)}{\mathcal{P}_{ {\theta} } \left( s_{t+1} \mid \bm{s}_t \right)}.
    \end{align*}
The right-hand side can be further decomposed as follows.
    \begin{align*}
    &\sum_{\bm{s}_{t+1} \in \mathcal{S}^{t+1}} \mathcal{P}_{ {\theta}' } \left( \bm{s}_{t+1} \right) \log \frac{\mathcal{P}_{ {\theta}' } \left( s_{t+1} \mid \bm{s}_t \right)}{\mathcal{P}_{ {\theta} } \left( s_{t+1} \mid \bm{s}_t \right)} \\
    &= \sum_{\bm{s}_{t} \in \mathcal{S}^{t}} \mathcal{P}_{ {\theta}' } \left( \bm{s}_{t} \right) 
    \sum_{x \in \mathcal{S}} \mathcal{P}_{ {\theta}' } \left( s_{t+1} = x \mid \bm{s}_t \right) \log \frac{\mathcal{P}_{ {\theta}' } \left( s_{t+1} = x \mid \bm{s}_t \right)}{\mathcal{P}_{ {\theta} } \left( s_{t+1} = x \mid \bm{s}_t \right)} \\
    &= \sum_{\bm{s}_{t-1} \in \mathcal{S}^{t-1}} \mathcal{P}_{ {\theta}' } \left( \bm{s}_{t-1} \right) \sum_{x' \in \mathcal{S}}\sum_{{a} \in \mathcal{A}} \mathcal{P}_{ {\theta}' } \left( s_{t} = x', a_t = {a} \mid \bm{s}_{t-1} \right) \\
    &\quad \times \sum_{x \in \mathcal{S}} \mathcal{P}_{ {\theta}' } \left( s_{t+1} = x \mid \bm{s}_{t-1}, s_t = x', a_t = {a} \right) 
    \underbrace{ \log \frac{\mathcal{P}_{ {\theta}' } \left( s_{t+1} = x \mid \bm{s}_{t-1}, s_t = x', a_t = {a} \right)}{\mathcal{P}_{ {\theta} } \left( s_{t+1} = x \mid \bm{s}_{t-1}, s_t = x', a_t = {a} \right)} }_{I_1}.
    \end{align*}
Note that at state $x_1$, the transition probability does not depend on the action taken and the underlying transition core. This implies that
$\mathcal{P}_{ {\theta}' } \left( s_{t+1} = x \mid \bm{s}_{t-1}, s_t = x', a_t = {a} \right)= \mathcal{P}_{ {\theta}} \left( s_{t+1} = x \mid \bm{s}_{t-1}, s_t = x', a_t = {a} \right)$ for all ${\theta}$, ${\theta}'$. This means that if $x'=x_1$, we have $I_1=0$. Then it holds that
    \begin{align*}
    &\sum_{\bm{s}_{t+1} \in \mathcal{S}^{t+1}} \mathcal{P}_{ {\theta}' } \left( \bm{s}_{t+1} \right) \log \frac{\mathcal{P}_{ {\theta}' } \left( s_{t+1} \mid \bm{s}_t \right)}{\mathcal{P}_{ {\theta} } \left( s_{t+1} \mid \bm{s}_t \right)} \\
    &=\sum_{\bm{s}_{t-1} \in \mathcal{S}^{t-1}} \mathcal{P}_{ {\theta}' } \left( \bm{s}_{t+1} \right) \sum_{{a}} \mathcal{P}_{ {\theta}' } \left( s_t = x_{0}, a_t = {a} \mid \bm{s}_{t-1}\right) \\
    &\quad \times \sum_{x \in \mathcal{S}} \mathcal{P}_{ {\theta}' } \left( s_{t+1} = x \mid \bm{s}_{t-1}, s_t = x_{0}, a_t = {a} \right)
    \log \frac{\mathcal{P}_{ {\theta}' } \left( s_{t+1}=s \mid \bm{s}_{t-1}, s_t=x_{0},a_t={a} \right)}{\mathcal{P}_{ {\theta} } \left( s_{t+1}=s \mid \bm{s}_{t-1}, s_t=x_{0},a_t={a} \right)} \\
    &= \sum_{{a}} \mathcal{P}_{ {\theta}' } \left( s_t = x_{0,1}, a_t = {a} \right) \\
    &\quad \times \underbrace{ \sum_{x \in \mathcal{S}} \mathcal{P}_{ {\theta}' } \left( s_{t+1}=s \mid s_{t} = x_{0}, a_t = {a} \right)
    \log \frac{\mathcal{P}_{ {\theta}' } \left( s_{t+1}=x \mid s_t=x_{0},a_t={a} \right)}{\mathcal{P}_{ {\theta} } \left( s_{t+1}=x \mid s_t=x_{0},a_t={a} \right)} }_{I_2}.
    \end{align*}
To bound $I_2$, we know that $s_{t+1}$ follows the Bernoulli distribution over $x_0$ and $x_1$ with probability $1-f(a^\top\theta')$ and $f(a^\top\theta')$. Then, we have
    \begin{align*}
    I_2 
    &= \left( 1-f(a^\top\theta')\right) \log \frac{ 1-f(a^\top\theta') }{ 1-f(a^\top\theta) } 
    +  f(a^\top\theta')  \log \frac{ f(a^\top\theta') }{ f(a^\top\theta)}.
    \end{align*}
Note that 
$$\frac{1}{100}\geq \frac{101}{100}\delta\geq \delta + \Delta=f(\bar\Delta)\geq f(a^\top\theta')\geq f(-\bar\Delta) \geq \frac{\delta}{2}$$
where the first inequality is due to $\delta\leq 1/101$, the second holds because $100\Delta\leq \delta$, and the last inequality is by Lemma \ref{lb-min-prob}. Moreover, since $f(\bar\Delta)\leq 1/100$, 
$$f(a^\top\theta)-f(a^\top\theta')\leq f(\bar\Delta)\leq \frac{1}{100}\leq 1- f(\bar \Delta)\leq 1-f(a^\top\theta').$$
Then we deduce that
\begin{align*}
    I_2\leq \frac{2 \left(f(a^\top\theta')-f(a^\top\theta)\right)^2}{f(a^\top\theta')}\leq \frac{16(\delta+\Delta)^2\bar\Delta^2}{\delta(d-1)^2}\leq \left(\frac{101}{99}\right)^2\frac{16\Delta^2}{\delta(d-1)^2}
    \end{align*}
    where the first inequality is implied by Lemma~\ref{lem:Lemma 20 in Auer_Jacsch2010} with $\delta' = f(a^\top\theta')$ and $\epsilon' = f(a^\top\theta)-f(a^\top\theta')$, the second inequality holds because of $f(a^\top \theta')\geq \delta/2$ and Lemma \ref{mvt-mnl}. Then 
    \begin{align*}
    \mathrm{KL} \left( \mathcal{P}_{ {\theta'} } \parallel \mathcal{P}_{ {\theta} } \right)
    &= \sum_{t=1}^{T-1} \sum_{\bm{s}_{t+1} \in \mathcal{S}^{t+1}} \mathcal{P}_{ {\theta'} } (\bm{s}_{t+1}) \log \frac{ \mathcal{P}_{ {\theta'} } (s_{t+1} \mid \bm{s}_t) }{ \mathcal{P}_{ {\theta} } (s_{t+1} \mid \bm{s}_t)} \\
    &\le \left(\frac{101}{99}\right)^2\frac{16 \Delta^2}{(d-1)^2 \delta} \sum_{t=1}^{T-1} \sum_{{a}} \mathcal{P}_{ {\theta}' } (s_{t} = x_{0} \mid a_t = {a}) \\
    &= \left(\frac{101}{99}\right)^2\frac{16 \Delta^2}{(d-1)^2 \delta} \sum_{t=1}^{T-1} \mathcal{P}_{ {\theta}' } (s_{t} = x_{0}) \\
    &= \left(\frac{101}{99}\right)^2\frac{16 \Delta^2}{(d-1)^2 \delta} \mathbb{E}_{\theta'}{N}_0,
    \end{align*}
as required.

\end{document}